\PassOptionsToPackage{frozencache}{minted}

\IfFileExists{./prepreamble-amspreprint.sty}{\RequirePackage[packages,theorems,changes]{prepreamble-amspreprint}}{}

\makeatletter
\providecommand*{\input@path}{}\edef\input@path{{./scoop-latex/}\input@path}
\makeatother
\PassOptionsToPackage{style = authoryear-comp, backend = biber}{biblatex}
\PassOptionsToPackage{commentmarkup = footnote}{changes}
\RequirePackage{algpseudocode}    

\documentclass[english]{amsart}

\RequirePackage{biblatex}
\RequirePackage{ltxcmds}
\IfFileExists{./preamble-amspreprint.sty}{\RequirePackage[packages,theorems,changes]{preamble-amspreprint}}{}

\usepackage{scoop-local}

\IfFileExists{./postpreamble-amspreprint.sty}{\RequirePackage[packages,theorems,changes]{postpreamble-amspreprint}}{}

\makeatletter
\@ifpackageloaded{changes}{
\definechangesauthor[name = {Lukas Baumgärtner}, color = {red!80!black}]{LB}
\definechangesauthor[name = {Ronny Bergmann}, color = {blue!80!black}]{RB}
\definechangesauthor[name = {Roland Herzog}, color = {orange!80!black}]{RH}
\definechangesauthor[name = {Stephan Schmidt}, color = {green!70!black}]{STS}
\definechangesauthor[name = {Manuel Weiß}, color = {teal}]{MW}
}{}
\makeatother        

\makeatletter
\@ifpackageloaded{hyperref}{%
	\hypersetup{
		pdftitle = {Total Generalized Variation of the Normal Vector Field and Applications to Mesh Denoising},
		pdfauthor = {Lukas Baumgärtner, Ronny Bergmann, Roland Herzog, Stephan Schmidt, Manuel Weiß},
		pdfkeywords = {total generalized variation, manifold-valued data, mesh denoising, split Bregman iteration},
		pdfcreator = {Created using the Scoop Template Engine version 1.6.0.}
	}
}{
	\pdfinfo{
		/Title (Total Generalized Variation of the Normal Vector Field and Applications to Mesh Denoising)
		/Author (Lukas Baumgärtner, Ronny Bergmann, Roland Herzog, Stephan Schmidt, Manuel Weiß)
		/Subject ()
		/Keywords (total generalized variation, manifold-valued data, mesh denoising, split Bregman iteration)
		/Creator (Created using the Scoop Template Engine version 1.6.0.)
	}
}
\makeatother

\title[TGV of the Normal and Applications to Mesh Denoising]{Total Generalized Variation of the Normal Vector Field and Applications to Mesh Denoising}

\author[L. Baumgärtner]{Lukas Baumgärtner\orcidlink{0000-0003-1007-4815}}
\address[L. Baumgärtner]{Institut für Mathematik, Humboldt University of Berlin, 10099 Berlin, Germany}
\email{\detokenize{lukas.baumgaertner@hu-berlin.de}}
\urladdr{https://www.mathematik.hu-berlin.de/en/people/mem-vz/1693318}

\author[R. Bergmann]{Ronny Bergmann\orcidlink{0000-0001-8342-7218}}
\address[R. Bergmann]{Norwegian University of Science and Technology, Department of Mathematical Sciences, NO-7041 Trondheim, Norway}
\email{\detokenize{ronny.bergmannn@ntnu.no}}
\urladdr{https://www.ntnu.edu/employees/ronny.bergmann}

\author[R. Herzog]{Roland Herzog\orcidlink{0000-0003-2164-6575}}
\address[R. Herzog]{Interdisciplinary Center for Scientific Computing, Heidelberg University, 69120 Heidelberg, Germany and Institute for Mathematics, Heidelberg University, 69120 Heidelberg, Germany}
\email{\detokenize{roland.herzog@iwr.uni-heidelberg.de}}
\urladdr{https://scoop.iwr.uni-heidelberg.de}

\author[S. Schmidt]{Stephan Schmidt\orcidlink{0000-0002-4888-0794}}
\address[S. Schmidt]{University of Trier, Universitätsring 15, 54296 Trier, Germany}
\email{\detokenize{stephan.schmidt@uni-trier.de}}
\urladdr{https://www.math.uni-trier.de/\string~schmidt}

\author[M. Weiß]{Manuel Weiß\orcidlink{0000-0003-2164-6575}}
\address[M. Weiß]{Interdisciplinary Center for Scientific Computing, Heidelberg University, 69120 Heidelberg, Germany}
\email{\detokenize{roland.herzog@iwr.uni-heidelberg.de}}
\urladdr{https://scoop.iwr.uni-heidelberg.de}

\thanks{This work was supported by DFG grants HE 6077/10--2 and SCHM~3248/2--2 within the Priority Program SPP~1962 (Non-smooth and Complementarity-based Distributed Parameter Systems: Simulation and Hierarchical Optimization), which is gratefully acknowledged.}

\date{\today}

\dedicatory{}

\begin{document}

\begin{abstract}
We propose a novel formulation for the second-order total generalized variation (TGV) of the normal vector on an oriented, triangular mesh embedded in $\R^3$.
The normal vector is considered as a manifold-valued function, taking values on the unit sphere.
Our formulation extends previous discrete TGV models for piecewise constant scalar data that utilize a Raviart-Thomas function space.
To extend this formulation to the manifold setting, a tailor-made tangential Raviart-Thomas type finite element space is constructed in this work.
The new regularizer is compared to existing methods in mesh denoising experiments.

\end{abstract}


\makeatletter
\ltx@ifpackageloaded{hyperref}{%
\subjclass[2010]{\href{https://mathscinet.ams.org/msc/msc2020.html?t=65D18}{65D18}, \href{https://mathscinet.ams.org/msc/msc2020.html?t=49Q10}{49Q10}, \href{https://mathscinet.ams.org/msc/msc2020.html?t=49M15}{49M15}, \href{https://mathscinet.ams.org/msc/msc2020.html?t=90C30}{90C30}, \href{https://mathscinet.ams.org/msc/msc2020.html?t=65K05}{65K05}}
}{%
}
\makeatother

\maketitle

\section{Introduction}
\label{section:introduction}

The total variation (TV) seminorm is a commonly used regularizer for various kinds of inverse problems.
It was first proposed as a regularizer for image denoising problems in \cite{RudinOsherFatemi:1992:1} and is ever since omnipresent in the field of mathematical image processing.
On a bounded domain $\Omega \subseteq \R^2$, the TV-seminorm of a function $u \in L^1(\Omega)$ can be defined as
\begin{equation}
	\label{eq:tv:primal}
	\TV(u)
	\coloneqq
	\sup \setDef[auto]{\int_\Omega u \div \bv \dx}{\bv \in \cC_c^1(\Omega, \R^2) \text{ \st\ } \norm{\bv}_{L^\infty(\Omega, \R^2)} \le 1}
	,
\end{equation}
where $\cC_c^1(\Omega, \R^2)$ is the set of continuously differentiable vector fields with compact support in $\Omega$.
Unlike smooth regularizers, the TV-seminorm is capable of removing noise while preserving discontinuities in the data.
However, it suffers from the so-called staircasing effect, meaning that discontinuous reconstructions with several small jumps occur even where smoother ones are desired.

The imaging community has proposed numerous modifications to the total variation regularizer in order to overcome the staircasing effect for imaging problems; see \eg \cite{ChambolleLions:1997:1,ChanTai:2004:1,ChanEsedogluPark:2010:1}.
One of the most popular extensions to this day is the total generalized variation ($\TGV$), introduced in \cite{BrediesKunischPock:2010:1}.
Given weights $\alpha_0, \alpha_1 \in \R_{>0}$ its second-order non-symmetric version reads
\begin{multline}
	\label{eq:tgv:primal}
	\TGV_{(\alpha_0, \alpha_1)}^2(u)
	\\
	=
	\sup \setDef[auto]{\int_\Omega u \div \Div V \dx}{V \in \cC_c^2(\Omega, \R^{2 \times 2})
		\text{ \st}
		\paren[auto]\{\}{%
			\begin{aligned}
				\norm{V}_{L^\infty(\Omega, \R^{2 \times 2})}
				&
				\le
				\alpha_0
				\\
				\norm{\Div V}_{L^\infty(\Omega, \R^2)}
				&
				\le
				\alpha_1
			\end{aligned}
		}
	}
	,
\end{multline}
where $\Div V$ denotes the row-wise divergence operator of the twice continuously differentiable matrix-valued field $V \in \cC_c^2(\Omega, \R^{2 \times 2})$.
Often the above formulation is reformulated using Fenchel duality to obtain
\begin{equation}
	\label{eq:tgv:dual}
	\TGV_{(\alpha_0, \alpha_1)}^2(u)
	=
	\min_{\bw \in \operatorname{BV}(\Omega,\R^2)}
	\alpha_1 \, \norm{\nabla u - \bw}_{\cM(\Omega, \R^2)}
	+
	\alpha_0 \, \norm{\nabla \bw}_{\cM(\Omega, \R^{2 \times 2})}
	,
\end{equation}
where $\operatorname{BV}$ is the space of bounded variation, $\nabla$ is the distributional gradient and $\norm{\cdot}_{\cM}$ is the Radon norm; see \cite{HollerKunisch:2014:1} for more details.
Many authors also consider the symmetric variant, which utilizes the symmetrized gradient operator in the $\alpha_0$-term above.
Both variants of the second-order $\TGV$ regularizer favor piecewise linear instead of piecewise constant reconstructions and thereby overcome the staircasing effect.

Notice that in the case of piecewise constant functions~$u$, both variants of $\TGV$ reduce to $\alpha_1 \, \TV$ when taken literally.
This has led to a number of application specific discrete formulations of $\TGV$, which are not equivalent to the continuous formulation.

A $\TGV$ formulation for graph signals was proposed in~\cite{OnoYamadaKumazawa:2015:1}.
This concept was subsequently applied to the dual graph of a triangular mesh in \cite{GongSchullckeKruegerZiolekZhangMuellerLisseMoeller:2018:1} to postulate the earliest version of $\TGV$ for piecewise constant data on triangular meshes.
It was observed in \cite[Section~2.3.3]{BaumgaertnerBergmannHerzogSchmidtVidalNunez:2023:1} that this formulation can be interpreted as using a divergence-like operator in the $\alpha_0$-term in \eqref{eq:tgv:dual} instead of a gradient.
We refer the reader to \cite{BrinkmannBurgerGrah:2018:1}, where various differential operators in the $\alpha_0$-term were originally investigated.
The numerical results presented there strongly suggest that the divergence operator generally leads to oscillations, which are also present in the numerical results of \cite{GongSchullckeKruegerZiolekZhangMuellerLisseMoeller:2018:1}.

To avoid these oscillations, we have proposed an improved formulation of $\TGV$ suitable for piecewise constant functions on triangular meshes in \cite[Section~3]{BaumgaertnerBergmannHerzogSchmidtVidalNunez:2023:1}.
Our formulation utilizes a gradient-like operator for the $\alpha_0$-term and a lowest-order Raviart--Thomas function for the auxiliary variable~$\bw$.

\subsection*{Total Generalized Variation for Mesh Denoising}

Two alternative formulations of $\TGV$ for piecewise constant functions were proposed in \cite{LiuLiWangLiuChen:2022:1,ZhangHeWang:2022:1} for the purpose of mesh denoising.
On the one hand, the authors of \cite{LiuLiWangLiuChen:2022:1} proposed a formulation based on the $\TGV$ on graphs from \cite{GongSchullckeKruegerZiolekZhangMuellerLisseMoeller:2018:1} but added an additional weight function into the divergence operator.
The authors of \cite{ZhangHeWang:2022:1}, on the other hand, proposed a novel way to compute discrete (second-order) derivatives of piecewise constant functions.
They replaced the differential operators in \eqref{eq:tgv:dual} by their discrete analogs to obtain a formulation of $\TGV$.
Both \cite{LiuLiWangLiuChen:2022:1,ZhangHeWang:2022:1} successfully utilize their respective formulations for the purpose of mesh denoising based on the total generalized variation of the unit normal vector.
A slightly different approach to a discrete formulation of $\TGV$ was taken by \cite{ZhangPeng:2022:1} for continuous, piecewise linear data on triangular meshes.
This approach requires the definition of normal vectors at mesh vertices to be utilized for mesh denoising.

It is worth mentioning that \cite{LiuLiWangLiuChen:2022:1,ZhangHeWang:2022:1,ZhangPeng:2022:1} treat the normal vector as an element of $\R^3$ and not as an element of the unit sphere~$\sphere \coloneqq \setDef[normal]{\bn \in \R^3}{\abs{\bn}_2 = 1} \subseteq \R^3$.
A variant of $\TV$ of the normal vector of a mesh is developed in \cite{ZhangWuZhangDeng:2015:1} as well as in \cite{WuZhengCaiFu:2015:1,BergmannHerrmannHerzogSchmidtVidalNunez:2020:2}, where the latter two take the manifold nature of~$\sphere$ into account.
As these are based on first-order $\TV$, but not $\TGV$, the staircasing effect also occurs, resulting in poor reconstructions of curved areas.
To overcome this, the concept of total general variation needs to be reinterpreted for normal vector data since the sphere is not a linear space.
While the $\TGV$ seminorm \eqref{eq:tgv:dual} favors piecewise linear functions and thus piecewise constant gradients, the sought-after $\TGV$ formulation for the normal vector should favor areas of constant principal curvatures.

\subsection*{Contributions}

The goal of this paper is to propose a formulation of $\TGV$ for normal vector fields on triangular meshes that favors areas of constant (discrete) curvature.
To this end, we propose an adaptation of the discrete $\TGV$ formulation for piecewise constant real-valued data from \cite{BaumgaertnerBergmannHerzogSchmidtVidalNunez:2023:1}.
We construct a special Raviart--Thomas-like finite element space for the analog of the auxiliary variable~$\bw$ in \eqref{eq:tgv:dual} that captures derivative information of the normal vector field.
It is worth mentioning here that concepts of total generalized variation for manifold-valued data have been considered in \cite{BrediesHollerStorathWeinmann:2018:1}.
Closely related approaches based on second-order total variation have been taken in \cite{BacakBergmannSteidlWeinmann:2016:1, BergmannFitschenPerschSteidl:2017:1, BergmannFitschenPerschSteidl:2017:2}.
However, all of these approaches work with data on two-dimensional Cartesian grids.
While some of our ideas are similar, we exploit the close relation of the mesh to its normal vector and the fact that the normal vector field maps into $\sphere$-valued instead of a general manifold.

\subsection*{Organization}

The structure of this paper is as follows.
In \cref{section:total-generalized-variation-piecewise-constant}, we recap the discrete formulation for piecewise constant data on triangular meshes from \cite{BaumgaertnerBergmannHerzogSchmidtVidalNunez:2023:1}, which employs a Raviart--Thomas finite element function.
Next, in \cref{section:discrete-tgv-of-the-normal-vector-field} we review the total variation (TV) of the normal.
Then, we extend the formulation from \cite{BaumgaertnerBergmannHerzogSchmidtVidalNunez:2023:1} to obtain a novel discrete formulation of total generalized variation of the normal vector.
Therein, we utilize a tailor-made Raviart--Thomas space for the auxiliary function~$\bw$ whose function values represent part of the derivative (push-forward) of the normal vector field, just as $\bw$ in \eqref{eq:tgv:dual} represents part of the derivative of the scalar data~$u$.
We then address the numerical realization of the resulting method in \cref{section:numerical-realization}.
Finally, we present numerical results for mesh denoising problems in \cref{section:numerical-results-for-mesh-denoising}, comparing our approach to \cite{LiuLiWangLiuChen:2022:1} and \cite{ZhangHeWang:2022:1}.

\section{Total Generalized Variation with Piecewise Constant Finite Elements}
\label{section:total-generalized-variation-piecewise-constant}

In this section, we recap the formulation for $\R$-valued data from \cite{BaumgaertnerBergmannHerzogSchmidtVidalNunez:2023:1}.

\subsection{Notation and Finite Element Spaces}
\label{subsection:notation-and-finite-element-spaces}

Let $\mesh$ be a triangulated and orientable surface mesh embedded in $\R^3$.
We denote its set of triangles by~$\triangles$, its edge set by~$\edges$ and its vertex set by~$\vertices$.
Every edge $\edge \in \edges$ is assumed to have exactly two adjacent triangles, which we denote by $\eplus{\triangle}$ and $\eminus{\triangle}$.
This choice is arbitrary but will remain fixed, even after deformations.
On an edge~$\edge$, we define $\eplus{\bmu}$ as the co-normal vector to the triangle $\eplus{\triangle}$, \ie, the unit vector orthogonal to $\edge$ that lies in the plane of $\eplus{\triangle}$, pointing away from $\eplus{\triangle}$.
We define the other co-normal vector~$\eminus{\bmu}$ similarly; see \cref{figure:co-normal-vectors}.

\tdplotsetmaincoords{75}{140}
\begin{figure}[htb]
	\centering
	\begin{subfigure}{0.49\linewidth}
		\centering
		\begin{tikzpicture}[tdplot_main_coords, > = stealth]
			\pgfmathsetmacro{\rotationangleC}{0}
			\pgfmathsetmacro{\rotationangleD}{130}
			\pgfmathsetmacro{\arrowscale}{0.3}

			\pgfmathsetmacro{\Sx}{1.5}
			\pgfmathsetmacro{\Sy}{4}
			\pgfmathsetmacro{\Sz}{0}

			\pgfmathsetmacro{\Ax}{0}
			\pgfmathsetmacro{\Ay}{0}
			\pgfmathsetmacro{\Az}{0}
			\pgfmathsetmacro{\Bx}{3}
			\pgfmathsetmacro{\By}{0}
			\pgfmathsetmacro{\Bz}{0}

			\pgfmathsetmacro{\Cx}{\Sx}
			\pgfmathsetmacro{\Cy}{{cos(\rotationangleC)*\Sy - sin(\rotationangleC)*\Sz}}
			\pgfmathsetmacro{\Cz}{{sin(\rotationangleC)*\Sy + cos(\rotationangleC)*\Sz}}
			\pgfmathsetmacro{\Dx}{\Sx}
			\pgfmathsetmacro{\Dy}{{cos(\rotationangleD)*\Sy - sin(\rotationangleD)*\Sz}}
			\pgfmathsetmacro{\Dz}{{sin(\rotationangleD)*\Sy + cos(\rotationangleD)*\Sz}}

			\pgfmathsetmacro{\ABx}{{\Bx-\Ax}}
			\pgfmathsetmacro{\ABy}{{\By-\Ay}}
			\pgfmathsetmacro{\ABz}{{\Bz-\Az}}
			\pgfmathsetmacro{\ACx}{{\Cx-\Ax}}
			\pgfmathsetmacro{\ACy}{{\Cy-\Ay}}
			\pgfmathsetmacro{\ACz}{{\Cz-\Az}}
			\pgfmathsetmacro{\ADx}{{\Dx-\Ax}}
			\pgfmathsetmacro{\ADy}{{\Dy-\Ay}}
			\pgfmathsetmacro{\ADz}{{\Dz-\Az}}

			\coordinate (A) at (\Ax,\Ay,\Az);
			\coordinate (B) at (\Bx,\By,\Bz);
			\coordinate (M) at ($(A)!0.5!(B)$);

			\coordinate (C) at (\Cx,\Cy,\Cz);
			\coordinate (D) at (\Dx,\Dy,\Dz);

			\tdplotcrossprod(\ABx,\ABy,\ABz)(\ACx,\ACy,\ACz)
			\pgfmathsetmacro{\nplusx}{\arrowscale*\tdplotresx}
			\pgfmathsetmacro{\nplusy}{\arrowscale*\tdplotresy}
			\pgfmathsetmacro{\nplusz}{\arrowscale*\tdplotresz}

			\tdplotcrossprod(\ADx,\ADy,\ADz)(\ABx,\ABy,\ABz)
			\pgfmathsetmacro{\nminusx}{\arrowscale*\tdplotresx}
			\pgfmathsetmacro{\nminusy}{\arrowscale*\tdplotresy}
			\pgfmathsetmacro{\nminusz}{\arrowscale*\tdplotresz}

			\pgfmathsetmacro{\muplusx}{\nplusx}
			\pgfmathsetmacro{\muplusy}{{cos(90)*\nplusy - sin(90)*\nplusz}}
			\pgfmathsetmacro{\muplusz}{{sin(90)*\nplusy + cos(90)*\nplusz}}

			\pgfmathsetmacro{\muminusx}{\nminusx}
			\pgfmathsetmacro{\muminusy}{{cos(-90)*\nminusy - sin(-90)*\nplusz}}
			\pgfmathsetmacro{\muminusz}{{sin(-90)*\nminusy + cos(-90)*\nplusz}}

			\pgfmathsetmacro{\tEx}{3+\arrowscale}
			\pgfmathsetmacro{\tEy}{0}
			\pgfmathsetmacro{\tEz}{0}

			\draw [TolVibrantBlue, ->, very thick] (M) -- ++ (\muplusx,\muplusy,\muplusz) node [anchor = south]{$\eplus{\bmu}$};
			\draw [TolVibrantBlue, ->, very thick] (M) -- ++ (\muminusx,\muminusy,\muminusz) node [anchor = west]{$\eminus{\bmu}$};

			\draw [thick,fill = TolVibrantBlue, opacity = 0.5] (A) -- (B) -- (C) -- cycle;
			\node at (barycentric cs:A=1,B=1,C=1) {$T_{E_+}$};
			\draw [thick,fill = TolVibrantBlue, opacity = 0.5] (A) -- (B) -- (D) -- cycle;
			\node at (barycentric cs:A=1,B=1,D=1) {$T_{E_-}$};

			\draw [TolVibrantRed, ->, very thick] (M) -- ++ (\nplusx,\nplusy,\nplusz) node [anchor = south](N){$\eplus{\bn}$};
			\draw [TolVibrantRed, ->, very thick] (M) -- ++ (\nminusx,\nminusy,\nminusz) node [anchor = west]{$\eminus{\bn}$};

			\draw [TolVibrantOrange, ->, very thick] (M) -- ++ (\tEx,\tEy,\tEz) node [anchor = south]{$\bt_\edge$};

			\pgfmathsetmacro{\dotproductnmu}{\nplusx*\muminusx+\nplusy*\muminusy+\nplusz*\muminusz}
			\pgfmathsetmacro{\signnmu}{sign(\dotproductnmu)}

			\pgfmathsetmacro{\dotproductnn}{\nplusx*\nminusx+\nplusy*\nminusy+\nplusz*\nminusz}
			\pgfmathsetmacro{\normnplus}{sqrt(\nplusx*\nplusx+\nplusy*\nplusy+\nplusz*\nplusz)}
			\pgfmathsetmacro{\normnminus}{sqrt(\nminusx*\nminusx+\nminusy*\nminusy+\nminusz*\nminusz)}

			\pgfmathsetmacro{\angle}{acos(\dotproductnn/\normnplus/\normnminus)}
			\pgfmathsetmacro{\arccos}{acos(\dotproductnn/\normnplus/\normnminus)/360*2*3.1415}

			\coordinate (N) at ($(M)+(\nplusx,\nplusy,\nplusz)$);
			\draw [black,->, very thick] (N) -- ++ (\arccos*\signnmu*\muplusx,\arccos*\signnmu*\muplusy,\arccos*\signnmu*\muplusz) node [anchor = west]{$\logarithm{\eplus{\bn}}{\eminus{\bn}}$};

			\tdplotsetrotatedcoordsorigin{(M)}
			\path [fill = black, fill opacity = 0.2, tdplot_rotated_coords] (M) -- plot [variable = \t, domain = {\rotationangleC+90}:{\rotationangleD-90}](xyz spherical cs:radius = \normnplus, longitude = 0.001, latitude = \t);
		\end{tikzpicture}
	\end{subfigure}
	\hfill
	\begin{subfigure}{0.49\linewidth}
		\begin{tikzpicture}[tdplot_main_coords, > = stealth]
			\pgfmathsetmacro{\rotationangleC}{-50}
			\pgfmathsetmacro{\rotationangleD}{180}
			\pgfmathsetmacro{\arrowscale}{0.3}

			\pgfmathsetmacro{\Sx}{1.5}
			\pgfmathsetmacro{\Sy}{4}
			\pgfmathsetmacro{\Sz}{0}

			\pgfmathsetmacro{\Ax}{0}
			\pgfmathsetmacro{\Ay}{0}
			\pgfmathsetmacro{\Az}{0}
			\pgfmathsetmacro{\Bx}{3}
			\pgfmathsetmacro{\By}{0}
			\pgfmathsetmacro{\Bz}{0}

			\pgfmathsetmacro{\Cx}{\Sx}
			\pgfmathsetmacro{\Cy}{{cos(\rotationangleC)*\Sy - sin(\rotationangleC)*\Sz}}
			\pgfmathsetmacro{\Cz}{{sin(\rotationangleC)*\Sy + cos(\rotationangleC)*\Sz}}
			\pgfmathsetmacro{\Dx}{\Sx}
			\pgfmathsetmacro{\Dy}{{cos(\rotationangleD)*\Sy - sin(\rotationangleD)*\Sz}}
			\pgfmathsetmacro{\Dz}{{sin(\rotationangleD)*\Sy + cos(\rotationangleD)*\Sz}}

			\pgfmathsetmacro{\ABx}{{\Bx-\Ax}}
			\pgfmathsetmacro{\ABy}{{\By-\Ay}}
			\pgfmathsetmacro{\ABz}{{\Bz-\Az}}
			\pgfmathsetmacro{\ACx}{{\Cx-\Ax}}
			\pgfmathsetmacro{\ACy}{{\Cy-\Ay}}
			\pgfmathsetmacro{\ACz}{{\Cz-\Az}}
			\pgfmathsetmacro{\ADx}{{\Dx-\Ax}}
			\pgfmathsetmacro{\ADy}{{\Dy-\Ay}}
			\pgfmathsetmacro{\ADz}{{\Dz-\Az}}

			\coordinate (A) at (\Ax,\Ay,\Az);
			\coordinate (B) at (\Bx,\By,\Bz);
			\coordinate (M) at ($(A)!0.5!(B)$);

			\coordinate (C) at (\Cx,\Cy,\Cz);

			\coordinate (D) at (\Dx,\Dy,\Dz);

			\tdplotcrossprod(\ABx,\ABy,\ABz)(\ACx,\ACy,\ACz)
			\pgfmathsetmacro{\nplusx}{\arrowscale*\tdplotresx}
			\pgfmathsetmacro{\nplusy}{\arrowscale*\tdplotresy}
			\pgfmathsetmacro{\nplusz}{\arrowscale*\tdplotresz}

			\tdplotcrossprod(\ADx,\ADy,\ADz)(\ABx,\ABy,\ABz)
			\pgfmathsetmacro{\nminusx}{\arrowscale*\tdplotresx}
			\pgfmathsetmacro{\nminusy}{\arrowscale*\tdplotresy}
			\pgfmathsetmacro{\nminusz}{\arrowscale*\tdplotresz}

			\pgfmathsetmacro{\muplusx}{\nplusx}
			\pgfmathsetmacro{\muplusy}{{cos(90)*\nplusy - sin(90)*\nplusz}}
			\pgfmathsetmacro{\muplusz}{{sin(90)*\nplusy + cos(90)*\nplusz}}

			\pgfmathsetmacro{\muminusx}{\nminusx}
			\pgfmathsetmacro{\muminusy}{{cos(-90)*\nminusy - sin(-90)*\nplusz}}
			\pgfmathsetmacro{\muminusz}{{sin(-90)*\nminusy + cos(-90)*\nplusz}}

			\pgfmathsetmacro{\tEx}{3+\arrowscale}
			\pgfmathsetmacro{\tEy}{0}
			\pgfmathsetmacro{\tEz}{0}

			\draw [TolVibrantBlue, ->, very thick] (M) -- ++ (\muplusx,\muplusy,\muplusz) node [anchor = south]{$\eplus{\bmu}$};
			\draw [TolVibrantBlue, ->, very thick] (M) -- ++ (\muminusx,\muminusy,\muminusz) node [anchor = south]{$\eminus{\bmu}$};

			\draw [thick,fill = TolVibrantBlue, opacity = 0.5] (A) -- (B) -- (C) -- cycle;
			\node at (barycentric cs:A=1,B=1,C=1) {$T_{E_+}$};
			\draw [thick,fill = TolVibrantBlue, opacity = 0.5] (A) -- (B) -- (D) -- cycle;
			\node at (barycentric cs:A=1,B=1,D=1) {$T_{E_-}$};

			\draw [TolVibrantRed, ->, very thick] (M) -- ++ (\nplusx,\nplusy,\nplusz) node [anchor = west](N){$\eplus{\bn}$};
			\draw [TolVibrantRed, ->, very thick] (M) -- ++ (\nminusx,\nminusy,\nminusz) node [anchor = south]{$\eminus{\bn}$};

			\draw [TolVibrantOrange, ->, very thick] (M) -- ++ (\tEx,\tEy,\tEz) node [anchor = south]{$\bt_\edge$};

			\pgfmathsetmacro{\dotproductnmu}{\nplusx*\muminusx+\nplusy*\muminusy+\nplusz*\muminusz}
			\pgfmathsetmacro{\signnmu}{sign(\dotproductnmu)}

			\pgfmathsetmacro{\dotproductnn}{\nplusx*\nminusx+\nplusy*\nminusy+\nplusz*\nminusz}
			\pgfmathsetmacro{\normnplus}{sqrt(\nplusx*\nplusx+\nplusy*\nplusy+\nplusz*\nplusz)}
			\pgfmathsetmacro{\normnminus}{sqrt(\nminusx*\nminusx+\nminusy*\nminusy+\nminusz*\nminusz)}

			\pgfmathsetmacro{\angle}{acos(\dotproductnn/\normnplus/\normnminus)}
			\pgfmathsetmacro{\arccos}{acos(\dotproductnn/\normnplus/\normnminus)/360*2*3.1415}

			\coordinate (N) at ($(M)+(\nplusx,\nplusy,\nplusz)$);
			\draw [black, ->, very thick] (N) -- ++ (\arccos*\signnmu*\muplusx,\arccos*\signnmu*\muplusy,\arccos*\signnmu*\muplusz) node [anchor = west]{$\logarithm{\eplus{\bn}}{\eminus{\bn}}$};

			\tdplotsetrotatedcoordsorigin{(M)}
			\path [fill = black, fill opacity = 0.2, tdplot_rotated_coords] (M) -- plot [variable = \t, domain = {\rotationangleC+90}:{\rotationangleD-90}](xyz spherical cs:radius = \normnplus, longitude = 0.001, latitude = \t);
		\end{tikzpicture}
	\end{subfigure}
	\caption{%
		Illustration of normals $\eplus{\bn}$ and $\eminus{\bn}$ of two triangles $\eplus{\triangle}$, $\eminus{\triangle}$ sharing an edge~$\edge$.
		The triangles' co-normals are $\eplus{\bmu}$ and $\eminus{\bmu}$ and the unit vector tangent to the edge is $\bt_\edge$.
		The logarithmic map, described in \cref{section:discrete-tgv-of-the-normal-vector-field} is also pictured.
	}
	\label{figure:co-normal-vectors}
\end{figure}

We define the standard discontinuous Galerkin finite element space on $\mesh$ by
\begin{equation}
	\DG{r}(\mesh, \R^n)
	\coloneqq
	\setDef[Big]{\bu \colon \bigcup_{\triangle \in \triangles} \triangle \to \R^n}{\restr{\bu}{\triangle} \in P_r(\triangle, \R^n) \text{ for all } \triangle \in \triangles}
	,
\end{equation}
where $P_r(\triangle,V)$ is the set of all polynomials defined on~$\triangle$ of maximum degree~$r$ with values in some vector space~$V$.
The restriction of a function $\bu \in \DG{r}(\mesh, V)$ to a triangle~$\triangle \in \triangles$ is denoted by~$\bu_\triangle$.
Likewise, for an edge $\edge \in \edges$, we denote the restriction of~$\bu$ to $\eplus{\triangle}$ by $\eplus{\bu}$, and the restriction of $\bu$ to $\eminus{\triangle}$ by $\eminus{\bu}$
The jump of~$\bu$ across an edge~$\edge$ is denoted by $\jump{\bu}_\edge \coloneqq \eplus{\bu} - \eminus{\bu}$.

Furthermore, we define the finite element space on the skeleton of the mesh as
\begin{equation}
	\DG{r}(\edges, \R^n)
	\coloneqq
	\setDef[Big]{\bu \colon \bigcup_{\edge \in \edges} \edge \to \R^n}{\restr{\bu}{\edge} \in P_r(\edge, \R^n) \text{ for all } \edge \in \edges}
	,
\end{equation}
where $P_r(\edge,V)$ is the set of all polynomials defined on~$\edge$ of maximum degree~$r$ with values in some vector space~$V$.

A key ingredient to the $\TGV$ formulation from \cite{BaumgaertnerBergmannHerzogSchmidtVidalNunez:2023:1} for piecewise constant functions is the lowest-order Raviart--Thomas finite element space~$\RT{0}$.
In the case of a planar (2D) mesh, $\RT{0}$ is defined as the smallest $H(\div)$-conforming space that maps the divergence surjectively onto $\DG{0}$.
The $H(\div)$-conformity is equivalent to the continuity of the co-normal component across the edges of the mesh.
As described in \cite{RognesHamCotterMcRae:2013:1,HerrmannHerzogKroenerSchmidtVidalNunez:2018:1}, the space can be generalized to triangular meshes~$\mesh$ embedded in $\R^3$ by using piecewise polynomial functions with the same basis functions on the reference element as in the planar case.
The requirement of $H(\div)$-conformity then becomes
\begin{equation}
	\label{eq:RT0:conormal-continuity}
	\dotproduct{\eplus{\bw}}{\eplus{\bmu}}
	=
	- \dotproduct{\eminus{\bw}}{\eminus{\bmu}}
\end{equation}
on all edges~$\edge$, and we can obtain the following description of the lowest-order Raviart--Thomas space on $\mesh$:
\begin{equation}
	\label{eq:RT0-space}
	\RT{0}(\mesh, \R^3)
	\coloneqq
	\setDef[auto]{\bw \in \DG{1}(\mesh, \R^3)}{%
		\begin{aligned}
			&
			\restr{\bw}{\triangle} \in P_0(\triangle, \tangentSpace{\triangle}[\mesh]) + (\bx - \bx_\triangle) \, P_0(\triangle, \R)
			\\
			&
			\text{and }
			\jump{\dotproduct{\bw}{\bmu}}_\edge
			=
			0
			\text{ for all }
			\triangle \in \triangles
			\text{ and }
			\edge \in \edges
		\end{aligned}
	}
	.
\end{equation}
Here $\bx$ denotes the spatial coordinate on~$\mesh$, $\bx_\triangle$ is a fixed reference point in~$\triangle$, and $\tangentSpace{\triangle}[\mesh]$ is the common tangent space to $\mesh$ at all points in~$\triangle$.
Notice that the function values $\restr{\bw}{\triangle}$ belong to $\tangentSpace{\triangle}[\mesh]$.
The co-normal continuity \eqref{eq:RT0:conormal-continuity} is conveniently realized by choosing
\begin{equation}
	\label{eq:RT0:degrees-of-freedom}
	\int_\edge \dotproduct{\eplus{\bw}}{\eplus{\bmu}} \dS
	=
	- \int_\edge \dotproduct{\eminus{\bw}}{\eminus{\bmu}} \dS
\end{equation}
as the global degrees of freedom, which results in the following choice of basis functions for the space \eqref{eq:RT0-space}:
\begin{equation*}
	\Phi_\edge(\bx)
	\coloneqq
	\begin{cases}
		\tfrac{1}{2 \, \abs{\eplus{\triangle}}} (\bx - \eplus{\bp})
		&
		\text{ if }
		\bx \in \eplus{\triangle}
		,
		\\
		\tfrac{-1}{2 \, \abs{\eminus{\triangle}}} (\bx - \eminus{\bp})
		&
		\text{ if }
		\bx \in \eminus{\triangle}
		,
		\\
		0
		&
		\text{ else}
		,
	\end{cases}
\end{equation*}
with $\bp_{\edge_\pm}$ denoting the coordinate of the vertex of $\triangle_{\edge_\pm}$ opposite to~$\edge$.

\subsection{Discrete TGV for Piecewise Constant Functions}
\label{subsection:discrete-tgv-for-piecewise-constant-functions}

The first-order total variation of a piecewise constant function $u \in \DG{0}(\mesh, \R)$ amounts to
\begin{equation}
	\label{eq:tv:DG0}
	\TV(u)
	=
	\sum_{\edge \in \edges} \abs{\jump{u}_\edge} \dS
	=
	\sum_{\edge \in \edges} \abs{\edge} \, \abs{\jump{u}_\edge}
	.
\end{equation}
As we have shown in~\cite[eq.~(2.10)]{BaumgaertnerBergmannHerzogSchmidtVidalNunez:2023:1}, the second-order total generalized variation seminorm \eqref{eq:tgv:dual} reduces to TV \eqref{eq:tv:DG0} and therefore offers no advantage.
To overcome this, a discrete adaptation of $\TGV$ is required.
In \eqref{eq:tgv:dual}, the $\alpha_1$-term couples the gradient of $u$ to the auxiliary variable~$\bw$.
When $u$ is piecewise constant, the gradient information is concentrated on the edges in form of the jump $\jump{u}$.
We proposed to couple this scalar value of $\jump{u}_\edge$ on an edge~$\edge$ to the degree of freedom located on~$\edge$ of a Raviart--Thomas function $\bw \in \RT{0}(\mesh, \R^3)$ as in \eqref{eq:RT0-space}--\eqref{eq:RT0:degrees-of-freedom}.
The $\alpha_0$-term then measures the (discrete) total variation of the auxiliary variable $\bw \in \RT{0}(\mesh, \R^3)$, leading to a concept of discretely linear, piecewise constant functions.
Overall, the formulation proposed in \cite{BaumgaertnerBergmannHerzogSchmidtVidalNunez:2023:1} reads
\begin{multline}
	\label{eq:tgv:DG0}
	\fetgv_{(\alpha_0, \alpha_1)}^2(u)
	\coloneqq
	\min_{\bw \in \RT{0}(\mesh, \R^3)}
	\alpha_1 \sum_{\edge \in \edges} \int_\edge \abs[big]{\jump{u}_\edge + h_\edge \, \dotproduct{\eplus{\bw}}{\eplus{\bmu}}} \dS
	\\
	+
	\alpha_0 \sum_{\triangle \in \triangles} \int_\triangle \abs{\nabla \bw_\triangle}_F \dx
	+
	\alpha_0 \sum_{\edge \in \edges} \int_\edge \interpolate[big]{1}{\abs[big]{\jump{\bw}}_2} \dS
	,
\end{multline}
where $\abs{\,\cdot\,}_F$ is the Frobenius norm of a matrix and $\interpolate{1}{\cdot}$ denotes the linear interpolation at the endpoints of an edge, denoted by $X_{\edge, 1}$ and $X_{\edge, 2}$.
Furthermore, $h_\edge$ denotes a mesh-dependent factor chosen as the distance between the circumcenters of the two adjacent triangles sharing the edge~$\edge$.
Therefore, $\jump{u}_\edge/h_\edge$ is a finite difference that corresponds to the directional derivative of $u$ in direction $- \eplus{\bmu}$.
For more details we refer the reader to~\cite[Section~3]{BaumgaertnerBergmannHerzogSchmidtVidalNunez:2023:1}.

\section{Discrete Total Generalized Variation of the Normal}
\label{section:discrete-tgv-of-the-normal-vector-field}

In this section we extend the discrete total generalized variation for piecewise constant functions \eqref{eq:tgv:DG0} to the piecewise constant unit normal vector field~$\bn$ on a triangular mesh~$\mesh$ embedded in~$\R^3$.
Unlike the methods proposed in \cite{LiuLiWangLiuChen:2022:1, ZhangHeWang:2022:1} for $\TGV$ mesh denoising, we consider~$\bn$ with values in the manifold~$\sphere$ rather than in~$\R^3$.
This has significant implications on the auxiliary variable~$\bw$ in \eqref{eq:tgv:DG0}, which is responsible for capturing changes in the data, in this case, in the normal vector.
Before defining the proposed formulation in \cref{subsection:discrete-tgv-of-the-normal-vector-field}, we review some elementary geometric calculus for the sphere in \cref{subsection:geometric-calculus-for-the-sphere-and-identities-on-triangulated-meshes}.
In \cref{subsection:discrete-tv-of-the-normal-vector-field}, we revisit the first-order total variation of the normal vector field, and then we define the tailored tangential Raviart--Thomas space in \cref{subsection:tangential-Raviart-Thomas-space} that captures derivative information of the normal vector field.

\subsection{Geometric Calculus for the Sphere and Identities on Triangulated Meshes}
\label{subsection:geometric-calculus-for-the-sphere-and-identities-on-triangulated-meshes}

We briefly recall some basic concepts on the Riemannian manifold $\sphere$, the $2$-sphere, in the context of the normal vector of a triangulated mesh embedded in~$\R^3$, following \cite[Appendix]{BergmannHerrmannHerzogSchmidtVidalNunez:2020:1}.
Given two vectors $\bn_1, \bn_2 \in \sphere$ with $\bn_1 \neq - \bn_2$, the logarithmic map is given as
\begin{equation}
	\label{eq:sphere:logarithmic-map}
	\logarithm{\bn_1}{\bn_2}
	=
	\begin{cases}
		\bnull
		&
		\text{ if }
		\bn_1 = \bn_2
		\\
		\displaystyle \dist{\bn_1}{\bn_2} \frac{\bn_2 - (\dotproduct{\bn_1}{\bn_2}) \, \bn_1}{\abs{\bn_2 - (\dotproduct{\bn_1}{\bn_2}) \, \bn_1}_2}
		&
		\text{ else}
		,
	\end{cases}
\end{equation}
where
\begin{equation*}
	\dist{\bn_1}{\bn_2}
	\coloneqq
	\arccos(\dotproduct{\bn_1}{\bn_2})
\end{equation*}
is the geodesic distance on~$\sphere$.
The logarithmic map is the vector in the tangent space $\tangentSpace{\bn_1}[\sphere]$ pointing from $\bn_1$ to $\bn_2$ and of length~$\dist{\bn_1}{\bn_2}$.
It also enters the so-called parallel transport, which transforms a vector $\bxi \in \tangentSpace{\bn_1}[\sphere]$ to a vector in $\tangentSpace{\bn_2}[\sphere]$ along the shortest geodesic (assuming $\bn_1 \neq - \bn_2$) by
\begin{subequations}
	\label{eq:sphere:parallel-transport}
	\begin{align}
		\parallelTransport{\bn_1}{\bn_2}(\bxi)
		&
		=
		\begin{cases}
			\bxi
			&
			\text{ if }
			\bn_1 = \bn_2
			\\
			\bxi
			-
			\frac{\dotproduct{\bxi}{\logarithm{\bn_1}{\bn_2}}}{\dist{\bn_1}{\bn_2}^2} \paren[auto](){\logarithm{\bn_1}{\bn_2} +\logarithm{\bn_2}{\bn_1}}
			&
			\text{ else}
		\end{cases}
		\label{eq:sphere:parallel-transport:1}
		\\
		&
		=
		\paren[auto](){\id - \frac{\bn_2 + \bn_1}{1 + \dotproduct{\bn_2}{\bn_1}} \, \bn_2^\transp} \bxi
		.
		\label{eq:sphere:parallel-transport:2}
	\end{align}
\end{subequations}
When $\bn_1 = \bn_2$, the equality in between \eqref{eq:sphere:parallel-transport:1} and \eqref{eq:sphere:parallel-transport:2} is obvious due to $\dotproduct{\bn_2}{\bxi} = 0$.
Otherwise, using the definition of the logarithmic map \eqref{eq:sphere:logarithmic-map} and expanding the norms yields
\begin{align*}
	\MoveEqLeft
	\bxi
	-
	\frac{\dotproduct{\bxi}{\logarithm{\bn_1}{\bn_2}}}{\dist{\bn_1}{\bn_2}^2} \paren[auto](){\logarithm{\bn_1}{\bn_2} +\logarithm{\bn_2}{\bn_1}}
	\\
	&
	=
	\bxi
	-
	\dotproduct{\bxi}{\frac{\bn_2-(\dotproduct{\bn_2}{\bn_1}) \, \bn_1}{\abs{\bn_2-(\dotproduct{\bn_2}{\bn_1}) \, \bn_1}_2}} \paren[auto](){\frac{\bn_2-(\dotproduct{\bn_2}{\bn_1}) \, \bn_1}{\abs{\bn_2-(\dotproduct{\bn_2}{\bn_1}) \, \bn_1}_2} + \frac{\bn_1 - (\dotproduct{\bn_2}{\bn_1}) \, \bn_2}{\abs{\bn_1 - (\dotproduct{\bn_2}{\bn_1}) \, \bn_2}_2}}
	\\
	&
	=
	\bxi
	-
	\dotproduct{\bxi}{\frac{\bn_2}{\sqrt{1 - (\dotproduct{\bn_2}{\bn_1})^2}}} \frac{(1 - \dotproduct{\bn_2}{\bn_1})(\bn_2 + \bn_1)}{\sqrt{1 - (\dotproduct{\bn_2}{\bn_1})^2}}
	\\
	&
	=
	\bxi
	-
	\dotproduct{\bxi}{\bn_2} \frac{\bn_2 + \bn_1}{1 + \dotproduct{\bn_2}{\bn_1}}
	\\
	&
	=
	\paren[auto](){\id - \frac{\bn_2 + \bn_1}{1 + \dotproduct{\bn_2}{\bn_1}} \, \bn_2^\transp} \bxi
	.
\end{align*}

To make use of these definitions on triangular meshes, define a unit vector $\bt_\edge$, tangential to an edge $\edge\in \edges$ with arbitrary but fixed orientation.
Then, $\set{\eplus{\bn}, \eplus{\bmu}, \bt_\edge}$ forms an orthonormal basis of~$\R^3$ \wrt to the standard inner product at a point on an edge~$\edge$.
Analogously, $\set{\eminus{\bn}, \eminus{\bmu}, \bt_\edge}$ also forms an orthonormal basis of~$\R^3$.
This setup is illustrated in \cref{figure:co-normal-vectors}.

Using this property, the logarithmic map between two normal vectors of adjacent triangles can be simplified.
\begin{lemma}
	\label{lemma:sphere:logarithmic-map:using-co-normals}
	Let $\edge$ be the edge shared by the triangles $\eplus{\triangle}$, $\eminus{\triangle}$ with respective normal vectors $\eplus{\bn}$, $\eminus{\bn}$ and co-normal vectors $\eplus{\bmu}$, $\eminus{\bmu}$.
	Then
	\begin{equation}
		\label{eq:sphere:logarithmic-map:using-co-normals}
		\begin{aligned}
			\logarithm{\eplus{\bn}}{\eminus{\bn}}
			&
			=
			\sign \paren[big](){\dotproduct{\eminus{\bn}}{\eplus{\bmu}}} \, \dist{\eplus{\bn}}{\eminus{\bn}} \, \eplus{\bmu}
			,
			\\
			\logarithm{\eminus{\bn}}{\eplus{\bn}}
			&
			=
			\sign \paren[big](){\dotproduct{\eminus{\bn}}{\eplus{\bmu}}} \, \dist{\eplus{\bn}}{\eminus{\bn}} \, \eminus{\bmu}
			.
		\end{aligned}
	\end{equation}
\end{lemma}
\begin{proof}
	We start with the first identity and exclude the obvious case $\eplus{\bn} = \eminus{\bn}$.
	Then, the logarithmic map from \eqref{eq:sphere:logarithmic-map}, up to a scaling factor, is
	\begin{equation}
		\label{eq:sphere:logarithmic-map:up-to-scaling}
		\frac{\eminus{\bn} - (\dotproduct{\eplus{\bn}}{\eminus{\bn}}) \, \eplus{\bn}}{\abs{\eminus{\bn} - (\dotproduct{\eplus{\bn}}{\eminus{\bn}}) \, \eplus{\bn}}_2}
		.
	\end{equation}
	It is easy to see that \eqref{eq:sphere:logarithmic-map:up-to-scaling} is orthogonal to $\eplus{\bn}$ and $\bt_\edge$.
	Since $\set{\eplus{\bn},\eplus{\bmu},\bt_\edge}$ form an orthonormal basis of~$\R^3$, we have
	\begin{equation*}
		\frac{\eminus{\bn} - (\dotproduct{\eplus{\bn}}{\eminus{\bn}}) \, \eplus{\bn}}{\abs{\eminus{\bn} - (\dotproduct{\eplus{\bn}}{\eminus{\bn}}) \, \eplus{\bn}}_2}
		=
		\sigma \, \eplus{\bmu}
	\end{equation*}
	for some $\sigma \in \R$.
	Since the left side has norm one, $\sigma \in \set{-1, 1}$.
	Taking the inner product with $\eplus{\bmu}$ on both sides yields $\sign(\dotproduct{\eminus{\bn}}{\eplus{\bmu}}) = \sigma$.
	Plugging these identities into the definition of $\logarithm{\eplus{\bn}}{\eminus{\bn}}$ \eqref{eq:sphere:logarithmic-map} shows the desired first identity in \eqref{eq:sphere:logarithmic-map:using-co-normals}.
	To show the second, we can swap $\eplus{\triangle}$ and $\eminus{\triangle}$.
	There we can proceed analogously to obtain
	\begin{equation*}
		\logarithm{\eminus{\bn}}{\eplus{\bn}}
		=
		\sign \paren[auto](){\dotproduct{\eplus{\bn}}{\eminus{\bmu}}} \dist{\eplus{\bn}}{\eminus{\bn}} \eminus{\bmu}
		.
	\end{equation*}
	It remains to show $\dotproduct{\eplus{\bn}}{\eminus{\bmu}} = \dotproduct{\eminus{\bn}}{\eplus{\bmu}}$, for which we use the help of the orthogonal matrix $Q = \eminus{\bn}^{} \eminus{\bmu}^\transp - \eminus{\bmu}^{} \eminus{\bn}^\transp + \bt_\edge^{} \, \bt_\edge^\transp$.
	Using the triple vector product, we obtain
	\begin{align*}
		\dotproduct{\eplus{\bn}}{\eminus{\bmu}}
		&
		=
		\dotproduct{(Q \, \eplus{\bn})}{(Q \, \eminus{\bmu})}
		\\
		&
		=
		\dotproduct{\paren[big](){\eminus{\bn}^{} \eminus{\bmu}^\transp \eplus{\bn}^{} - \eminus{\bmu}^{} \eminus{\bn}^\transp \eplus{\bn}^{}}}{\eminus{\bn}^{}}
		\\
		&
		=
		\dotproduct{\paren[big](){\eplus{\bn} \times (\eminus{\bn} \times \eminus{\bmu})}}{\eminus{\bn}^{}}
		\\
		&
		=
		\dotproduct{\eminus{\bmu}}{\eplus{\bn}^{}}
		.
	\end{align*}
\end{proof}
Having established these identities for the logarithmic maps between normals of adjacent triangles, the parallel transport \eqref{eq:sphere:parallel-transport} for this situation can be significantly simplified.
\begin{lemma}
	\label{lemma:sphere:parallel-transport:using-co-normals}
	The parallel transport \eqref{eq:sphere:parallel-transport} between vectors in adjacent tangent spaces $\tangentplus$ and $\tangentminus$ is given by
	\begin{equation}
		\label{eq:sphere:parallel-transport:adjacent-tangent-spaces}
		\begin{aligned}
			\parallelTransport{\eminus{\bn}}{\eplus{\bn}}(\bxi)
			&
			=
			\paren[big](){\id - \eminus{\bmu} \eminus{\bmu}^\transp - \eplus{\bmu} \eminus{\bmu}^\transp} \, \bxi
			,
			\\
			\parallelTransport{\eplus{\bn}}{\eminus{\bn}}(\bchi)
			&
			=
			\paren[big](){\id - \eplus{\bmu} \eplus{\bmu}^\transp - \eminus{\bmu} \eplus{\bmu}^\transp} \, \bchi
			,
		\end{aligned}
	\end{equation}
	where $\bxi \in \tangentminus$ and $\bchi \in \tangentplus$ respectively.
	In particular, we have
	\begin{equation}
		\label{eq:sphere:parallel-transport:transporting-co-normals}
		\parallelTransport{\eminus{\bn}}{\eplus{\bn}}(\eminus{\bmu})
		=
		- \eplus{\bmu}
		\quad
		\text{and}
		\quad
		\parallelTransport{\eplus{\bn}}{\eminus{\bn}}(\eplus{\bmu})
		=
		- \eminus{\bmu}
		.
	\end{equation}
\end{lemma}
\begin{proof}
	In the case of $\eplus{\bn} = \eminus{\bn}$ the parallel transport is an identity and we have $\eplus{\bmu} = - \eminus{\bmu}$, which gives the desired result.
	Otherwise, plugging in the result of \cref{lemma:sphere:logarithmic-map:using-co-normals} into \eqref{eq:sphere:parallel-transport} yields
	\begin{align*}
		\parallelTransport{\eminus{\bn}}{\eplus{\bn}}(\bxi)
		&
		=
		\bxi
		- \sign \paren[auto](){\dotproduct{\eminus{\bn}}{\eplus{\bmu}}}^2 \dist{\eplus{\bn}}{\eminus{\bn}}^2 \frac{\dotproduct{\bxi}{\eminus{\bmu}}}{\dist{\eplus{\bn}}{\eminus{\bn}}^2} \paren[auto](){\eminus{\bmu} + \eplus{\bmu}}
		\\
		&
		=
		\bxi - \dotproduct{\bxi}{\eminus{\bmu}} \paren[auto](){\eminus{\bmu} + \eplus{\bmu}}
		\\
		&
		=
		\paren[big](){\id - \eminus{\bmu} \eminus{\bmu}^\transp - \eplus{\bmu} \eminus{\bmu}^\transp} \, \bxi
		,
	\end{align*}
	and an analogous result is obtained for $\parallelTransport{\eplus{\bn}}{\eminus{\bn}}(\bchi)$.
	The identities \eqref{eq:sphere:parallel-transport:transporting-co-normals} for the parallel transport of co-normals $\bmu_{\edge_\pm}$ follow immediately from plugging in $\bxi = \eminus{\bmu}$ and $\bchi = \eplus{\bmu}$ into \eqref{eq:sphere:parallel-transport:adjacent-tangent-spaces}.
\end{proof}

\subsection{Discrete Total Variation of the Normal Vector Field}
\label{subsection:discrete-tv-of-the-normal-vector-field}

The normal vector on a triangular mesh embedded in $\R^3$ is constant on each triangle, \ie, $\bn \in \DG{0}(\mesh, \sphere)$.
Therefore the variation of the normal is concentrated on the edges.
In this manifold-valued setting, the total variation of the normal is defined as
\begin{equation}
	\label{eq:tv:normal}
	\TV_\sphere(\bn)
	=
	\sum_{\edge \in \edges} \int_\edge \dist{\eplus{\bn}}{\eminus{\bn}} \dS
	,
\end{equation}
\ie, the absolute value of the difference for scalar-valued data \eqref{eq:tv:DG0} is replaced by the geodesic distance $\dist{\cdot}{\cdot}$; see for instance \cite{LellmannStrekalovskiyKoetterCremers:2013:1}.
Generally, the geodesic distance can be expressed using the norm on the tangent space of the logarithmic map; see \cite{BergmannHerrmannHerzogSchmidtVidalNunez:2020:2}.
Hence,
\begin{equation}
	\label{eq:tv:normal:alternative}
	\TV_\sphere(\bn)
	=
	\sum_{\edge \in \edges} \int_\edge \abs[big]{\logarithm{\eplus{\bn}}{\eminus{\bn}} }_2 \dS
	.
\end{equation}
Comparing this to the total variation in the scalar-valued setting~\eqref{eq:tv:DG0}, it can be observed that $\logarithm{\eplus{\bn}}{\eminus{\bn}}$ takes the role of the jump in \eqref{eq:tv:DG0}.
Indeed the logarithmic map can be conceived as a generalization of the difference.

In order to pass to the discrete formulation \eqref{eq:tgv:DG0} of second-order $\TGV$, we observe that the jump $\jump{u}$ is coupled to the co-normal component $\dotproduct{\eplus{\bw}}{\eplus{\bmu}}$ of the Raviart--Thomas function~$\bw$ in the $\alpha_1$-term.
To extend this formulation to the case of the normal vector, we need to replace $\jump{u}_\edge$ by $\logarithm{\eplus{\bn}}{\eminus{\bn}}$, which carries the information about the variations of neighboring normal vectors.
We thus need to couple $\logarithm{\eplus{\bn}}{\eminus{\bn}}$ to an auxiliary variable from a Raviart--Thomas space with a degree of freedom located also on the edge~$\edge$.
However, the logarithmic map is tangent space-valued, therefore a tangent space-valued Raviart--Thomas space is required to adapt \eqref{eq:tgv:DG0} to the normal vector.
Such a space is non-standard and constructed in the following two subsections.

\subsection{First- and Second-Order Derivatives of Normal Vector Fields on Surfaces}
\label{subsection:first-and-second-order-derivatives-of-normal-vector-fields-on-surfaces}

In order to motivate what follows, we need to briefly discuss first- and second-order derivatives of normal vectors of manifolds.
To this end, suppose that $\mesh$ is a smooth submanifold of $\R^3$ equipped with the parallel transport $\parallelTransportCurve{\gamma}{t} \colon \tangentSpace{\gamma(0)}[\mesh] \to \tangentSpace{\gamma(t)}[\mesh]$ along smooth curves $\gamma$ on~$\mesh$ compatible with the Euclidean metric in~$\R^3$.
Furthermore, let $\bn \colon \mesh \to \sphere$ be the normal vector field of $\mesh$.

The first-order derivative (or push-forward) of~$\bn$ at a point $\bx \in \mesh$, denoted by $\bn_\bx'$, is a linear mapping from the domain tangent space $\tangentSpace{\bx}[\mesh]$ to the co-domain tangent space $\tangentSpace{\bn(\bx)}[\sphere]$.
We denote such mappings by $L(\tangentSpace{\bx}[\mesh],\tangentSpace{\bn(\bx)}[\sphere])$.

In contrast to derivatives for functions in linear spaces, second-order derivatives cannot be defined in a straightforward manner in an iterated fashion.
The reason is that, when $\bx \in \mesh$ is perturbed slightly to $\bar{\bx}$, the tangent space changes as well and $\bn_{\bar{\bx}}' \in L(\tangentSpace{\bar{\bx}}[\mesh],\tangentSpace{\bn(\bar{\bx})}[\sphere])$.
Hence, $\bn_\bx'$ and $\bn_{\bar{\bx}}'$ belong to different spaces.
To define a second-order derivative of the normal vector field~$\bn$, it is convenient to use an identification of $\tangentSpace{\bn(\bx)}[\sphere]$ with $\tangentSpace{\bx}[\mesh]$. This is visualized in \cref{figure:identification-of-tangent-spaces} and is discussed, for instance, in \cite[Sec.~2.2]{BergmannHerrmannHerzogSchmidtVidalNunez:2020:1}.

\begin{figure}[htp]
	\centering
	\begin{adjustbox}{trim = 0pt 0pt 0pt 100pt, clip}
		\begin{tikzpicture}
			\begin{axis}[
				scale = 1.5,
				view = {315}{40},
				hide axis,
				enlargelimits = false,
				xmin = -0.75,
				xmax = 1.075,
				ymin = -0.75,
				ymax = 1.075,
				zmin = 0.45,
				zmax = 1.75,
				]
				\addplot3[surfaceBoundary, domain = -0.7:0.7, samples = 40, samples y = 0,
					mark = none](0.7, x, {\pz{0.7}{x}});
				\addplot3[surfaceBoundary, domain = -0.7:0.7, samples = 40, samples y = 0,
					mark = none](x, 0.7, {\pz{x}{0.7}});
				\addplot3[surf, shader = flat, surface, samples = 40, domain = -0.7:0.7, y domain = -0.7:0.7] {\pz{x}{y}};
				\addplot3[surfaceBoundary, domain = -0.7:0.7, samples = 40, samples y = 0,
					mark = none](x, -0.7, { \pz{x}{-0.7}});
				\addplot3[surfaceBoundary, domain = -0.7:0.7, samples = 40, samples y = 0,
					mark = none](-0.7, x, { \pz{-0.7}{x} });
				\addplot3[curve, domain = -0.5:0.5, samples = 40, samples y = 0,
					mark = none](-x, x, { \pz{-x}{x} });
				\pgfplotsinvokeforeach{-.45, .45, 0}{%
					\draw[plane] \p{#1}{-#1} ++\txUnitScaled{#1}{#1}{.2} %
						-- ++\tyUnitScaled{#1}{-#1}{.2} %
						-- ++\txUnitScaled{#1}{-#1}{-.4} -- ++\tyUnitScaled{#1}{-#1}{-.4}%
						-- ++\txUnitScaled{#1}{-#1}{.4} -- ++\tyUnitScaled{#1}{-#1}{.2};
					\draw[tang] \p{#1}{-#1} -- ++\txUnitScaled{#1}{-#1}{.2};
					\draw[tang] \p{#1}{-#1} -- ++\tyUnitScaled{#1}{-#1}{.2};
					\draw[orth] \p{#1}{-#1} -- ++\tNUnitScaled{#1}{-#1}{.2};
					\addplot3[surf, shader = flat,
						domain = 0:360, domain y = 0:90, samples = 36, samples y = 18,
						z buffer = sort, sphere]
						({0.2*sin(x)*cos(y)}, {0.2*cos(x)*cos(y)}, {\pz{0}{0}+0.2*sin(y)});
					\draw[plane2] \p{#1}{-#1} ++\tNUnitScaled{#1}{-#1}{.2}
						++\txUnitScaled{#1}{-#1}{.2} %
						-- ++\tyUnitScaled{#1}{-#1}{.2} %
						-- ++\txUnitScaled{#1}{-#1}{-.4} -- ++\tyUnitScaled{#1}{-#1}{-.4}%
						-- ++\txUnitScaled{#1}{-#1}{.4} -- ++\tyUnitScaled{#1}{-#1}{.2};
				}
			\end{axis}
		\end{tikzpicture}
	\end{adjustbox}
	\caption{%
		Visualization of the relation between $\tangentSpace{\bn(\bx)}[\sphere]$ and $\tangentSpace{\bx}[\mesh]$.
		Adapted from \cite[Fig.~1]{BergmannHerrmannHerzogSchmidtVidalNunez:2020:1}.
	}
	\label{figure:identification-of-tangent-spaces}
\end{figure}
Consequently, $\bn_\bx'$ can be treated as an element of $L(\tangentSpace{\bx}[\mesh],\tangentSpace{\bx}[\mesh])$ for the purpose of differentiation.
Let $\tangentBundle[\mesh]$ denote the tangent bundle and $\cotangentBundle[\, \mesh]$ denote the cotangent bundle on~$\mesh$.
The function $\bn' \colon \tangentBundle[\mesh] \to \tangentBundle[\mesh]$ can be seen as a $(1,1)$-tensor field~$T$ on~$\mesh$, \ie, a bilinear map with one argument in the cotangent bundle and one argument in the tangent bundle.
More precisely, the relation between~$\bn'$ and $T$ is given by $T_\bx[\bTheta, \bXi] \coloneqq \bTheta_x[\bn_x'[\bXi_x]]$ where $\bTheta \in \cotangentBundle[\, \mesh]$ and $\bXi \in \tangentBundle[\mesh]$.

As described in \cite[Lemma~4.6]{Lee:1997:1}, such a $(1,1)$-tensor field $T$ can be differentiated using covariant derivatives. Given a smooth co-vector field $\bTheta \in \cotangentBundle[\, \mesh]$ and smooth vector fields $\bChi,\bXi \in \tangentBundle[\mesh]$, the derivative of $T \colon \cotangentBundle[\,\mesh] \times \tangentBundle[\mesh] \to \R$ in direction $\bChi$ is given by
\begin{equation}
	\label{eq:covariant-derivative-of-a-tensor}
	\covariantDerivative{\bChi}{T}[\bTheta, \bXi]
	\coloneqq
	(T[\bTheta, \bXi])'[\bChi] - T[\covariantDerivative{\bChi}{\bTheta}, \bXi] - T[\bTheta, \covariantDerivative{\bChi}{\bXi}]
	,
\end{equation}
where $\covariantDerivative{X}{\bTheta}$ and $\covariantDerivative{X}{\bXi}$ are the covariant derivatives induced by the parallel transport on~$\mesh$ as for instance described in \cite[Definition~4.1.2]{Jost:2017:1}.

Through this construction, the value of $\covariantDerivative{\bChi}{T}[\bTheta, \bXi]$ at $\bx \in \mesh$ is computed only from values of~$T$ along the curve $\gamma \colon \interval(){-\varepsilon}{\varepsilon} \to \mesh$ with $\gamma(0) = \bx$ and $\dot{\gamma}(t) = \bChi_{\gamma(t)}$.
In fact, $\covariantDerivative{\bChi}{T}[\bTheta, \bXi]$ at~$\bx$ only depends on $\bTheta_\bx$, $\bChi_\bx$ and $\bXi_\bx$, which means that $\bTheta$ and $\bXi$ can be defined as extensions of quantities $\btheta \in \cotangentSpace{\bx}[\,\mesh]$ and $\bxi \in \tangentSpace{\bx}[\mesh]$ in an arbitrary but smooth way along the curve~$\gamma$.
The specific choice of such an extension $\bTheta_{\gamma(t)} \coloneqq \btheta \circ \parallelTransportCurve{\gamma}{t}^{-1}$ and $\bXi_{\gamma(t)} = \parallelTransportCurve{\gamma}{t}(\bxi)$ achieves that \eqref{eq:covariant-derivative-of-a-tensor} is simplified, because $\covariantDerivative{\bChi}{\bTheta} = 0$ and $\covariantDerivative{\bChi}{\bXi}= 0$ holds at~$\bx$.
Overall, the derivative of a $(1,1)$-tensor field~$T$ at a point~$\bx$ in direction~$\bchi \in \tangentSpace{\bx}[\mesh]$ is given by
\begin{equation}
	\covariantDerivative{\bchi}{T}[\btheta, \bxi]
	=
	\lim_{t \to 0}
	\frac{T_{\gamma(t)}\paren[big][]{\btheta \circ \parallelTransportCurve{\gamma}{t}^{-1}, \parallelTransportCurve{\gamma}{t}(\bxi)} - T_\bx[\btheta, \bxi]}{t}
	,
\end{equation}
where $\gamma \colon (-\varepsilon, \varepsilon) \to \mesh$ is a smooth curve with $\gamma(0) = \bx$ and $\dot{\gamma}(\bx) = \bchi$.

Going back to the normal vector $\bn' \colon \tangentBundle[\mesh] \to \tangentBundle[\mesh]$, it is possible to drop $\btheta$ by using $(\cotangentSpace{\bx}[\mesh])^* \cong \tangentSpace{\bx}[\mesh]$.
Hence, we define
\begin{equation}
	\label{eq:second-derivative-of-the-normal}
	\covariantDerivative{\bchi}\bn'[\bxi]
	\coloneqq
	\lim_{t \to 0}
	\frac{\parallelTransportCurve{\gamma}{t}^{-1}\paren[big](){\bn_{\gamma(t)}'[\parallelTransportCurve{\gamma}{t}(\bxi)]}
		-
		\bn_\bx'[\bxi]
	}{t}
	.
\end{equation}

We can use $\covariantDerivative{\bchi}\bn'[\bxi]$ to define the analog of a linear function $\mesh \to \R$, \ie, of a function whose second derivative vanished.
The normal vector field $\bn \colon \mesh \to \sphere$ can be considered \enquote{linear} in a region~$\mesh_0$ with non-empty interior in case $\covariantDerivative{\bchi}\bn'[\bxi] = \bnull$ holds for all $\bchi, \bxi \in \tangentSpace{\bx}[\mesh]$ and all points~$\bx \in \mesh_0$.
Such areas can be exptected to be favored by the $\TGV$ regularizer we devise in \cref{subsection:discrete-tgv-of-the-normal-vector-field}.
We now relate this \enquote{linearity} of the normal vector field with the curvature of the surface~$\mesh$.
To this end, recall that the eigenvalues of $\bn_\bx' \in L(\tangentSpace{\bx}[\mesh],\tangentSpace{\bx}[\mesh])$ are known as principal curvatures while the eigenvectors are the principal directions of curvature.
Provided that $\covariantDerivative{\bchi}\bn'[\bxi] = \bnull$ for all $\bchi, \bxi \in \tangentSpace{\bx}[\mesh]$ holds in a region~$\mesh_0$ with non-empty interior, the numerator of \eqref{eq:second-derivative-of-the-normal} is zero along curves $\gamma$ in $\mesh_0$.
Choosing then $\bxi$ as an eigenvector of $\bn_\bx'$ with eigenvalue $\lambda$ yields
\begin{align*}
	\parallelTransportCurve{\gamma}{t}^{-1} \paren[big](){\bn_{\gamma(t)}'[\parallelTransportCurve{\gamma}{t}(\bxi)]}
	=
	\bn_\bx'[\bxi]
	=
	\lambda \, \bxi
	.
\end{align*}
Applying the parallel transport on both sides shows
\begin{equation*}
	\bn_{\gamma(t)}' \paren[big][]{\parallelTransportCurve{\gamma}{t}(\bxi)}
	=
	\lambda \, \parallelTransportCurve{\gamma}{t}(\bxi)
	,
\end{equation*}
\ie, $\lambda$ is an eigenvalue of $\bn_{\gamma(t)}'$ to the eigenvector $\parallelTransportCurve{\gamma}{t}(\bxi)$.
This means that \enquote{linearity} of the normal vector in the previously mentioned sense implies that the principal curvatures are constant in $\mesh_0$ and the principal directions of curvature are related by simple parallel transports.
This property is characteristic for planes, spheres and the lateral surface of cylinders.
Consequently, we expect that the regularizer based on the total generalized variation of the normal vector (see \eqref{eq:tgv:normal}) will favor shapes that are piecewise surfaces of these three types.

Notice that after the derivation and interpretation of \eqref{eq:second-derivative-of-the-normal}, the identification of $\tangentSpace{\bn(\bx)}[\sphere]$ with $\tangentSpace{\bx}[\mesh]$ is no longer required. Hence, $\covariantDerivative{\bchi}\bn_\bx'[\bxi]$ can be treated as an element of $\tangentSpace{\bn(\bx)}[\sphere]$ rather then $\tangentSpace{\bx}[\mesh]$.

\subsection{Tangential Raviart--Thomas Space}
\label{subsection:tangential-Raviart-Thomas-space}

We now return to the setting where $\mesh$ is a triangular and orientable surface mesh embedded in $\R^3$.
Using insights from the smooth setting in \cref{subsection:first-and-second-order-derivatives-of-normal-vector-fields-on-surfaces}, we construct a tailored Raviart--Thomas-like finite element space for the auxiliary variable that captures derivatives of the normal vector field in a discrete $\TGV$ formulation.
Notice that a similar idea has already been mentioned in \cite[Remark~2.4]{BrediesHollerStorathWeinmann:2018:1}, but only for Cartesian grids.
While it is natural to consider the derivative of the normal vector field intrinsically between two-dimensional tangent spaces, we find it convenient for implementation purposes to work with the ambient space $\R^3$.
To this end, we identify the space $L(\tangentSpace{\bx}[\cM],\tangentSpace{\bn(x)}[\sphere])$ with a subspace of $L(\R^3,\R^3)$, by mapping vectors orthogonal to $\tangentSpace{\bx}[\cM]$ to $\bnull$.
Using, for instance, the standard basis in~$\R^3$, an element of $L(\R^3,\R^3)$ can be represented by a $\R^{3 \times 3}$ matrix whose rows can be interpreted as elements of $\cotangentSpace{\bx}[\cM] \subseteq \R^{1 \times 3}$ and columns are interpreted as elements of $\tangentSpace{\bn(x)}[\sphere] \subseteq \R^3$.

For each point in the interior of a triangle $\triangle \in \triangles$, a function of the RT space under construction should take values in the tensor product space $\tangentSpace{\bn_\triangle}[\sphere] \otimes \tangentSpace{\triangle}[\mesh]$.
Therefore, we denote the space by $\tangentRT$.
In contrast to previous work on finite element spaces with tangent space-valued data such as \cite{Sander:2012:1}, we benefit from the fact that the tangent space is constant across each triangle, which simplifies the setting significantly.
Across an edge~$\edge$, however, the normal vector~$\bn$ may change.
Hence the tangent space spaces $\tangentplus$ and $\tangentminus$ may be different as well.

The desired mapping properties of $W$ into $\tangentSpace{\bn_\triangle}[\sphere] \otimes \tangentSpace{\triangle}[\mesh]$ entail that
\begin{equation*}
	\eplus{W} \eplus{\bmu}
	\in
	\tangentplus
	\quad
	\text{and}
	\quad
	- \eminus{W} \eminus{\bmu}
	\in
	\tangentminus
\end{equation*}
since the co-normal $\eplus{\bmu}$ belongs to $\tangentplus$ and $\eminus{\bmu}$ belongs to $\tangentminus$.

We impose the co-normal continuity of~$W$ in the sense of
\begin{equation}
	\label{eq:tangent-RT0:conormal-continuity}
	\eplus{W} \eplus{\bmu}
	=
	\parallelTransport[big]{\eminus{\bn}}{\eplus{\bn}}(- \eminus{W} \eminus{\bmu})
\end{equation}
on all edges.
This is a natural, intrinsic generalization of \eqref{eq:RT0:conormal-continuity} for the standard RT space.

To obtain an equivalent, more manageable formulation of \eqref{eq:tangent-RT0:conormal-continuity}, we use the identity $\id = \eplus{\bn}^{} \eplus{\bn}^\transp + \eplus{\bmu}^{} \eplus{\bmu}^\transp + \bt_\edge^{} \bt_\edge^\transp$ (and similarly with $\edge_-$) to rewrite \eqref{eq:tangent-RT0:conormal-continuity} as
\begin{equation*}
	(\eplus{\bmu} \eplus{\bmu}^\transp + \bt_\edge^{} \bt_\edge^\transp) \, \eplus{W} \eplus{\bmu}
	=
	\parallelTransport[big]{\eminus{\bn}}{\eplus{\bn}}(-(\eminus{\bmu} \eminus{\bmu}^\transp + \bt_\edge^{} \bt_\edge^\transp) \, \eminus{W} \eminus{\bmu})
	.
\end{equation*}
Using $\parallelTransport{\eminus{\bn}}{\eplus{\bn}}(- \eminus{\bmu}) = \eplus{\bmu}$ and $\parallelTransport{\eminus{\bn}}{\eplus{\bn}}(\bt_\edge) = \bt_\edge$, see \cref{lemma:sphere:parallel-transport:using-co-normals} and \cref{figure:co-normal-vectors}, we can further rewrite this as
\begin{equation*}
	\eplus{\bmu} \paren[big](){\eplus{\bmu}^\transp \eplus{W}^{} \eplus{\bmu}^{}}
	+
	\bt_\edge^{} \paren[big](){\bt_\edge^\transp \eplus{W}^{} \eplus{\bmu}^{}}
	=
	\eplus{\bmu} \paren[big](){\eminus{\bmu}^\transp \eminus{W}^{} \eminus{\bmu}^{}}
	+
	\bt_\edge \paren[big](){- \bt_\edge^\transp \eminus{W}^{} \eminus{\bmu}^{}}
	.
\end{equation*}
Since $\eplus{\bmu}$ and $\bt_\edge$ are orthogonal, this is equivalent to
\begin{equation}
	\label{eq:tangent-RT0:conormal-continuity:alternative}
	\begin{aligned}
		\eplus{\bmu}^\transp \eplus{W}^{} \eplus{\bmu}
		&
		=
		\eminus{\bmu}^\transp \eminus{W}^{} \eminus{\bmu}
		,
		\\
		\bt_\edge^\transp \eplus{W}^{} \eplus{\bmu}
		&
		=
		- \bt_\edge^\transp \eminus{W}^{} \eminus{\bmu}
		.
	\end{aligned}
\end{equation}
This motivates the following basis functions indexed by the edges~$\edge$ and supported on the adjacent triangles $T_{E_\pm}$:
\begin{subequations}
	\label{eq:tangent-RT0:basis}
	\begin{align}
		\Phi_{\edge,1}(\bx)
		&
		=
		\begin{cases}
			\tfrac{1}{2 \, \abs{\eplus{\triangle}}} \eplus{\bmu}(\bx - \eplus{\bp})^\transp
			&
			\text{ if }
			\bx \in \eplus{\triangle}
			,
			\\
			\tfrac{1}{2 \, \abs{\eminus{\triangle}}} \eminus{\bmu}(\bx - \eminus{\bp})^\transp
			&
			\text{ if }
			\bx \in \eminus{\triangle}
			,
			\\
			0
			&
			\text{ else}
			,
		\end{cases}
		\label{eq:tangent-RT0:basis:1}
		\\
		\Phi_{\edge,2}(\bx)
		&
		=
		\begin{cases}
			\mrep{\tfrac{1}{2 \, \abs{\eplus{\triangle}}} \bt_\edge \, (\bx - \eplus{\bp})^\transp}{\tfrac{1}{2 \, \abs{\eplus{\triangle}}} \eplus{\bmu}(\bx - \eplus{\bp})^\transp}
			&
			\text{ if }
			\bx \in \eplus{\triangle}
			,
			\\
			\tfrac{-1}{2 \, \abs{\eminus{\triangle}}} \bt_\edge \, (\bx - \eminus{\bp})^\transp
			&
			\text{ if }
			\bx \in \eminus{\triangle}
			,
			\\
			0
			&
			\text{ else}
			.
		\end{cases}
		\label{eq:tangent-RT0:basis:2}
	\end{align}
\end{subequations}
where $\bx$ is the spatial coordinate and $\bp_{E_\pm}$ is the coordinate of the vertex of~$T_{E_\pm}$ opposite to~$\edge$.
It is easy to see that, as desired, $\Phi_{\edge,1}$ and $\Phi_{\edge,2}$ are linear, matrix-valued functions with rows from $\cotangentSpace{\triangle}[\, \mesh]$ and columns from $\tangentSpace{\bn_\triangle}[\sphere]$ on each triangle $\triangle \in \triangles$.
To verify the intrinsic co-normal continuity \eqref{eq:tangent-RT0:conormal-continuity:alternative} of these basis functions, a simple geometric consideration shows $(\bx - \bp_{E_\pm})^\transp \bmu_{E_\pm} = \tfrac{2 \, \abs{T_{E_\pm}}}{\abs{\edge}}$ for any $\bx \in \edge$.
Plugging this in yields
\begin{align*}
	\eplus{\bmu}^\transp \eplus{(\Phi_{\edge,1})} \eplus{\bmu}
	&
	\equiv
	\frac{1}{\abs{\edge}}
	=
	\eminus{\bmu}^\transp \eminus{(\Phi_{\edge,1})} \eminus{\bmu}
	,
	\\
	\bt_\edge^\transp \eplus{(\Phi_{\edge,2})} \eplus{\bmu}
	&
	\equiv
	\frac{1}{\abs{\edge}}
	=
	- \bt_\edge^\transp \eminus{(\Phi_{\edge,2})} \eminus{\bmu}
	.
\end{align*}
Consequently, one obtains
\begin{alignat*}{2}
	\delta_{i1}
	&
	=
	\int_\edge \eplus{\bmu}^\transp \eplus{(\Phi_{\edge,i})} \eplus{\bmu} \dS
	&
	&
	=
	\int_\edge \eminus{\bmu}^\transp \eminus{(\Phi_{\edge,i})} \eminus{\bmu} \dS
	,
	\\
	\delta_{i2}
	&
	=
	\int_\edge \bt_\edge^\transp \eplus{(\Phi_{\edge,i})} \eplus{\bmu} \dS
	&
	&
	=
	- \int_\edge \bt_\edge^\transp \eminus{(\Phi_{\edge,i})} \eminus{\bmu} \dS.
\end{alignat*}
for $i = 1, 2$, where $\delta_{ij}$ is the Kronecker delta.
Furthermore, for an edge $\widetilde{\edge} \neq \edge$, we have
\begin{align*}
	0
	&
	=
	\int_\edge \eplus{\bmu}^\transp \eplus{(\Phi_{\widetilde{\edge},i})} \eplus{\bmu} \dS
	,
	\\
	0
	&
	=
	\int_\edge \bt_\edge^\transp \eplus{(\Phi_{\widetilde{\edge},i})} \eplus{\bmu} \dS
	,
\end{align*}
for $i = 1,2$.
This is because we have $(\bx - \widetilde{\bp})^\transp \eplus{\bmu} = 0$ for $\widetilde{\bp}$ for $\bx \in \edge$, where $\widetilde{\bp}$ is the coordinate of the vertex opposite to $\widetilde{\edge}$.
The situation on $\eminus{\triangle}$ is similar.
Therefore,
\begin{subequations}
	\label{eq:tangent-RT0:degrees-of-freedom}
	\begin{alignat}{2}
		\int_\edge \eplus{\bmu}^\transp \eplus{W} \eplus{\bmu} \dS
		&
		=
		\int_\edge \eminus{\bmu}^\transp \eminus{W} \eminus{\bmu} \dS
		\\
		\int_\edge \bt_\edge^\transp \eplus{W} \eplus{\bmu} \dS
		&
		=
		- \int_\edge \bt_\edge^\transp \eminus{W} \eminus{\bmu} \dS
	\end{alignat}
\end{subequations}
are the degrees of freedom corresponding to the basis functions $\phi_{\edge,1}$, $\phi_{\edge,2}$,
which is closely related to the degrees of freedom of the standard Raviart--Thomas space from the literature~\eqref{eq:RT0:degrees-of-freedom}.
Therefore, we define
\begin{equation}
	\tangentRT
	\coloneqq
	\Span \bigcup_{\edge \in \edges} \set{\phi_{\edge,1}, \phi_{\edge,2}}
	.
\end{equation}

In the discrete $\TGV$ \eqref{eq:tgv:DG0} from \cite{BaumgaertnerBergmannHerzogSchmidtVidalNunez:2023:1}, \ie, the scalar-valued setting, the $\alpha_0$-terms evaluate the (discrete) total variation of the auxiliary variable $\bw \in \RT{0}$.
When adapting the formulation to the normal vector of a mesh, the same has to be done for $W \in \tangentRT$.
It consists of the variation of~$W$ on each triangle as well as contributions from the edges.

Since the triangles are planar, their tangent spaces are constant and the parallel transports in \eqref{eq:second-derivative-of-the-normal} can be omitted.
Hence, the derivative of the matrix-valued $W$ in a triangle $T$ is a constant tensor of order three which is computed by standard techniques.
This defines a piecewise Jacobian of $W$, which we denote by $D_\mesh W$.
Then, $\restr{\D_\mesh W}{\triangle} \in P_0(\triangle, \tangentSpace{\bn_\triangle}[\sphere] \otimes \tangentSpace{\triangle}[\mesh] \otimes \tangentSpace{\triangle}[\mesh])$ and
\begin{align*}
	\D_\mesh \Phi_{\edge,1}(\bx)
	&
	=
	\begin{cases}
		\tfrac{1}{2 \, \abs{\eplus{\triangle}}} \eplus{\bmu} \otimes (\id - \eplus{\bn}^{} \eplus{\bn}^\transp)
		&
		\text{ if }
		\bx \in \eplus{\triangle}
		,
		\\
		\tfrac{1}{2 \, \abs{\eminus{\triangle}}} \eminus{\bmu} \otimes (\id - \eminus{\bn}^{} \eminus{\bn}^\transp)
		&
		\text{ if }
		\bx \in \eminus{\triangle}
		,
		\\
		0
		&
		\text{ else }
		,
	\end{cases}
	\\
	\D_\mesh \Phi_{\edge,2}(\bx)
	&
	=
	\begin{cases}
		\mrep{\tfrac{1}{2 \, \abs{\eplus{\triangle}}} \bt_\edge \otimes (\id - \eplus{\bn}^{} \eplus{\bn}^\transp)}{\tfrac{1}{2 \, \abs{\eplus{\triangle}}} \eplus{\bmu} \otimes (\id - \eplus{\bn}^{} \eplus{\bn}^\transp)}
		&
		\text{ if }
		\bx \in \eplus{\triangle}
		,
		\\
		\tfrac{-1}{2 \, \abs{\eminus{\triangle}}} \bt_\edge \otimes (\id - \eminus{\bn}^{} \eminus{\bn}^\transp)
		&
		\text{ if }
		\bx \in \eminus{\triangle}
		,
		\\
		0
		&
		\text{ else}
		.
	\end{cases}
\end{align*}

For the edge contribution, things will be more involved.
If two adjacent triangles $\eplus{\triangle}$ and $\eminus{\triangle}$ at an edge $\edge$ are not co-planar, the tangent spaces will differ.
To measure the jump between $\eplus{W}$ and $\eminus{W}$ intrinsically, we proceed similarly as for the numerator in \eqref{eq:second-derivative-of-the-normal}.
Hence, we require a parallel transport from $\tangentSpace{\eplus{\triangle}}[\! \! \mesh]$ to $\tangentSpace{\eminus{\triangle}}[\!\!\mesh]$.
For this purpose, we again use the identification of $\tangentSpace{\triangle}[\mesh]$ and $\tangentSpace{\bn_\triangle}[\sphere]$ and use the parallel transport along shortest geodesics on the sphere \eqref{eq:sphere:parallel-transport:1}.
This means that the jump of~$W$, applied to a tangent vector $\bxi \in \tangentplus$, across an edge is computed via
\begin{equation}
	\label{eq:tangent-RT0:jump-of-W:general-direction}
	\parallelTransport[big]{\eminus{\bn}}{\eplus{\bn}}(\eminus{W} \parallelTransport{\eplus{\bn}}{\eminus{\bn}}{\bxi})
	-
	\eplus{W} \bxi
	.
\end{equation}
We examine the above term by inserting the orthonormal basis vectors $\eplus{\bmu},\bt_\edge$ for~$\bxi$.
When $\bxi = \eplus{\bmu}$, the difference in \eqref{eq:tangent-RT0:jump-of-W:general-direction} is zero due to the co-normal continuity \eqref{eq:tangent-RT0:conormal-continuity} and using $\parallelTransport{\eplus{\bn}}{\eminus{\bn}}(\eplus{\bmu}) = - \eminus{\bmu}$.
When $\bxi = \bt_\edge$ we obtain $\parallelTransport{\eplus{\bn}}{\eminus{\bn}}(\bt_\edge) = \bt_\edge$ and thus the inner parallel transport in \eqref{eq:tangent-RT0:jump-of-W:general-direction} can be omitted.
We thus define the intrinsic jump as
\begin{equation}
	\label{eq:tangent-RT0:jump-of-W}
	\jump{W}_\edge
	\coloneqq
	\parallelTransport[big]{\eminus{\bn}}{\eplus{\bn}}(\eminus{W} \bt_\edge)
	-
	\eplus{W} \bt_\edge
	\in
	\tangentplus
	.
\end{equation}

\subsection{Discrete Total Generalized Variation of the Normal Vector Field}
\label{subsection:discrete-tgv-of-the-normal-vector-field}

In the following, we adapt \eqref{eq:tgv:DG0} from piecewise constant, scalar-valued functions $u \colon \mesh \to \R$ to the piecewise constant normal vector field $\bn \colon \mesh \to \sphere$.
For the $\alpha_1$-term, we couple the logarithmic map from the total variation of the normal vector formula \eqref{eq:tv:normal:alternative} to the co-normal component of a function $W \in \tangentRT$.
Replacing the jump in \eqref{eq:tgv:DG0} by its analogue for the normal vector, the $\alpha_1$-term on an edge $\edge \in \edges$ becomes
\begin{equation}
	\label{eq:tangent-RT0:alpha1-term}
	\abs[big]{\logarithm{\eplus{\bn}}{\eminus{\bn}} + h_\edge \, \eplus{W} \eplus{\bmu}}_2
	.
\end{equation}
Using the fact that $\logarithm{\eplus{\bn}}{\eminus{\bn}}$ is a multiple of $\eplus{\bmu}$ by \cref{lemma:sphere:logarithmic-map:using-co-normals} and that $\set{\eplus{\bmu}, \bt_\edge}$ form an orthonormal basis of $\tangentplus$, this can be rewritten to
\begin{multline}
	\label{eq:tangent-RT0:alpha1-term:alternative}
	\abs[big]{\logarithm{\eplus{\bn}}{\eminus{\bn}} + h_\edge \, \eplus{W} \eplus{\bmu}}_2^2
	\\
	=
	\abs[big]{\eplus{\bmu}^\transp \paren(){\logarithm{\eplus{\bn}}{\eminus{\bn}} + h_\edge \, \eplus{W} \eplus{\bmu}}}^2
	+
	\abs[big]{h_\edge \, \bt_\edge^\transp \eplus{W} \eplus{\bmu}}^2
	.
\end{multline}
This reveals that \eqref{eq:tangent-RT0:alpha1-term} in fact couples the logarithmic map only to one of the degrees of freedom of $W$ on the edge, namely $\eplus{\bmu}^\transp \eplus{W}^{} \eplus{\bmu}^{}$, while the other one, $\bt_\edge^\transp \, \eplus{W}^{} \eplus{\bmu}^{}$, is penalized.
The term $\bt_\edge^\transp \, \eplus{W}^{} \eplus{\bmu}^{}$ being zero would mean that $\eplus{\bmu}$ and $\bt_\edge$ are principal directions of curvature, because in this basis, the matrix $\eplus{W}$ is diagonal.

Since we aim to avoid that the local quantities
$\eplus{\bmu}$ and $\bt_\edge$
become principal directions of curvature, we omit the last term in \eqref{eq:tangent-RT0:alpha1-term:alternative} and only use
\begin{equation}
	\abs[big]{\eplus{\bmu}^\transp \paren(){\logarithm{\eplus{\bn}}{\eminus{\bn}} + h_\edge \, \eplus{W} \eplus{\bmu}}}
\end{equation}
for the coupling between the normal vector~$\bn$ and the auxiliary variable~$W$.

For the $\alpha_0$-term we utilize the tangential Jacobian in the triangles as well as the tangential jump \eqref{eq:tangent-RT0:jump-of-W}.
Overall, we propose the following formulation as the discrete total generalized variation of the normal vector field $\bn \colon \mesh \to \sphere$:
\begin{align}
	\label{eq:tgv:normal}
	\fetgv_{(\alpha_0, \alpha_1)}^2(\bn)
	&
	\coloneqq
	\\
	\min_{W \in \tangentRT}
	&
	\alpha_1 \sum_{\edge \in \edges} \int_\edge \abs[big]{\eplus{\bmu}^\transp \paren(){\logarithm{\eplus{\bn}}{\eminus{\bn}} + h_\edge \, \eplus{W} \eplus{\bmu}}} \dS
	\notag
	\\
	{}
	+
	{}
	&
	\alpha_0 \sum_{\triangle \in \triangles} \int_\triangle \abs{\D_\mesh W_\triangle}_F \dx
	+
	\alpha_0 \sum_{\edge \in \edges} \int_\edge \interpolate[big]{1}{\abs[big]{\jump{W}_\edge}_2} \dS
	.
	\notag
\end{align}
The reader is invited to compare this with the discrete total generalized variation \eqref{eq:tgv:DG0} for a scalar quantity.
Here, the distance between the circumcenters, $h_\edge$, is measured intrinsically, as described in \cite[Section~5.2.3]{BaumgaertnerBergmannHerzogSchmidtVidalNunez:2023:1}, \ie,
\begin{equation}
	h_\edge
	\coloneqq
	\dotproduct{\eplus{\bmu}}{(\bm_\edge - \eplus{\bm})}
	+
	\dotproduct{\eminus{\bmu}}{(\bm_\edge - \eminus{\bm})}
	,
\end{equation}
where $\eplus{\bm}$ and $\eminus{\bm}$ denote the circumcenters of the triangles $\eplus{\triangle}$ and $\eminus{\triangle}$, and $\bm_\edge$ denotes the midpoint of the edge.

\begin{lemma}
	Formulation \eqref{eq:tgv:normal} is independent of the orientation of the edges.
\end{lemma}
\begin{proof}
	For the $\alpha_1$-term first notice that by \cref{lemma:sphere:logarithmic-map:using-co-normals}, we have
	\begin{equation*}
		\dotproduct{\eplus{\bmu}}{\logarithm{\eplus{\bn}}{\eminus{\bn}}}
		=
		\dotproduct{\eminus{\bmu}}{\logarithm{\eminus{\bn}}{\eplus{\bn}}}
		.
	\end{equation*}
	Thereby, using also the tangential co-normal continuity \eqref{eq:tangent-RT0:conormal-continuity:alternative}, it holds
	\begin{equation*}
		\abs[big]{\dotproduct{\eplus{\bmu}}{\paren(){\logarithm{\eplus{\bn}}{\eminus{\bn}} + h_\edge \, \eplus{W} \eplus{\bmu}}}}
		=
		\abs[big]{\dotproduct{\eminus{\bmu}}{\paren(){\logarithm{\eminus{\bn}}{\eplus{\bn}} + h_\edge \, \eminus{W} \eminus{\bmu}}}}
		.
	\end{equation*}
	Clearly, the triangle contribution will be independent of the orientation if the sign of the degrees of freedom \eqref{eq:tangent-RT0:degrees-of-freedom} are flipped accordingly.
	For the last term, we can use that the parallel transport \eqref{eq:sphere:parallel-transport} is norm preserving and linear and thus
	\begin{align*}
		\abs[big]{\jump{W}_\edge}_2
		&
		=
		\abs[big]{\parallelTransport{\eplus{\bn}}{\eminus{\bn}}(\jump{W}_\edge)}_2
		\\
		&
		=
		\abs[big]{ \parallelTransport[big]{\eplus{\bn}}{\eminus{\bn}}(\parallelTransport{\eminus{\bn}}{\eplus{\bn}}(\eminus{W} \bt_\edge)) - \parallelTransport{\eplus{\bn}}{\eminus{\bn}}(\eplus{W} \bt_\edge)}_2
		\\
		&
		=
		\abs[big]{\eminus{W} \bt_\edge - \parallelTransport{\eplus{\bn}}{\eminus{\bn}}(\eplus{W} \bt_\edge)}_2
		\\
		&
		=
		\abs[big]{ \parallelTransport{\eplus{\bn}}{\eminus{\bn}}(\eplus{W} \bt_\edge) - \eminus{W} \bt_\edge}_2
		.
	\end{align*}
	The last quantity is precisely \eqref{eq:tangent-RT0:jump-of-W} with the roles of $\eplus{\triangle}$ and $\eminus{\triangle}$ swapped.
\end{proof}

\section{Numerical Realization}
\label{section:numerical-realization}

In this section we present a realization of the discrete total generalized variation of the normal \eqref{eq:tgv:normal} as a regularizer utilizing the alternating direction method of multipliers (ADMM).
This method is utilized to deal with the non-smoothness present in all problems involving total variation terms.
To formulate the ADMM, finite element spaces with values in tangent spaces are required, which will be introduced first.

\subsection{Tangential Finite Element Spaces}
\label{subsection:notation-and-finite-element-spaces:tangential}

In addition to the standard spaces defined in \cref{subsection:notation-and-finite-element-spaces}, we define $\DG{0}(\mesh, \sphere)$ as the finite element space with constant, $\sphere$-valued data on each triangle. In particular, the normal vector~$\bn$ to~$\mesh$ belongs to $\DG{0}(\mesh, \sphere)$. Furthermore, we define $\DG{0}(\edges, \sphere)$ as the space with piecewise constant, $\sphere$-valued data on edges.

Since $\sphere$ is embedded in $\R^3$, the tangent space $\tangentSpace{\bm}[\sphere]$ at a point $\bm \in \sphere$ is a subspace of $\R^3$.
Given the normal vector field $\bn \in \DG{0}(\mesh, \sphere)$, we define $\DG{r}(\mesh, \tangentBundle[\sphere], \bn)$ as
\begin{equation}
	\DG{r}(\mesh, \tangentBundle[\sphere], \bn)
	\coloneqq
	\setDef[big]{\bu \in \DG{r}(\mesh, \R^3)}{\bu_\triangle \in P_r(\triangle, \tangentSpace{\bn_\triangle}[\sphere]) \text{ for all } \triangle \in \triangles}
	,
\end{equation}
a subspace of $\DG{r}(\triangle, \R^3)$.
In other words, $\DG{r}(\mesh, \tangentBundle[\sphere], \bn)$ consists of piecewise polynomials with values in the tangent space to the sphere~$\sphere$ at the point specified by the normal vector in the respective triangle.
Similarly, we define $\DG{r}(\edges, \tangentBundle[\sphere], \eplus{\bn})$ as the subspace of $\DG{r}(\edges, \R^3)$ with values in $\tangentplus$ on each $\edge \in \edges$.

Recall that we identify the tangent space $\tangentSpace{\triangle}[\mesh]$ of a triangle~$\triangle$ is denoted with $\tangentSpace{\bn_\triangle}[\sphere]$.
Therefore, we define $\DG{r}(\mesh, \tangentBundle[\mesh], \bn)\coloneqq \DG{r}(\mesh, \tangentBundle[\sphere], \bn)$.
However, we continue to use both notations.

\subsection{Derivation of the ADMM}

We consider the problem
\begin{align}
	\text{Minimize}
	\quad
	&
	\cF(\mesh)
	+
	\alpha_1 \sum_{\edge \in \edges} \int_\edge \abs[big]{\dotproduct{\eplus{\bmu}}{\paren(){\logarithm{\eplus{\bn}}{\eminus{\bn}} + h_\edge \, \eplus{W} \eplus{\bmu}}}} \dS
	\notag
	\\
	&
	+
	\alpha_0 \sum_{\triangle \in \triangles} \int_\triangle \abs{\D_\mesh W_\triangle}_F \dx
	+
	\alpha_0 \sum_{\edge \in \edges} \int_\edge \interpolate[big]{1}{\abs[big]{\jump{W}_\edge}_2} \dS
	\notag
	\\
	\text{\wrt\ }
	\quad
	&
	\text{the vertex positions in }
	\mesh
	\text{ and }
	W \in \tangentRT
	.
	\label{eq:generic-problem-with-tgv-of-the-normal-regularization}
\end{align}
where $\cF$ is some smooth function depending on the mesh~$\mesh$ and the remaining terms in the objective represent the discrete total generalized variation of the normal vector field $\bn$, see \eqref{eq:tgv:normal}.
A concrete example for $\cF$ for the purpose of mesh denoising will be specified in \eqref{eq:mesh-denoising-with-tgv-of-the-normal-regularization}.

The first optimization variable in \eqref{eq:generic-problem-with-tgv-of-the-normal-regularization} is the collection of vertex positions.
The connectivity of the mesh~$\mesh$ is preserved throughout the optimization.
The second optimization variable $W \in \tangentRT$ depends on the current vertex positions, and thus technically we are facing an optimization problem over a fiber bundle.
Notice that also the quantities $\eplus{\bmu}$, $\eplus{\bn}$ and $\eminus{\bn}$ depend on the vertex positions.

Following the ADMM paradigm of adding variables to achieve simpler subproblems, we introduce additional variables $d_0 \in \DG{0}(\skeleton, \R)$, $\bd_2 \in \DG{1}(\skeleton, \tangentBundle[\sphere], \eplus{\bn})$ defined on the skeleton~$\skeleton$ of the mesh, as well as $\bD_1 \in \DG{0}(\mesh, \tangentBundle[\sphere] \otimes \tangentBundle[\mesh] \otimes \tangentBundle[\mesh], \bn)$ defined on the triangles.
These variables are coupled to the original quantities of the problem via the following constraints:
\begin{subequations}
	\label{eq:ADMM:constraints}
	\begin{alignat}{2}
		d_{0,\edge}
		&
		=
		\dotproduct{\eplus{\bmu}}{\paren(){\logarithm{\eplus{\bn}}{\eminus{\bn}} + h_\edge \, \eplus{W} \eplus{\bmu}}}
		&
		&
		\in
		P_0(\edge,\R)
		\label{eq:ADMM:constraints:1}
		\\
		\bD_{1, \triangle}
		&
		=
		\D_\mesh W_\triangle
		&
		&
		\in
		P_0(\triangle,\tangentSpace{\bn_\triangle}[\sphere] \otimes \tangentSpace{\triangle}[\mesh] \otimes \tangentSpace{\triangle}[\mesh])
		\label{eq:ADMM:constraints:2}
		\\
		\bd_{2,\edge}
		&
		=
		\jump{W}_\edge
		&
		&
		\in
		P_1(\edge,\tangentplus)
		\label{eq:ADMM:constraints:3}
	\end{alignat}
\end{subequations}
on all edges $\edge \in \edges$ and triangles $\triangle \in \triangles$ respectively.
For these constraints, Lagrange multipliers $\lambda_0 \in \DG{0} \paren[auto](){\skeleton, \R}$, $\bLambda_1 \in \DG{0} \paren[auto](){\mesh, \tangentBundle[\sphere] \otimes \tangentBundle[\mesh] \otimes \tangentBundle[\mesh], \bn}$ and $\blambda_2 \in \DG{1} \paren[auto](){\skeleton, \tangentBundle[\sphere], \eplus{\bn}}$ are introduced.
That is, we represent Lagrange multipliers as primal objects rather than elements from the dual space of the constraint co-domain.
To adjoin the constraints and specify the norms for the penalty terms, we use the following inner products:
\begin{alignat*}{2}
	\inner{\lambda_0}{d_0}_{\DG{0}(\skeleton, \R)}
	&
	\coloneqq
	\sum_{\edge \in \edges}
	\abs{\edge} \, \lambda_{0,\edge} \, d_{0,\edge}
	,
	\\
	\inner{\bLambda_1}{\bD_1}_{\DG{0}(\mesh, \tangentBundle[\sphere] \otimes \tangentBundle[\mesh] \otimes \tangentBundle[\mesh], \bn)}
	&
	\coloneqq
	\sum_{\triangle \in \triangles}
	\abs{\triangle} \, \bLambda_{1, \triangle} \dprod \bD_{1, \triangle}
	=
	\sum_{\triangle \in \triangles}
	\abs{\triangle} \, \trace \paren[big](){\bLambda_{1, \triangle}^\transp \bD_{1, \triangle}}
	,
	\\
	\inner{\blambda_2}{\bd_2}_{\DG{1}(\skeleton, \tangentBundle[\sphere], \eplus{\bn})}
	&
	\coloneqq
	\sum_{\edge \in \edges}
	\sum_{i=1}^2 \frac{\abs{\edge}}{2} \, \dotproduct{\blambda_{2,\edge}(X_{\edge, i})}{\bd_{2,\edge}(X_{\edge, i})}
	.
\end{alignat*}
Recall that $X_{\edge,1}, X_{\edge,2}$ are the two endpoints of the edge~$\edge$.
Notice that the interpolation operator has the same support points $X_{\edge,1}, X_{\edge,2}$ as the inner product, which will be exploited later.

Using these definitions as well as penalty parameters $\rho_0, \rho_1, \rho_2 > 0$, we obtain the following augmented Lagrangian function for the problem at hand:
\makeatletter
\begin{align}
	\MoveEqLeft
	\cL_\rho(\mesh, W, \bd_0,\bD_1, \bd_2, \blambda_0, \bLambda_1, \blambda_2)
	\notag
	\\
	&
	=
	\cF(\mesh)
	+
	\alpha_1 \sum_{\edge \in \edges} \abs{\edge} \, \abs{d_{0,\edge}}
	\ltx@ifclassloaded{siamonline250106}{}{%
		\notag
		\\
		&
		\quad
	}
	+
	\alpha_0 \sum_{\triangle \in \triangles} \abs{\triangle} \, \abs{\bD_{1,\triangle}}_F
	+
	\alpha_0 \sum_{\edge \in \edges} \smash{\sum_{i=1}^2} \frac{\abs{\edge}}{2} \, \abs[big]{\bd_2(X_{\edge, i})}_2
	\notag
	\\
	&
	\quad
	\mrep[r]{{}+{}}{{}+ \frac{\rho_0}{2}} \sum_{\edge \in \edges} \abs{\edge} \, \lambda_{0,\edge} \, \paren[big][]{\dotproduct{\eplus{\bmu}}{\paren(){\logarithm{\eplus{\bn}}{\eminus{\bn}} + h_\edge \, \eplus{W} \eplus{\bmu}}} - d_{0,\edge}}
	\notag
	\\
	&
	\quad
	+
	\frac{\rho_0}{2} \sum_{\edge \in \edges} \abs{\edge} \, \abs[big]{d_{0,\edge} - \dotproduct{\eplus{\bmu}}{(\logarithm{\eplus{\bn}}{\eminus{\bn}} + h_\edge \, \eplus{W} \eplus{\bmu})}}^2
	\notag
	\\
	&
	\quad
	\mrep[r]{{}+{}}{{}+ \frac{\rho_1}{2}} \sum_{\triangle \in \triangles} \abs{\triangle} \, \bLambda_{1, \triangle} \dprod \paren[big][]{\D_\mesh W_\triangle - \bD_{1,\triangle}}
	+
	\frac{\rho_1}{2} \sum_{\triangle \in \triangles} \abs{\triangle} \, \abs[big]{\bD_{1, \triangle} - \D_\mesh W_\triangle}_F^2
	\notag
	\\
	&
	\quad
	\mrep[r]{{}+{}}{{}+ \frac{\rho_2}{2}} \sum_{\edge \in \edges} \smash{\sum_{i=1}^2} \frac{\abs{\edge}}{2} \, \dotproduct{\blambda_2(X_{\edge, i})}{\paren[big][]{\jump{W}_\edge(X_{\edge, i}) - \bd_2(X_{\edge, i})}}
	\notag
	\\
	&
	\quad
	+
	\frac{\rho_2}{2} \sum_{\edge \in \edges} \smash{\sum_{i=1}^2} \frac{\abs{\edge}}{2} \, \abs[big]{\bd_2(X_{\edge,i}) - \jump{W}_\edge(X_{\edge, i})}_2^2
	.
	\label{eq:ADMM:augmented-Lagrangian}
\end{align}
\makeatother

The ADMM for problem \eqref{eq:generic-problem-with-tgv-of-the-normal-regularization} requires the independent minimization of the augmented Lagrangian \eqref{eq:ADMM:augmented-Lagrangian} with respect to the vertex positions of $\mesh$ as well as $W$ and the auxiliary variables $(d_0, \bD_1, \bd_2)$, and the update of multipliers.
Notice that the problems involving $d_0, \bD_1, \bd_2$ are independent of each other and thus these problems can be solved simultaneously in each iteration.
The overall algorithm is described in \cref{algorithm:ADMM}.
We now proceed to discussing the subproblems in more detail, in their order of appearance in \cref{algorithm:ADMM}.

\begin{algorithm}[htb]
	\caption{ADMM for problem \eqref{eq:generic-problem-with-tgv-of-the-normal-regularization}.}
	\label{algorithm:ADMM}
	\begin{algorithmic}[1]
		\Require initial mesh $\sequence{\mesh}{0}$
		\Require regularization parameters $\alpha_0, \alpha_1 > 0$
		\Require penalty parameters $\rho_0, \rho_1, \rho_2 > 0$
		\Ensure approximate solution of \eqref{eq:generic-problem-with-tgv-of-the-normal-regularization}
		\State Set $k \coloneqq 0$
		\State Set $\sequence{W}{0} \coloneqq 0$
		\State Set $(\sequence{d_0}{0},\sequence{\bD_1}{0},\sequence{\bd_2}{0}) \coloneqq (0,\bnull,\bnull)$
		\State Set $(\sequence{\lambda_0}{0},\sequence{\bLambda_1}{0},\sequence{\blambda_2}{0}) \coloneqq (0,\bnull,\bnull)$
		\While{not converged}
		\State Set $\sequence{d_0}{k+1} \coloneqq \argmin\limits_{\mrep{d_0 \in \DG{0}(\sequence{\skeleton}{k}, \R)}{}} \cL_\rho(\sequence{\mesh}{k},\sequence{W}{k}, d_0, \sequence{\bD_1}{k}, \sequence{\bd_2}{k}, \sequence{\lambda_0}{k}, \sequence{\bLambda_1}{k}, \sequence{\blambda_2}{k})$
		\label{line:ADMM:d0-problem}
		\State Set $\sequence{\bD_1}{k+1} \coloneqq \argmin\limits_{\mrep{\bD_1 \in \DG{0}(\sequence{\mesh}{k}, \tangentBundle[\sphere] \otimes \tangentBundle[\mesh] \otimes \tangentBundle[\mesh], \sequence{\bn}{k})}{}} \cL_\rho(\sequence{\mesh}{k},\sequence{W}{k}, \sequence{d_0}{k}, \bD_1, \sequence{\bd_2}{k}, \sequence{\lambda_0}{k}, \sequence{\bLambda_1}{k}, \sequence{\blambda_2}{k})$
		\label{line:ADMM:D1-problem}
		\State Set $\sequence{\bd_2}{k+1} \coloneqq \argmin\limits_{\mrep{\bd_2 \in \DG{1}( \sequence{\skeleton}{k}, \tangentBundle[\sphere], \sequence{ \eplus{\bn} }{k})}{}} \cL_\rho(\sequence{\mesh}{k},\sequence{W}{k}, \sequence{d_0}{k}, \sequence{\bD_1}{k}, \bd_2, \sequence{\lambda_0}{k}, \sequence{\bLambda_1}{k}, \sequence{\blambda_2}{k})$
		\label{line:ADMM:d2-problem}
		\State Set $\sequence{W}{k+1} \coloneqq \argmin\limits_{\mrep{W \in \tangentRT[\sequence{\mesh}{k}]}{}} \cL_\rho(\sequence{\mesh}{k},W, \sequence{d_0}{k+1}, \sequence{\bD_1}{k+1}, \sequence{\bd_2}{k+1}, \sequence{\lambda_0}{k}, \sequence{\bLambda_1}{k}, \sequence{\blambda_2}{k})$
		\label{line:ADMM:W-problem}
		\State Perform a number of globalized Newton steps for the approximate solution of
		\begin{equation*}
			\sequence{\mesh}{k+1}
			\approx
			\argmin_{\mesh} \cL_\rho(\mesh,\sequence{W}{k+1}, \sequence{d_0}{k+1}, \sequence{\bD_1}{k+1}, \sequence{\bd_2}{k+1}, \sequence{\lambda_0}{k}, \sequence{\bLambda_1}{k}, \sequence{\blambda_2}{k})
			,
		\end{equation*}
		using parallel transports to the correct tangent spaces for $\sequence{\bD_1}{k+1},\sequence{\bd_2}{k+1}$ and $\sequence{\bLambda_1}{k+1},\sequence{\blambda_2}{k+1}$
		\label{line:ADMM:shape-problem}
		\State Parallely transport $\sequence{\bD_1}{k+1},\sequence{\bd_2}{k+1}$ and $\sequence{\bLambda_1}{k+1},\sequence{\blambda_2}{k+1}$ from the tangent spaces corresponding to $\sequence{\mesh}{k}$ to the tangent spaces corresponding to $\sequence{\mesh}{k+1}$ using \eqref{eq:sphere:parallel-transport} and \eqref{eq:parallel-transport-of-a-tensor}
		\label{line:ADMM:parallel-transport}
		\State Set $\sequence{\lambda_{0, \edge}}{k+1} \coloneqq \sequence{\lambda_{0, \edge}}{k} + \rho_0 \paren[big][]{\dotproduct{\sequence{\eplus{\bmu}}{k+1}}{(\logarithm{\sequence{\eplus{\bn}}{k+1}}{\sequence{\eminus{\bn}}{k+1}} + \sequence{h_\edge}{k+1} \, \sequence{\eplus{W}}{k+1} \sequence{\eplus{\bmu}}{k+1})} - \sequence{d_{0, \edge}}{k+1}}$
		\label{line:ADMM:lambda0-update}
		\State Set $\sequence{\bLambda_{1, \triangle}}{k+1} \coloneqq \sequence{\bLambda_{1,\triangle}}{k} + \rho_1 \paren[big][]{\D_\mesh \sequence{W_\triangle}{k+1} - \sequence{\bD_{1, \triangle}}{k+1}}$
		\label{line:ADMM:Lambda1-update}
		\State Set $ \sequence{\blambda_{2, \edge}}{k+1} \coloneqq \sequence{\blambda_{2,\edge}}{k} + \rho_2 \paren[big][]{\sequence{\jump{W}_\edge}{k+1} - \sequence{\bd_{2, \edge}}{k+1}}$
		\label{line:ADMM:lambda2-update}
		\State Set $k \coloneqq k+1$
		\EndWhile
	\end{algorithmic}
\end{algorithm}

\subsubsection*{\texorpdfstring{Minimization \wrt $(d_0,\bD_1,\bd_2)$}{Minimization with Respect to d0, D1, d2}}

The minimization \wrt $(d_0,\bD_1,\bd_2)$ of the augmented Lagrangian \eqref{eq:ADMM:augmented-Lagrangian} decouples into three independent problems, addressed in \crefrange{line:ADMM:d0-problem}{line:ADMM:d2-problem} of \cref{algorithm:ADMM}.
Moreover, each problem further decouples into very small subproblems, one for each edge or triangle.
By completing the squares in~\eqref{eq:ADMM:augmented-Lagrangian}, one is left with simple (non-smooth) minimization problems to compute the coefficients, denoted here by $d_{0,\edge}$, $\bD_{1,\triangle}$, $\bd_2(X_{\edge, 1})$ and $\bd_2(X_{\edge, 2})$, on all edges or triangles, respectively:
\begin{subequations}
	\label{eq:d-problems}
	\begin{align}
		\text{Minimize}
		\quad
		&
		\alpha_1 \, \abs{d_{0,\edge}}
		+
		\frac{\rho_0}{2} \abs[Big]{d_{0,\edge} - \dotproduct{\eplus{\bmu}}{\paren(){\logarithm{\eplus{\bn}}{\eminus{\bn}} + h_\edge \, \eplus{W} \eplus{\bmu}}} - \frac{\lambda_{0,\edge}}{\rho_0}}^2
		\notag
		\\
		\text{\wrt\ }
		\quad
		&
		d_{0,\edge} \in \R
		,
		\\[0.5\baselineskip]
		\text{Minimize}
		\quad
		&
		\alpha_0 \, \abs{\bD_{1, \triangle}}_F
		+
		\frac{\rho_1}{2} \, \abs[Big]{\bD_{1, \triangle} - \D_\mesh W_\triangle - \frac{\bLambda_{1, \triangle}}{\rho_1}}_F^2
		\notag
		\\
		\text{\wrt\ }
		\quad
		&
		\bD_{1, \triangle} \in \tangentSpace{\bn_\triangle}[\mesh] \otimes \tangentSpace{\triangle}[\mesh] \otimes \tangentSpace{\triangle}[\mesh]
		,
		\\[0.5\baselineskip]
		\text{Minimize}
		\quad
		&
		\alpha_0 \, \abs[big]{\bd_2(X_{\edge, i})}_2
		+
		\frac{\rho_2}{2} \, \abs[Big]{\bd_2(X_{\edge, i}) - \jump{W}_\edge(X_{\edge, i}) - \frac{\blambda_2(X_{\edge, i})}{\rho_2}}_2^2
		\notag
		\\
		\text{\wrt\ }
		\quad
		&
		\bd_2(X_{\edge, i}) \in \tangentplus
	\end{align}
\end{subequations}
for $i = 1,2$.
These problems belong to the following class of convex, piecewise quadratic problems
\begin{equation}
	\text{Minimize}
	\quad
	\alpha \, \abs{d}_*
	+
	\frac{\rho}{2} \abs{d-x}_*^2
	\quad
	\text{\wrt\ }
	d
\end{equation}
on a vector space with some norm $\abs{\,\cdot\,}_*$ induced by an inner product.
The solution is given by the soft-thresholding (shrinkage) operator
\begin{equation}
	\label{eq:soft-thresholding-function}
	\shrink[Big]{x}{\frac{\alpha}{\rho}}
	\coloneqq
	\begin{cases}
		\frac{x}{\abs{x}_*} \max \set{\abs{x}_* - \frac{\alpha}{\rho}, \; 0}
		&
		\text{if }
		x \neq 0
		,
		\\
		0
		&
		\text{if }
		x = 0
		.
	\end{cases}
\end{equation}

\subsubsection*{Minimization with Respect to the Auxiliary Variable $W$}

The minimization \wrt~$W$ in~\cref{line:ADMM:W-problem} is an unconstrained, positive definite quadratic problem and thus requires the solution of a linear system of equations. 
Since $\tangentRT$ has two scalar degrees of freedom on each edge, see \eqref{eq:tangent-RT0:degrees-of-freedom}, the size of the system is twice the number of edges. 
We use a conjugate gradient method to solve this problem.
We found that preconditioning via symmetric successive over relaxation (SSOR) was sufficiently efficient.
We are using the \petsc conjugate gradient (CG) implementation  \cite{BalayAbhyankarAdamsBrownBruneBuschelmanConstantinescuDenerFaibussowitschGroppIsaacKaushikKnepleyKongMcInnesMunsonRuppSananSarichSmithZhangZhangBensonSuhDalcinEijkhoutHaplaJolivetKarpeevKrugerMayMitchellRomanZampiniMillsZhang:2024:1} with a relative tolerance of $10^{-3}$ with respect to the Euclidean norm of the residual, which typically requires about $70$~iterations.

\subsubsection*{Approximate Solution with Respect to the Vertex Coordinates}

In the augmented Lagrangian \eqref{eq:ADMM:augmented-Lagrangian}, the auxiliary variable $W$, additional variables $\bD_1,\bd_2$ and multipliers $\bLambda_1,\blambda_2$ are required to be elements of the respective tangent spaces corresponding to the piecewise constant normal vector~$\bn$, which in turn depends on the optimization variable~$\mesh$.
For the minimization \wrt $\mesh$ in \cref{line:ADMM:shape-problem} of \cref{algorithm:ADMM} with \enquote{fixed} $\sequence{W}{k+1}$, $\sequence{d_0}{k+1}$, $\sequence{\bD_1}{k+1}$, $\sequence{\bd_2}{k+1}$, $\sequence{\lambda_0}{k+1}$, $\sequence{\bLambda_1}{k+1}$, $\sequence{\blambda_2}{k+1}$, one therefore has to define how these variables behave when the mesh is updated.

In case of $\sequence{W}{k+1} \in \tangentRT$, the basis functions \eqref{eq:tangent-RT0:basis} depend directly on the mesh, which means that they automatically adapt when $\mesh$ is deformed.
We simply leave the coefficients unchanged.
For $\sequence{\bD_1}{k+1},\sequence{\bLambda_1}{k+1} \in \DG{0}(\mesh, \tangentBundle[\sphere] \otimes \tangentBundle[\mesh] \otimes \tangentBundle[\mesh], \sequence{\bn}{k}) $ and $\sequence{\bd_2}{k+1}, \sequence{\blambda_2}{k+1} \in \DG{1}( \skeleton, \tangentBundle[\sphere], \sequence{ \eplus{\bn} }{k} )$ parallel transports to the tangent spaces after the update will be used.
This is discussed further in the respective paragraph for the parallel transports.

Even with these dependencies, the minimization of the augmented Lagrangian $\cL_\rho$ with respect to $\mesh$ is still smooth and can be carried out by standard techniques of unconstrained optimization.
In particular we use a globalized, inexact, truncated Newton-CG scheme similar to what is described, \eg, in \cite[p.49]{UlbrichUlbrich:2012:1}.
First- and second-order derivatives of the augmented Lagrangian with respect to the mesh coordinates are evaluated using a combination of algorithmic differentiation (AD) and hand-coded derivatives.
The implementation details match those in our recent publication \cite[Section~4]{BaumgaertnerBergmannHerzogSchmidtVidalNunezWeiss:2025:1}, which only differs in the objective function.

\subsubsection*{Parallel Transport}

As described in the previous paragraph, the minimization of the augmented Lagrangian with respect to $\mesh$ in \cref{line:ADMM:shape-problem} requires parallel transports of the \enquote{fixed} variables $\sequence{\bD_1}{k+1},\sequence{\bLambda_1}{k+1} \in \DG{0}(\mesh, \tangentBundle[\sphere] \otimes \tangentBundle[\mesh] \otimes \tangentBundle[\mesh], \sequence{\bn}{k}) $ and $\sequence{\bd_2}{k+1}, \sequence{\blambda_2}{k+1} \in \DG{1}( \skeleton, \tangentBundle[\sphere], \sequence{ \eplus{\bn} }{k} )$ from tangent spaces corresponding to the previous iterate $\sequence{\mesh}{k}$ to tangent spaces corresponding to the current~$\mesh$.
For $\sequence{\bd_2}{k+1}$ this is achieved by replacing every occurrence of $\sequence{\bd_2}{k+1}(X_{E,i})$ in the augmented Lagrangian \eqref{eq:ADMM:augmented-Lagrangian} by $\parallelTransport[big]{\sequence{\eplus{\bn}}{k}}{\eplus{\bn}}(\sequence{\bd_2}{k+1}(X_{E,i}))$.
Here, $\sequence{\bn}{k} \in \DG{0}(\mesh, \sphere)$ is the normal vector field of $\sequence{\mesh}{k}$ and $\bn$ is the normal vector field of the optimization variable $\mesh$.
We replace $\sequence{\blambda_2}{k+1}(X_{E,i})$ analogously.

For the tensor-valued quantity $\sequence{\bD_1}{k+1}$, the parallel transport is slightly more complicated.
As seen in \eqref{eq:sphere:parallel-transport:2}, the action of the parallel transport from~$\sequence{\bn}{k}$ to~$\bn$ can be represented elementwise by the matrix
\begin{equation}
	M
	\coloneqq
	\paren[Big](){\id - \frac{\bn_\triangle + \sequence{\bn_\triangle}{k}}{1 + \dotproduct{\bn_\triangle}{\sequence{\bn_\triangle}{k}}} \bn_\triangle^\transp}
	.
\end{equation}
This matrix is simply applied to each axis of the tensor, which amounts to
\begin{equation}
	\label{eq:parallel-transport-of-a-tensor}
	\paren[auto](){\parallelTransport{\sequence{\bn}{k}}{\bn}(\bD_1)}_{ijk}
	\coloneqq
	\sum_{a,b,c=1}^3 (\bD_1)_{abc} \, M_{ia} \, M_{jb} \, M_{kc}
	.
\end{equation}
Then, as before, every occurrence of $\sequence{\bD_{1,T}}{k+1}$ in the augmented Lagrangian \eqref{eq:ADMM:augmented-Lagrangian} is replaced by $\parallelTransport[big]{\sequence{\bn_\triangle}{k}}{\bn_\triangle}(\sequence{\bD_{1,T}}{k+1})$ (analogously $\bLambda_{1,T}$).

After the update of $\mesh$ by \cref{line:ADMM:shape-problem}, the variables $\sequence{\bD_1}{k+1}$, $\sequence{\bLambda_1}{k+1}$, $\sequence{\bd_2}{k+1}$, $\sequence{\blambda_2}{k+1}$ are parallely transported to tangent spaces corresponding to the new iterate $\sequence{\mesh}{k+1}$ in \cref{line:ADMM:parallel-transport}.
Thereby, the coefficients of the variables are changed in order to correspond to the respective mesh iterate $\sequence{\mesh}{k+1}$.

\subsubsection*{Multiplier Update}

Finally, \crefrange{line:ADMM:lambda0-update}{line:ADMM:lambda2-update} of \cref{algorithm:ADMM} are standard multiplier updates of ADMM.

\section{Numerical Results for Mesh Denoising}
\label{section:numerical-results-for-mesh-denoising}

In this section, we present numerical experiments for mesh denoising problems.
We compare the proposed total generalized variation of the normal regularizer \eqref{eq:tgv:normal} to first-order total variation regularization \eqref{eq:tv:normal} as described in \cite{BaumgaertnerBergmannHerzogSchmidtVidalNunezWeiss:2025:1} as well as to alternate formulations from \cite{LiuLiWangLiuChen:2022:1,ZhangHeWang:2022:1}.
In the $\alpha_1$-term of their respective formulation, both \cite{LiuLiWangLiuChen:2022:1} and \cite{ZhangHeWang:2022:1} couple the Euclidean difference of the normal vectors of two adjacent triangles to an auxiliary variable in $\DG{0}(\edges, \R^{3})$.
The authors of \cite{ZhangHeWang:2022:1} then proceed to use the approach similar as in \cite{GongSchullckeKruegerZiolekZhangMuellerLisseMoeller:2018:1} for images and utilize a divergence like operator for the $\alpha_0$-term by adding up the values of the auxiliary variable on the three edges of a triangle.
Unlike the approach in \cite{GongSchullckeKruegerZiolekZhangMuellerLisseMoeller:2018:1}, however, they add an additional weight to each term in order to be closer to a full derivative and avoid spurious oscillations.
The authors of \cite{LiuLiWangLiuChen:2022:1} use different combinations of edge values over larger patches of adjacent triangles to obtain more accurate derivative information from the auxiliary variable.

Both methods use a normal filtering approach to realize their formulation for mesh denoising. This means that first, the problem
\begin{equation}
	\mathop{\operatorname{Minimize}}\limits_{\bm \in \DG{0}(\mesh, \R^3)} \frac{1}{2} \sum_{\triangle \in \triangles} \abs{\bm_\triangle - \bn_\triangle}_2^2 + \cR(\bm)
\end{equation}
is solved on the noisy mesh, where $\cR$ is the respective variant of total generalized variation for piecewise constant data from \cite{LiuLiWangLiuChen:2022:1} or \cite{ZhangHeWang:2022:1}.
Then, the vertex positions are adapted such that the normal vector $\bn$ of $\mesh$ is similar to the optimized variable $\bm$; see also \cite{ZhangDengZhangBouazizLiu:2015:1, SunRosinMartinLangbein:2007:1}.

We on the other hand do not use normal filtering and instead optimize the vertex positions of the mesh directly.
The objective function we use for the purpose of mesh denoising is
\begin{equation}
	\label{eq:mesh-denoising-with-tgv-of-the-normal-regularization}
	\cF(\mesh)
	\coloneqq
	\frac{1}{2} \sum_{\vertex \in \vertices} \abs{\bx_\vertex - \bx^\text{data}_\vertex}_2^2
	+
	\tau \sum_{\triangle \in \cT} \frac{1}{\abs{\triangle}}
	.
\end{equation}
The first term is a fidelity term in the squared $\ell_2$-norm, and the second term is a barrier term that avoids degenerately small triangles, as used in \cite{BaumgaertnerBergmannHerzogSchmidtVidalNunezWeiss:2025:1}.
The proposed total generalized variation term \eqref{eq:tgv:normal} with parameters $\alpha_0, \alpha_1$ is added to the objective.
For comparison, we also consider denoising using the first-order total variation of the normal with parameter $\beta$ using the method from \cite{BaumgaertnerBergmannHerzogSchmidtVidalNunezWeiss:2025:1}.
The values of $\alpha_0, \alpha_1$ or $\beta$ and $\tau$ are to be balanced so that the regularizer term dominates unless triangles become extremely small.
Here, we always use $\tau = 10^{-12}$.
The implementation was achieved in the finite element framework \fenics, version~2019.2.0.\footnote{The code is publicly available at \url{https://github.com/LukasBaumgaertner/tgv-of-normal-code}.}

An optimized implementation of the tangential Raviart--Thomas finite element described in \cref{subsection:tangential-Raviart-Thomas-space} is outside the scope of this work.
Nevertheless, an implementation in \fenics is possible, which has the advantage of providing automatic derivatives with respect to geometric changes.
Our experiments were carried out on a desktop computer with an AMD Ryzen~5~3600 CPU.
To give a rough idea, the computation times are given in \cref{table:runtime}.
\begin{table}[h!]
	\centering
	\begin{tabular}{lrrl}
		\toprule
		test case & number of vertices & runtime in minutes & results
		\\
		\midrule
		spheres    & \num{12231} & 156 & \cref{figure:spheres}
		\\
		cylinders  & \num{13848} & 163 & \cref{figure:cylinders}
		\\
		fandisk    & \num{6475}  & 91  & \cref{figure:fandisk:denoising}
		\\
		joint      & \num{20902} & 434 & \cref{figure:joint:denoising}
		\\
		\bottomrule
	\end{tabular}
	\caption{Runtime of $300$ iterations of \cref{algorithm:ADMM} for the different test cases.}
	\label{table:runtime}
\end{table}

Our numerical experiments are organized as follows.
In \cref{subsection:spheres-and-cylinders-meshes} we present results for simple geometries consisting of hemispheres and half-cylinders, respectively.
Featuring piecewise constant principal curvatures, these geometries are idealized showcases for using the total generalized variation of the normal vector field as regularizer.
In particular, they serve to distinguish the discrete formulation $\fetgv^2$ proposed in \eqref{eq:tgv:normal} from that in \cite{LiuLiWangLiuChen:2022:1}, referred to as meshTGV.
(We cannot compare with rTGV~\cite{ZhangHeWang:2022:1} since no implementation is available to us.)

Subsequently, we consider two real-world geometries from the literature in \cref{subsection:fandisk-mesh,subsection:joint-mesh}.
In each test, Gaussian noise with standard deviation based on the average edge length is added to the vertex positions.
In fact, we use the noisy input data from \cite{ZhangHeWang:2022:1} in order to include their results in our comparison.
We thank the authors of \cite{LiuLiWangLiuChen:2022:1} for making their method publicly available, and the authors of \cite{ZhangHeWang:2022:1} for providing us with access to their numerical results.

\subsection{Spheres and Cylinders}
\label{subsection:spheres-and-cylinders-meshes}

As mentioned above, we expect denoising problems using the TGV of the normal vector field as regularizer to perform well in denoising problems with geometries featuring piecewise constant principal curvatures, such as planes, spheres, and cylinders.
To verify this numerically, we generate a $3 \times 3$ grid of hemispheres on a flat base, using \gmsh \cite{GeuzaineRemacle:2009:1}.
Each row of spheres uses a different radius, and each column uses a different mesh resolution.
Gaussian noise with standard deviation of $0.2$ times the average edge length is then added to the mesh.
We compare our proposed model \eqref{eq:tgv:normal} with the results using the method from \cite{LiuLiWangLiuChen:2022:1}, whose implementation is publicly available.

The results are shown in \cref{figure:spheres}.

\begin{figure}[ht]
	\includegraphics[width = 0.49\linewidth]{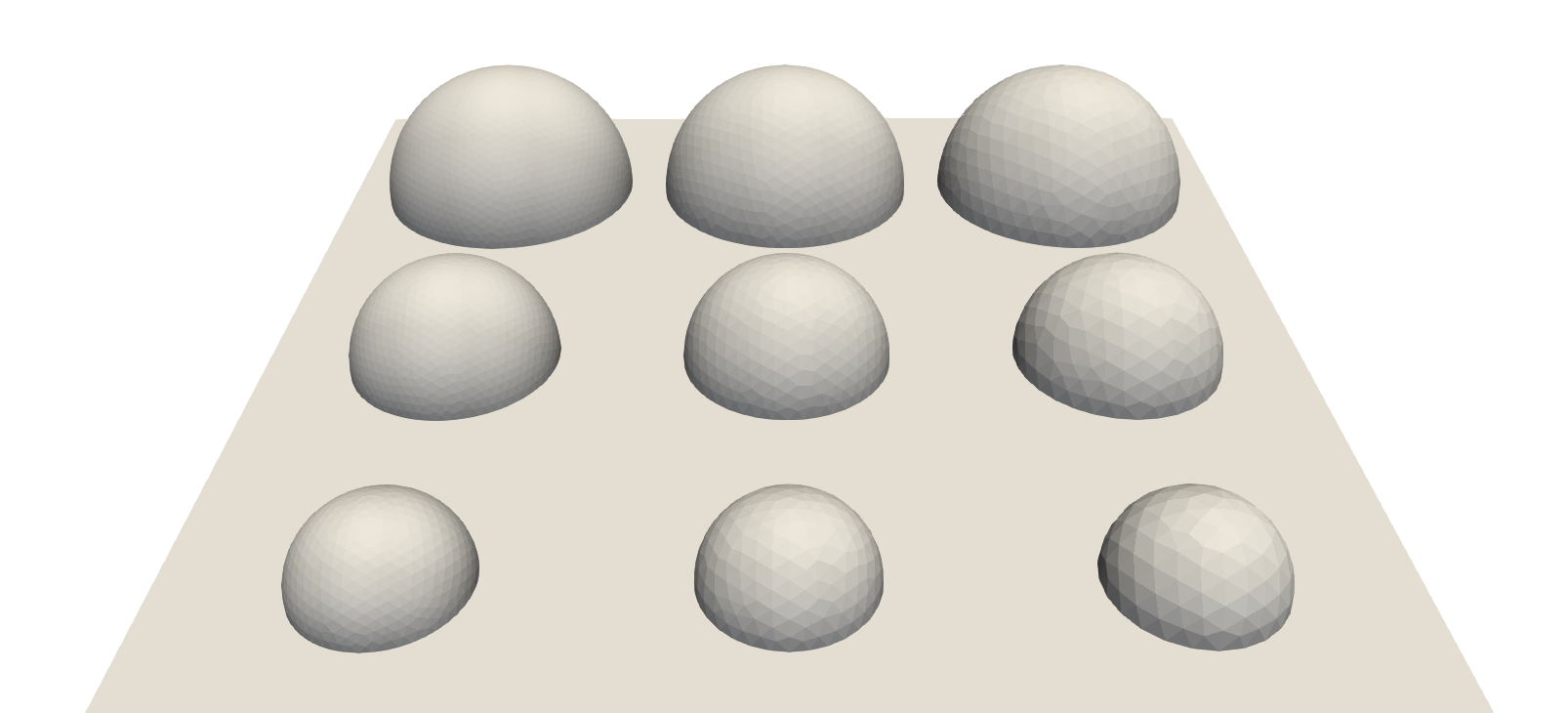}
	\hfill
	\includegraphics[width = 0.49\linewidth]{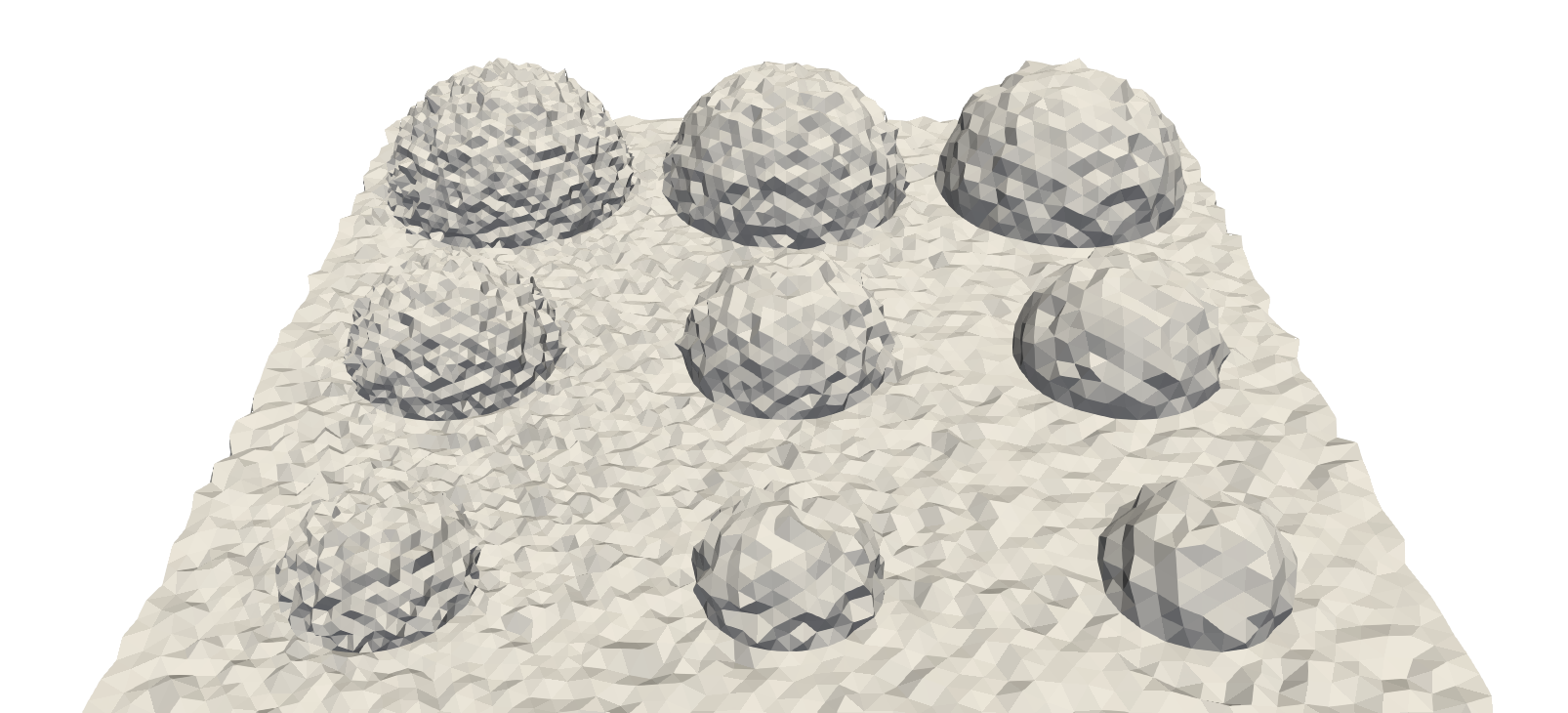}
	\\
	\includegraphics[width = 0.49\linewidth]{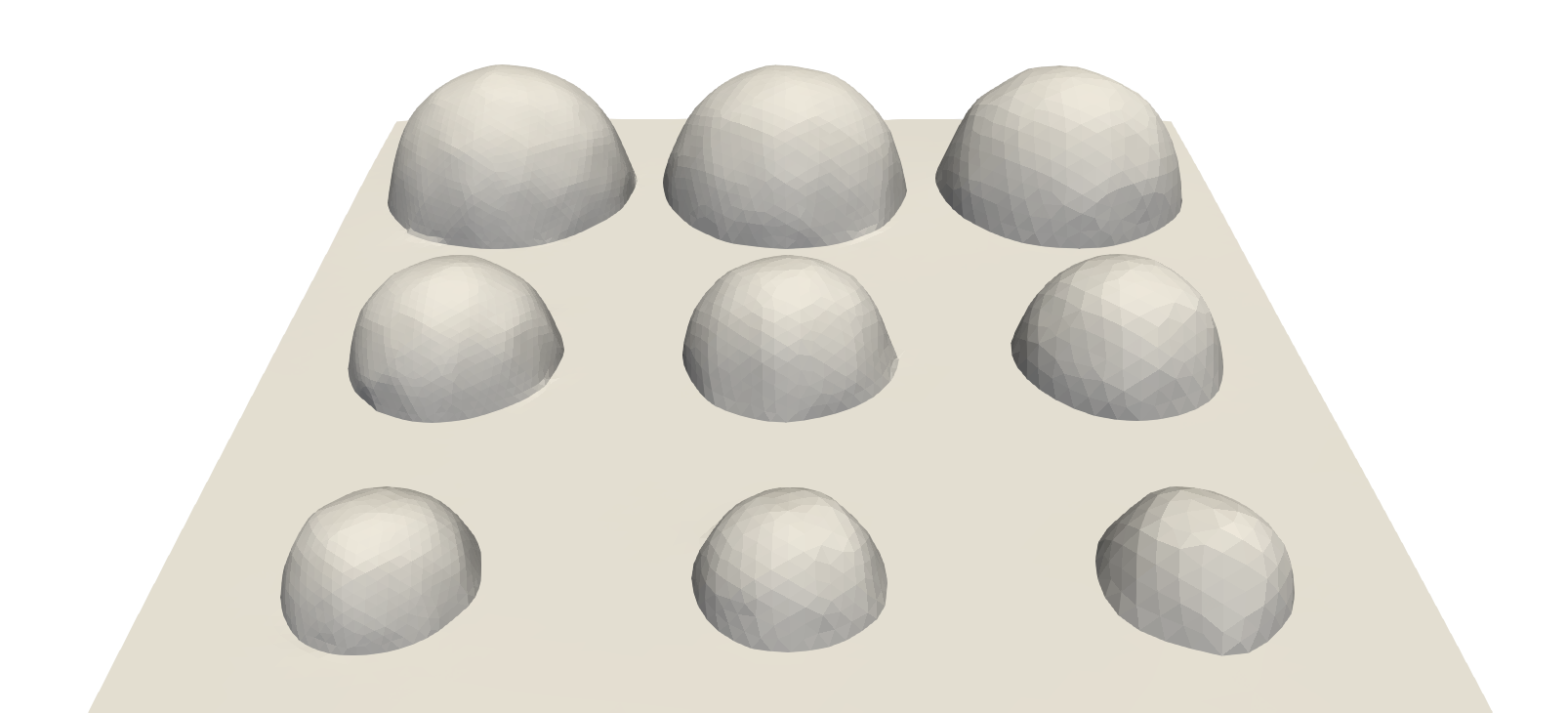}
	\hfill
	\includegraphics[width = 0.49\linewidth]{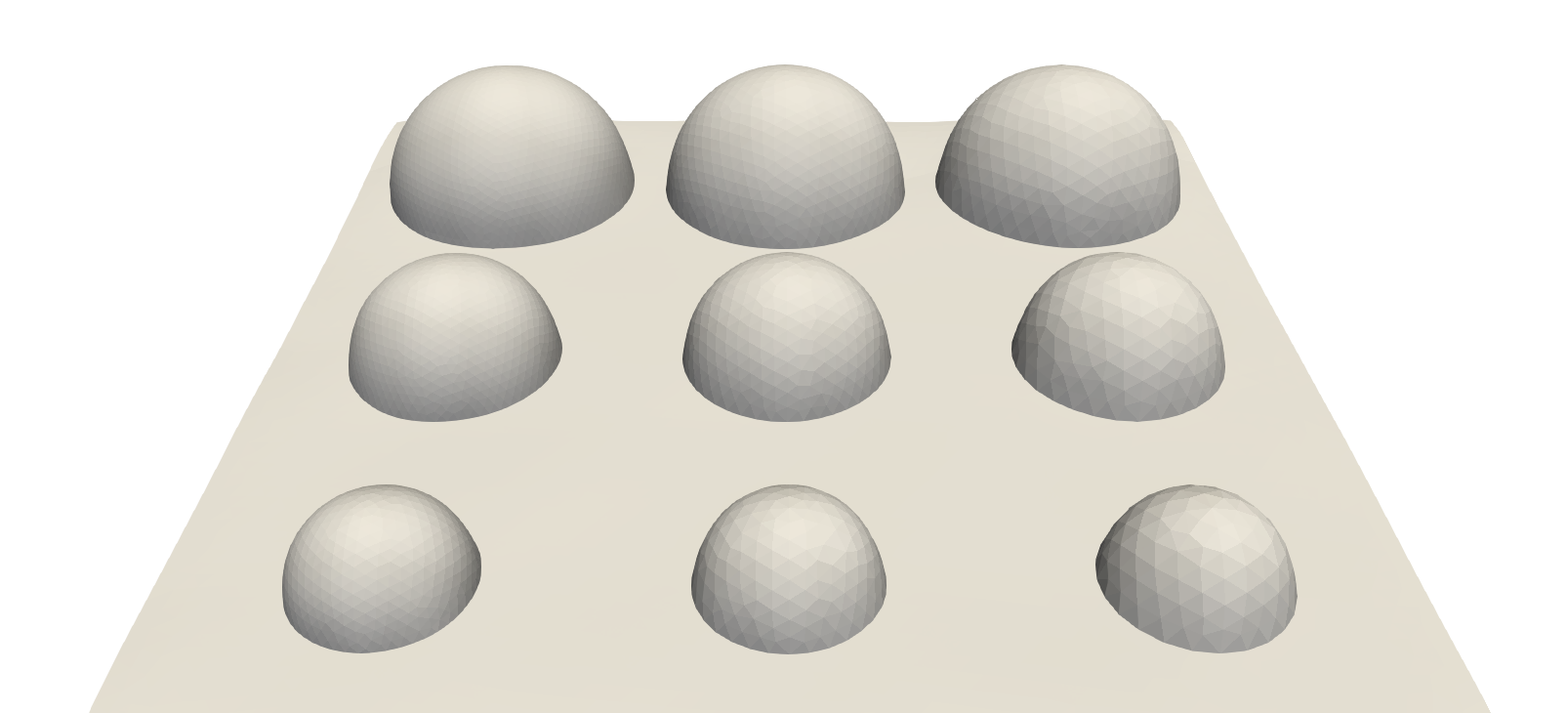}
	\caption{%
	Top row: original geometry (left), noisy geometry (right).
	Bottom row: reconstructions using meshTGV~\cite{LiuLiWangLiuChen:2022:1} (left) with $\alpha_0 = 0.2$ and $\alpha_1 = 1.1$, and the proposed $\fetgv^2$ (right) using \eqref{eq:tgv:normal} with $\alpha_0 = 3 \cdot 10^{-5}$, $\alpha_1 = 3.5 \cdot 10^{-3}$.
	}
	\label{figure:spheres}
\end{figure}
We observe that our method manages to reconstruct the spheres almost perfectly for all radii and mesh resolutions.
For the method from \cite{LiuLiWangLiuChen:2022:1} we were unable find suitable parameters to produce results of similar quality.
In particular, a slight staircasing effect always seems to remain, resulting in a less even reconstruction.

In order to assess the quality of each reconstruction quantitatively, we evaluate two distance metrics between the reconstructed and the original geometries.
As the first metric, we use the \texttt{Hausdorff Distance} function \cite{CignoniRocchiniScopigno:1998:1} in \meshlab \cite{CignoniCallieriCorsiniDellepianeGanovelliRanzuglia:2008:1}.
To compare two meshes $\Gamma_A, \Gamma_B$ and the geometries they represent, the \texttt{Hausdorff Distance function samples} each triangle of both meshes using several sample points, yielding two point clouds $P_A, P_B$.
The average Hausdorff distance between the meshes is then approximated through
\begin{equation}
	\label{eq:mesh_distance}
	\disttext{vertices}(\Gamma_A,\Gamma_B)
	\coloneqq
	\frac{1}{\abs{P_A}} \sum_{a \in P_A} \min_{b \in \Gamma_B} \abs{a-b}_2
	+
	\frac{1}{\abs{P_B}} \sum_{b \in P_B} \min_{a \in \Gamma_A} \abs{a-b}_2
	.
\end{equation}
As the second metric, we utilize the mean distance of the normal vectors, inspired by \cite{LiZhuFuHeng:2018:1}.
Since both meshes have the same connectivity, we can evaluate the distance of the normal vectors $\bn_{T_A}, \bn_{T_B}$ on a per-triangle basis, \ie,
\begin{equation}
	\label{eq:normal_distance}
	\disttext{normals}(\Gamma_A, \Gamma_B)
	\coloneqq
	\frac{1}{\abs{\cT}} \sum_{\triangle\in \triangles} \dist{\bn_{T_A}}{\bn_{T_B}}
	.
\end{equation}
For the example shown in \cref{figure:spheres}, the meshTGV method from \cite{LiuLiWangLiuChen:2022:1} achieved $\disttext{vertices} = \num{0.00160}$ and $\disttext{normals} = \num{0.0392}$, while our method reached $\disttext{vertices} = \num{0.00126}$ and $\disttext{normals} = \num{0.0247}$.
Hence, our approach performs measureably better \wrt both metrics.

We repeat the experiment using a $3 \times 3$ grid of half-cylinders, with the same changes of radii and mesh resolutions as for the hemispheres.
The results are shown in \cref{figure:cylinders}.

\begin{figure}[ht]
	\includegraphics[width = 0.49\linewidth]{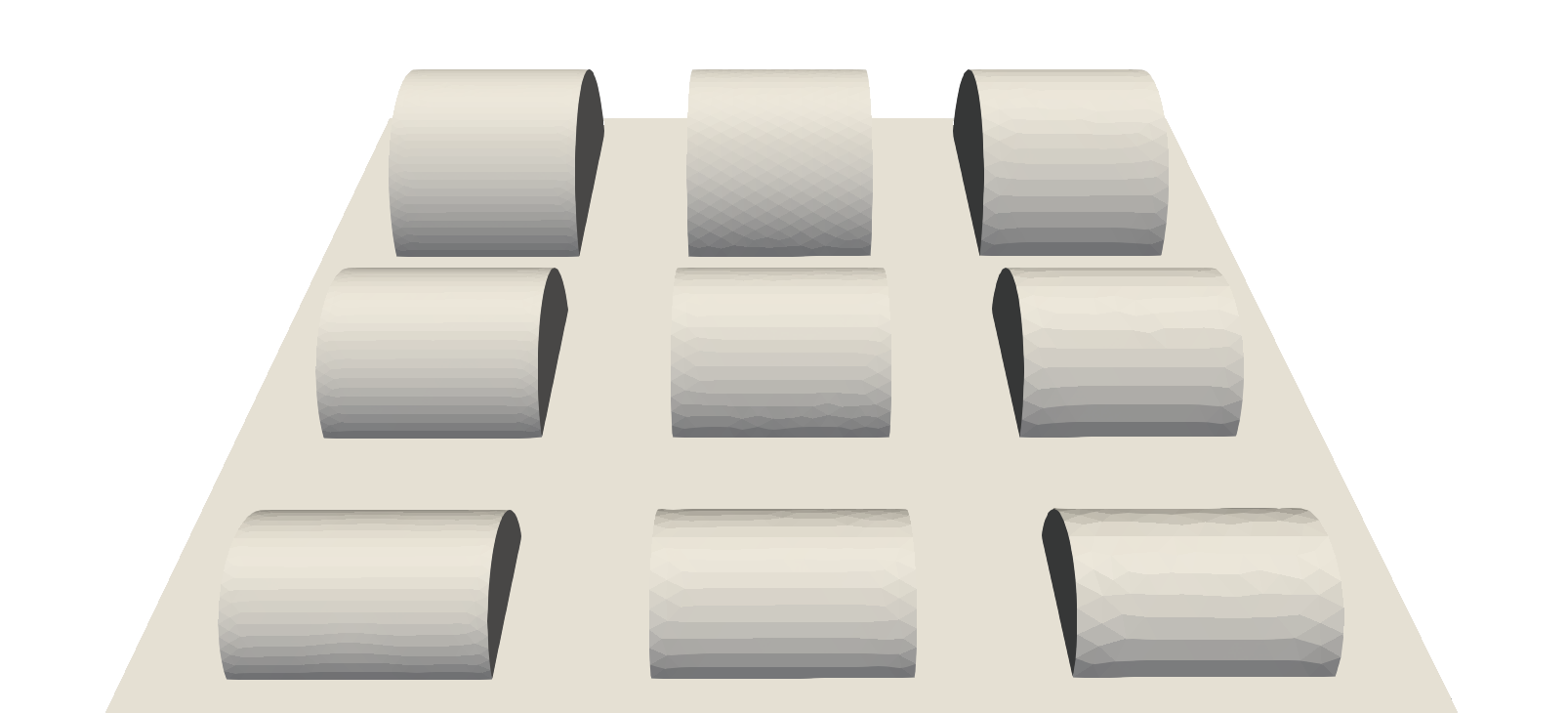}
	\hfill
	\includegraphics[width = 0.49\linewidth]{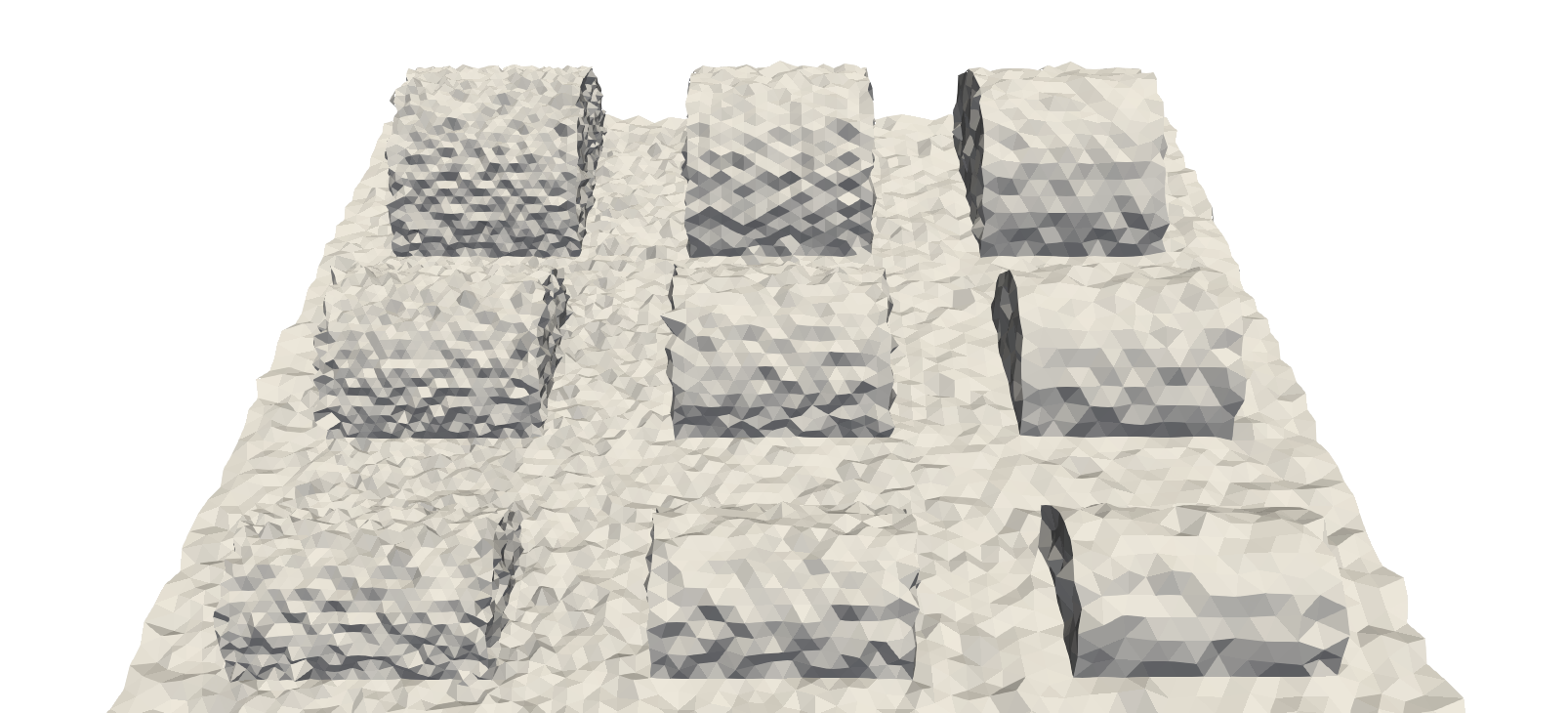}
	\\
	\includegraphics[width = 0.49\linewidth]{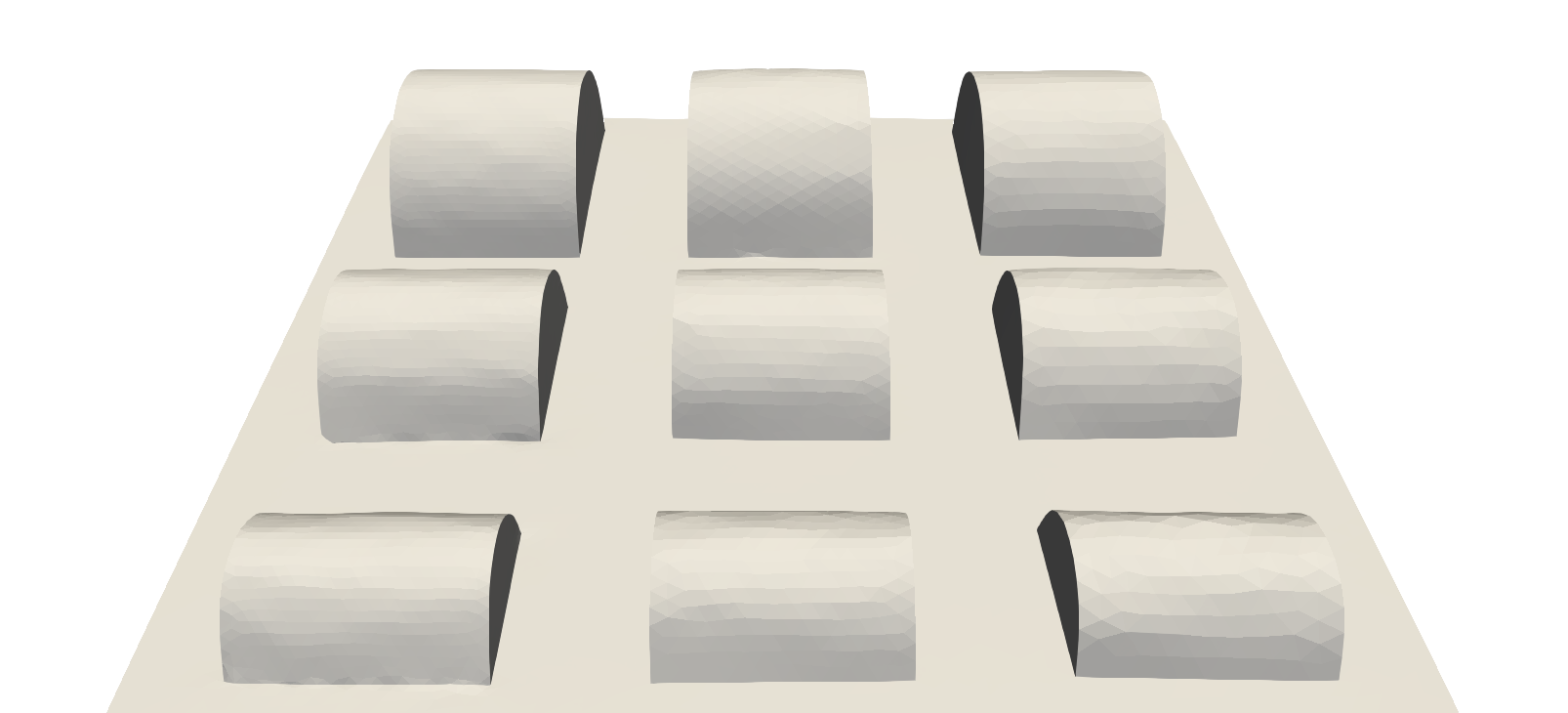}
	\hfill
	\includegraphics[width = 0.49\linewidth]{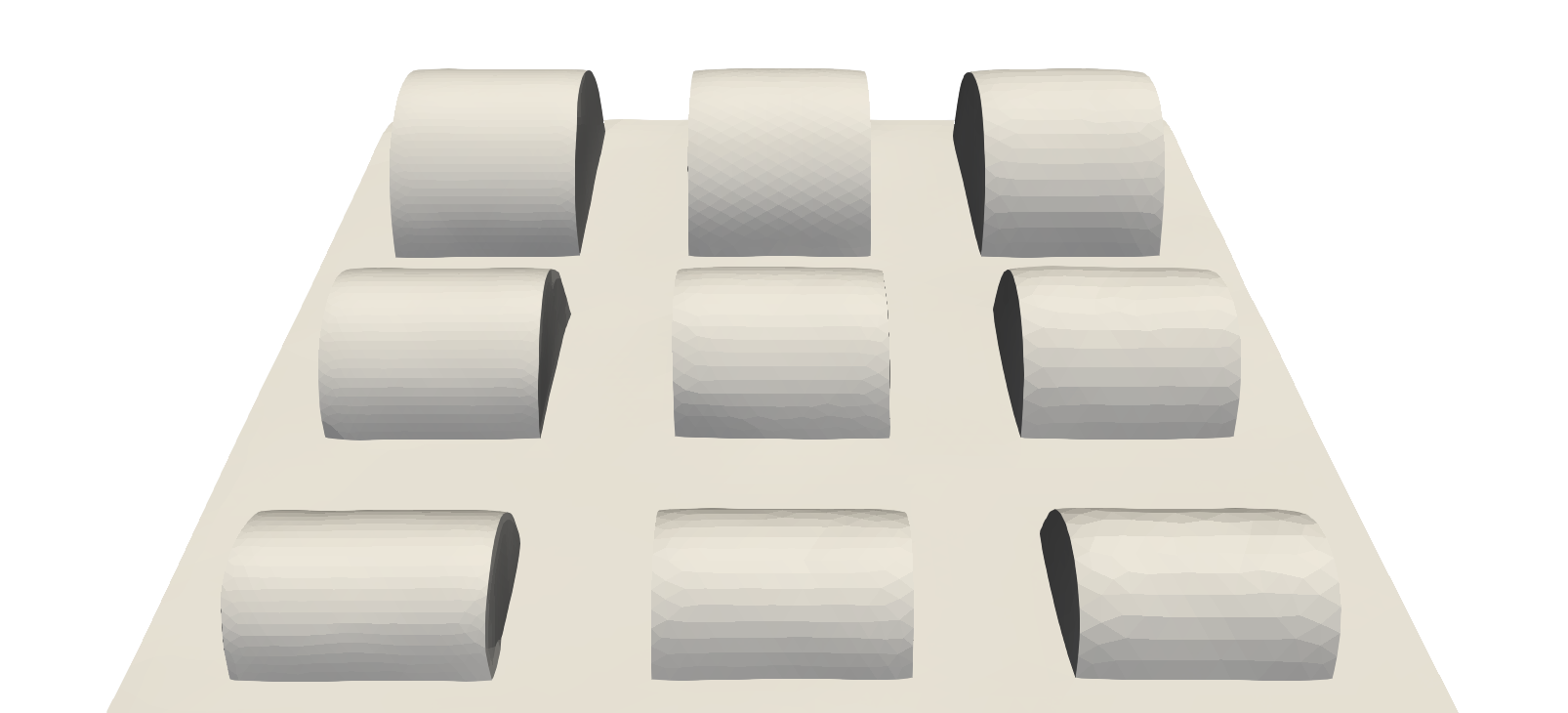}
	\caption{%
		Top row: original geometry (left), noisy geometry (right).
		Bottom row: reconstructions using meshTGV~\cite{LiuLiWangLiuChen:2022:1} (left) with $\alpha_0 = 0.2$ and $\alpha_1 = 1.5$, and the proposed $\fetgv^2$ (right) using \eqref{eq:tgv:normal} with $\alpha_0 = 3\cdot 10^{-5}$, $\alpha_1 = 3.5 \cdot 10^{-3}$.
	}
	\label{figure:cylinders}
\end{figure}

For the example shown in \cref{figure:cylinders}, the meshTGV method of \cite{LiuLiWangLiuChen:2022:1} achieved $\disttext{vertices} = \num{0.00141}$ and $\disttext{normals} = \num{0.0307}$, while our method reached $\disttext{vertices} = \num{0.00120}$ and $\disttext{normals} = \num{0.0302}$.
This time, the results of both models are quantitatively more similar and both achieve good reconstruction results.
However, minor staircasing artifacts remain visible in some of the cylinders reconstructed using the method by \cite{LiuLiWangLiuChen:2022:1}.

In summary, comparing columns in \cref{figure:spheres,figure:cylinders}, we may conclude that the method from \cite{LiuLiWangLiuChen:2022:1} appears to be fully suitable only for geometries with at least one of the principal curvatures vanishing locally.
While our approach produces just slightly better results in this case, it significantly outperforms the meshTGV method from \cite{LiuLiWangLiuChen:2022:1} in terms of reconstruction quality for the hemisphere case, were both principal curvatures are nonzero.
By comparing rows, we see that the performance of our method is independent of the mesh resolution and, in particular, does not require the regularization parameters to be resolution dependent.
For the meshTGV method of \cite{LiuLiWangLiuChen:2022:1}, the slight starcasing effect appears to be more pronounced for the highest mesh resolution.
This can be observed in the first column of the bottom left subplot of \cref{figure:cylinders}.

\subsection{Fandisk Mesh}
\label{subsection:fandisk-mesh}

The next experiment concerns the well-known fandisk mesh and features a relatively low amount of noise, which was provided via the dataset from \cite{ZhangHeWang:2022:1}.
Specifically, each component of a vertex coordinate is perturbed by Gaussian noise of standard deviation of $0.1$ times the average edge length; see \cite[Section~5.1]{ZhangHeWang:2022:1}.
The results of the three $\TGV$ models \cite{LiuLiWangLiuChen:2022:1}, \cite{ZhangHeWang:2022:1} and \eqref{eq:tgv:normal}, as well as first-order TV \eqref{eq:tv:normal} as in \cite{BaumgaertnerBergmannHerzogSchmidtVidalNunezWeiss:2025:1}, are shown in \cref{figure:fandisk:denoising}.
The distance measures to the original mesh via $\disttext{vertices}$ \eqref{eq:mesh_distance} and $\disttext{normals}$ \eqref{eq:normal_distance} are summarized in \cref{table:fandisk_distances}

\begin{figure}[ht]
	\centering
	\begin{tikzpicture}[spy using outlines = {circle, size = 2.5cm, magnification = 3, connect spies}]
		\node (img01) at (0,6.25) {%
				\includegraphics[width = 0.3\linewidth]{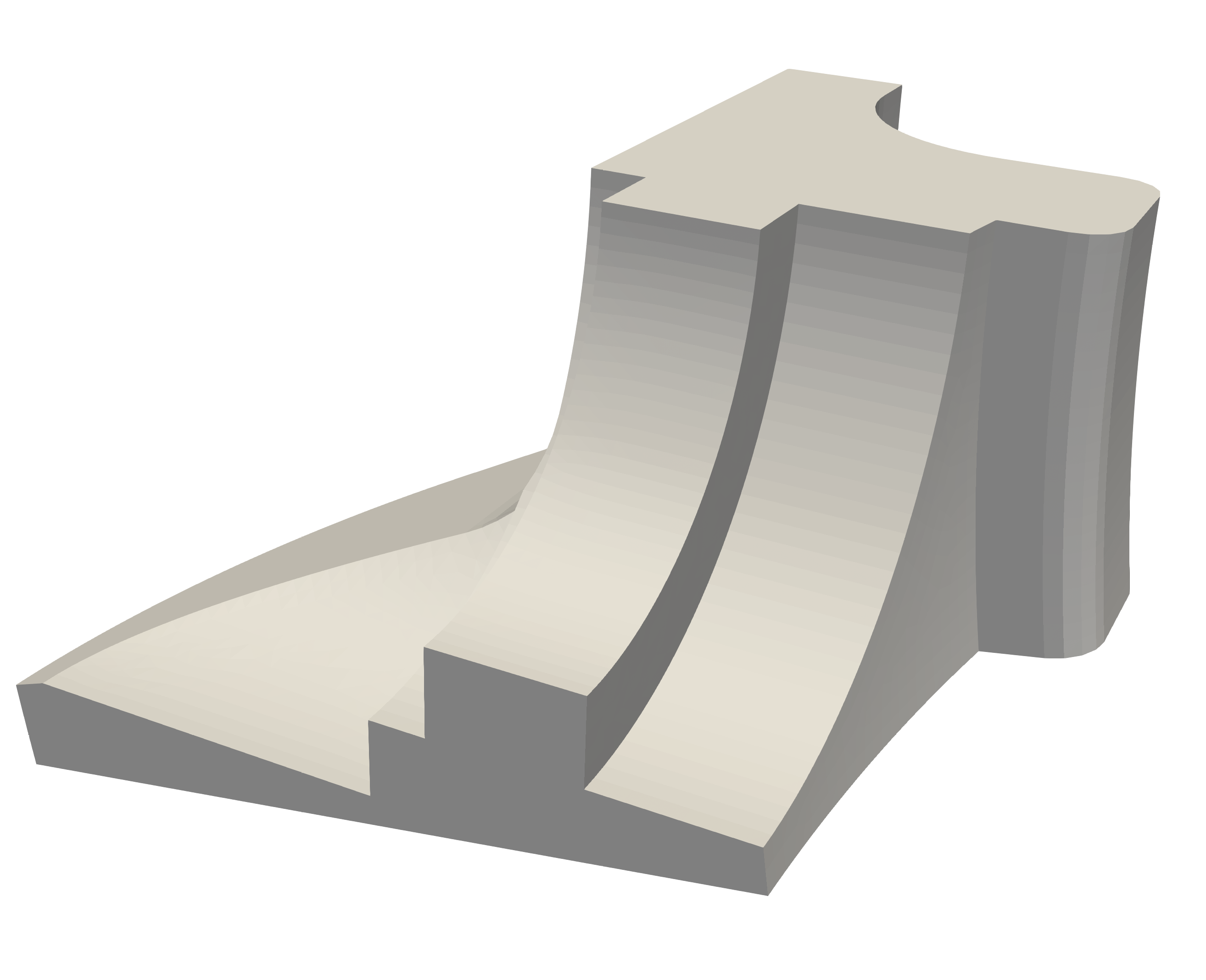}
			};
		\node (img02) at (0.33\linewidth,6.25) {%
				\includegraphics[width = 0.3\linewidth]{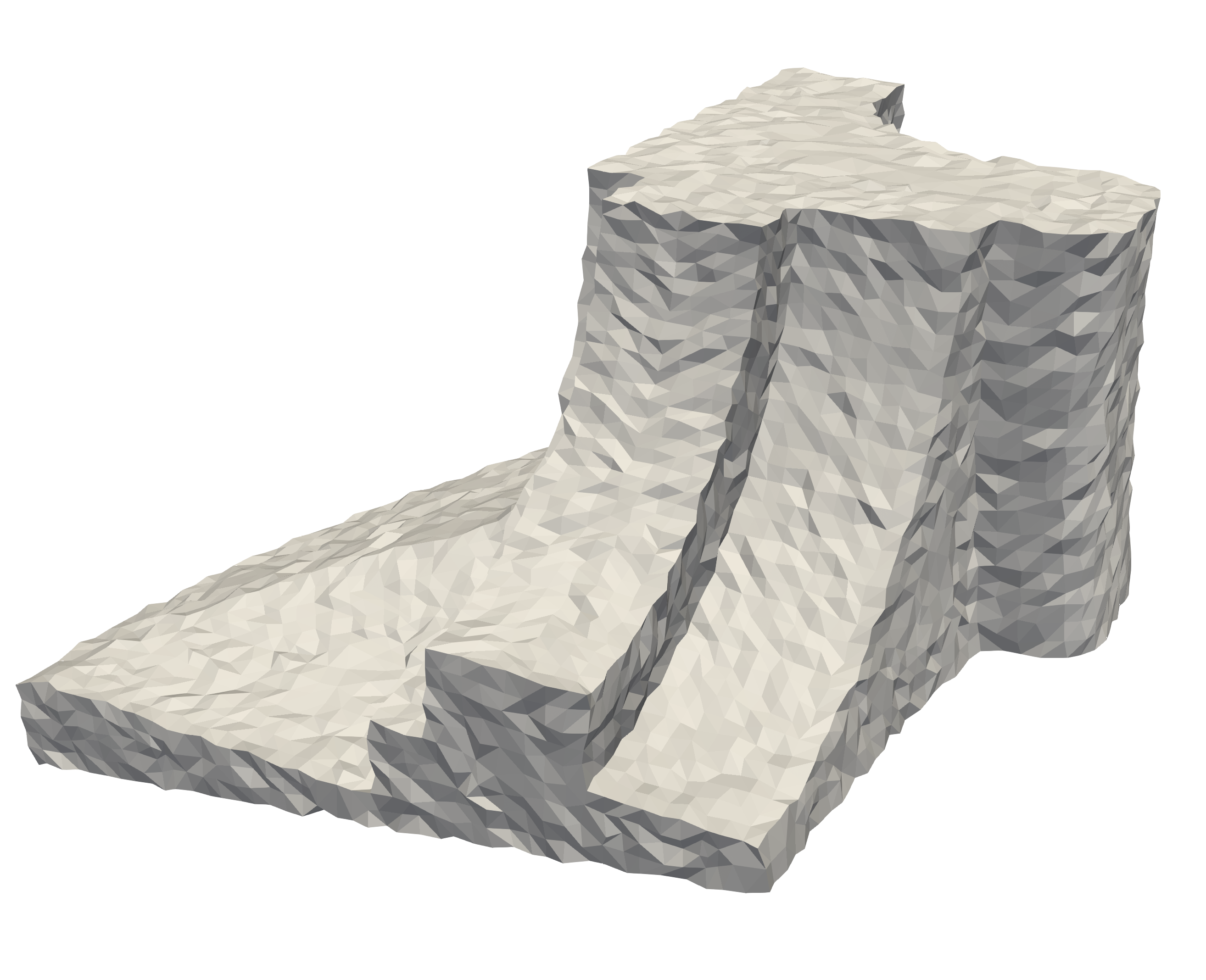}
			};
		\node (img03) at (0.66\linewidth,6.25) {%
				\includegraphics[width = 0.3\linewidth]{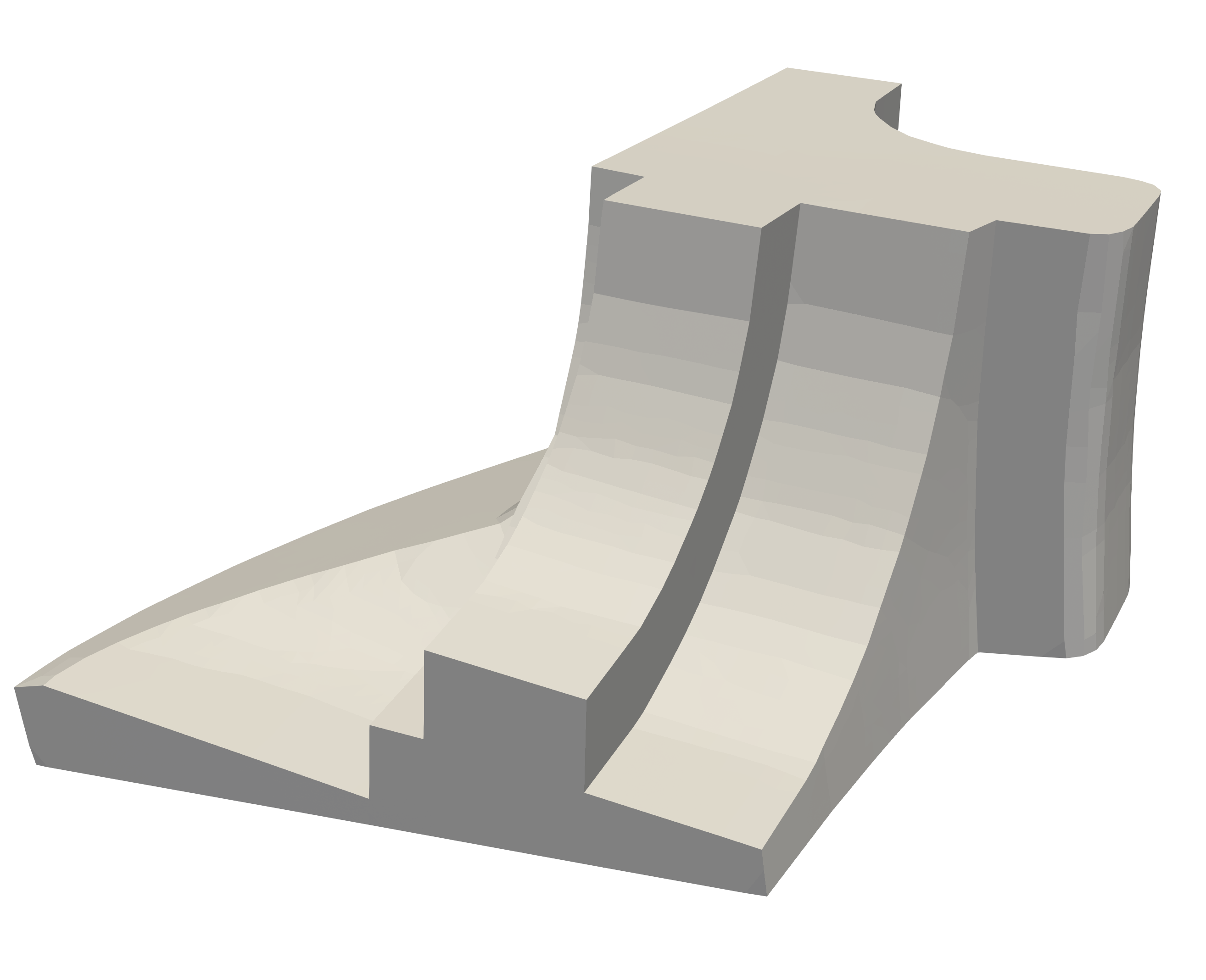}
			};
		\node (img1) at (0,0) {%
				\includegraphics[width = 0.3\linewidth]{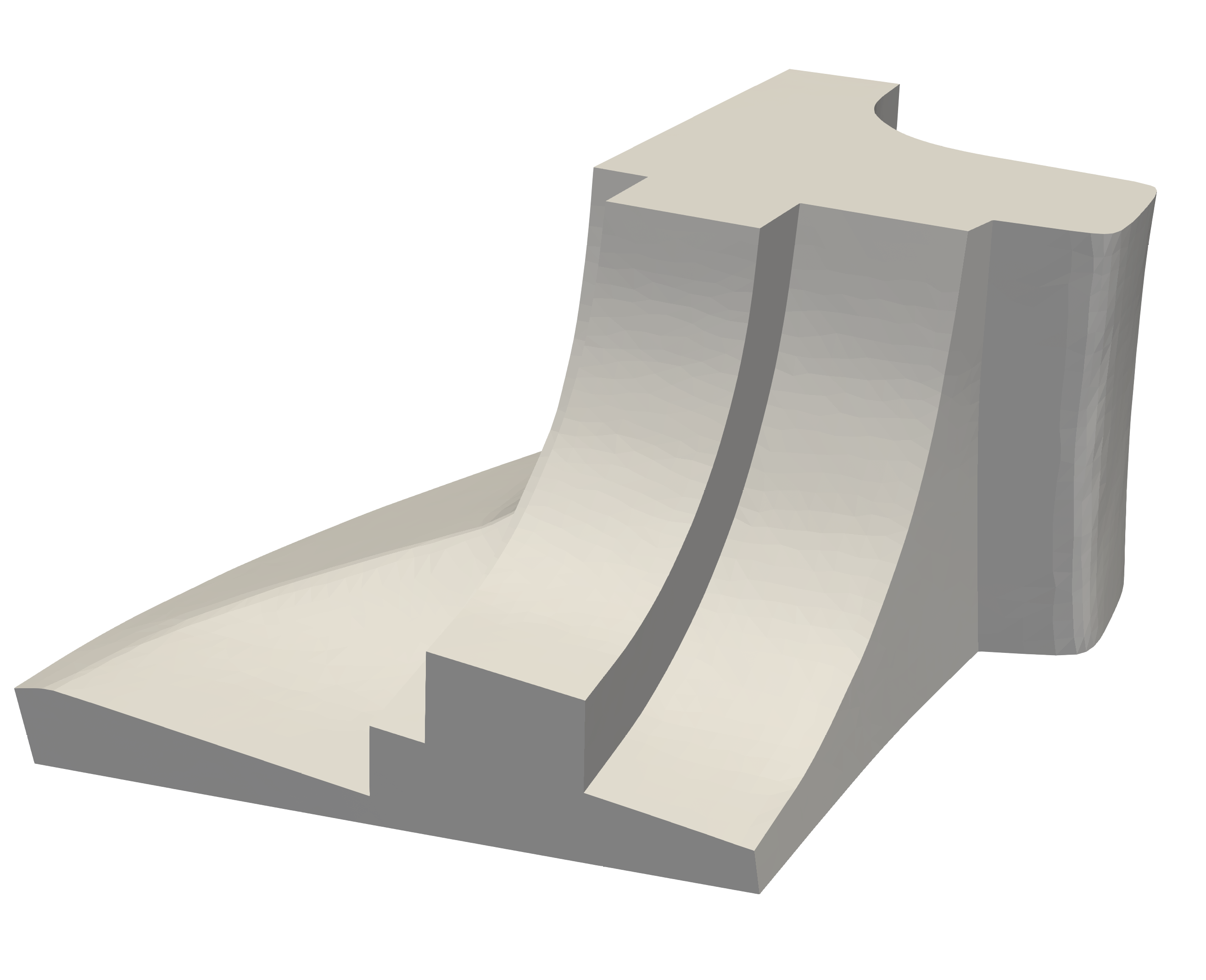}
			};
		\node (img2) at (0.33\linewidth,0) {%
				\includegraphics[width = 0.3\linewidth]{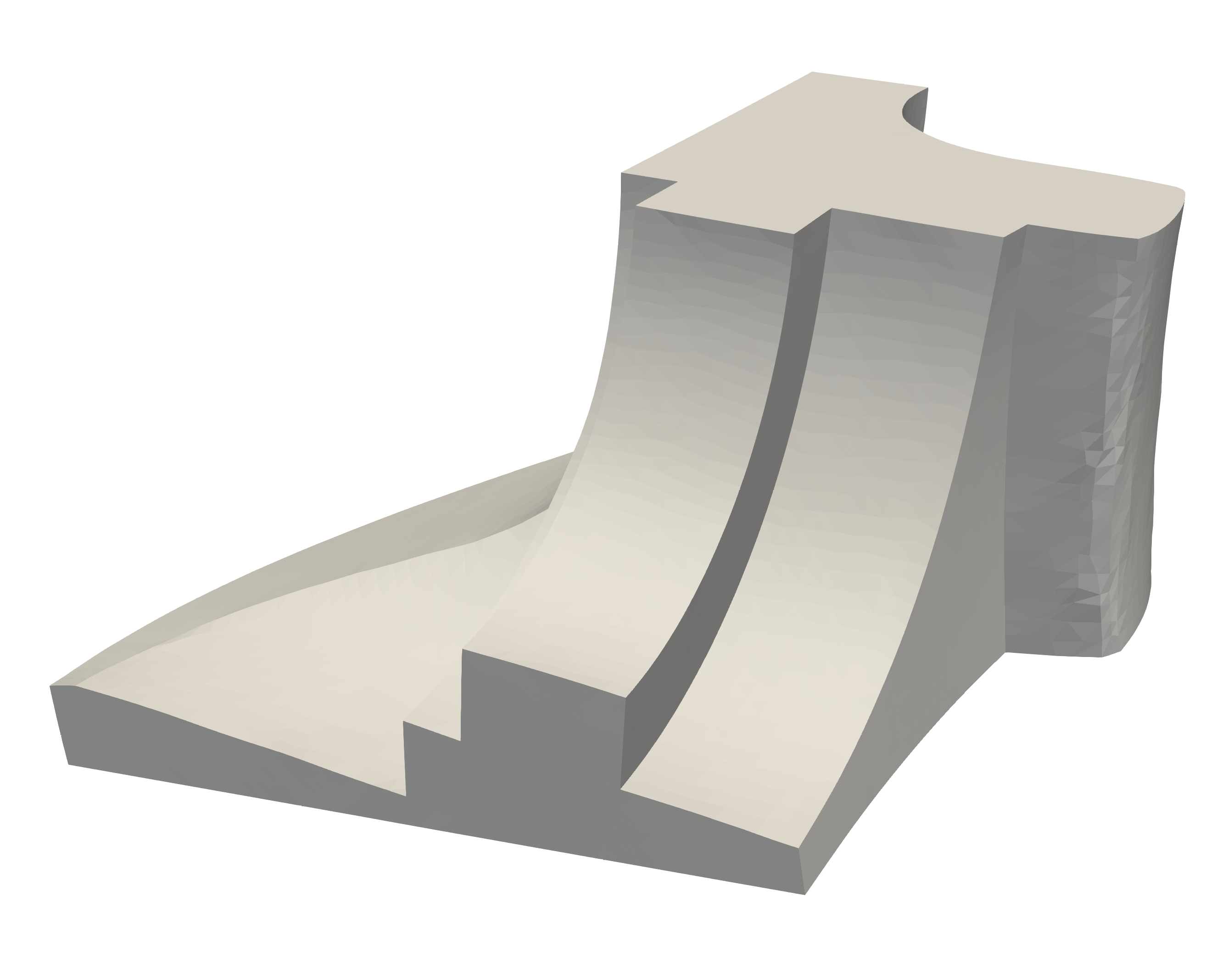}
			};
		\node (img3) at (0.66\linewidth,0) {%
				\includegraphics[width = 0.3\linewidth]{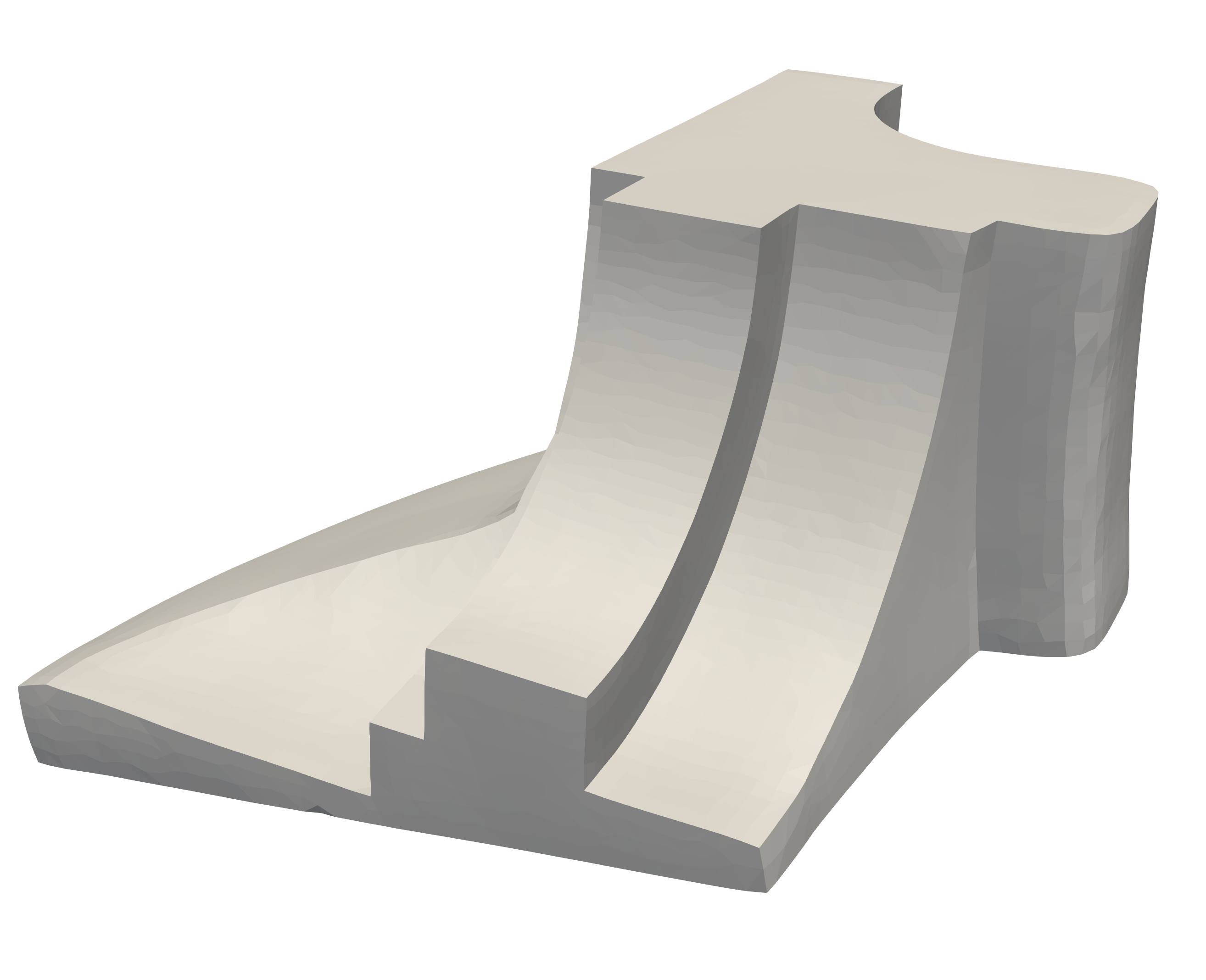}
			};
		\begin{scope}[x = {($ (img1.south east) - (img1.south west) $ )}, y = {( $ (img1.north west) - (img1.south west)$ )}, shift = {(img1.south west)}]
			%
			\node (spy1n) at (0.19\linewidth,0.53) {};
			\coordinate (spy1nto) at (.120\linewidth,1.25);
			\spy [black!100,thick] on (spy1n) in node at (spy1nto);

			\node (spy2n) at (0.33\linewidth + 0.20\linewidth,0.53) {};
			\coordinate (spy2nto) at (.3736\linewidth,1.25);
			\spy [black!100,thick] on (spy2n) in node at (spy2nto);

			\node (spy3n) at (0.66\linewidth + 0.19\linewidth,0.53) {};
			\coordinate (spy3nto) at (.6263\linewidth,1.25);
			\spy [black!100,thick] on (spy3n) in node at (spy3nto);

			\node (spy4n) at (0.66\linewidth + 0.19\linewidth,2.03) {};
			\coordinate (spy4nto) at (.880\linewidth,1.25);
			\spy [black!100,thick] on (spy4n) in node at (spy4nto);%
		\end{scope}
	\end{tikzpicture}
	\caption{%
		Top row: original fandisk geometry (left), noisy geometry (middle) and $\TV$ reconstruction (right) using \eqref{eq:tv:normal} from \cite{BaumgaertnerBergmannHerzogSchmidtVidalNunezWeiss:2025:1} with $\beta = 2 \cdot 10^{-2}$.
		Bottom row: reconstructions using meshTGV~\cite{LiuLiWangLiuChen:2022:1} (left) with $\alpha_0 = 0.2$ and $\alpha_1 = 1.2$, rTGV~\cite{ZhangHeWang:2022:1} (middle), and the proposed $\fetgv^2$ (right) using \eqref{eq:tgv:normal} with $\alpha_0 = 10^{-5}$, $\alpha_1 = 10^{-3}$.
	}
	\label{figure:fandisk:denoising}
\end{figure}

\makeatletter
\begin{table}[h]
	\centering
	\begin{tabular}{lSSSS}
		\toprule
		&
		\ltx@ifclassloaded{siamonline250106}{%
			{TV~\cite{BaumgaertnerBergmannHerzogSchmidtVidalNunezWeiss:2025:1}}
		}{%
			{TV}
		}
		&
		\ltx@ifclassloaded{siamonline250106}{%
			{meshTGV~\cite{LiuLiWangLiuChen:2022:1}}
		}{%
			{meshTGV}
		}
		&
		\ltx@ifclassloaded{siamonline250106}{%
			{rTGV~\cite{ZhangHeWang:2022:1}}
		}{%
			{rTGV}
		}
		&
		{our \cref{algorithm:ADMM}}
		\\
		\midrule
		 $\disttext{vertices}$ \eqref{eq:mesh_distance}
		 &
		 0.000629
		 &
		 0.00163
		 &
		 0.00336
		 &
		 0.000530
		 \\
		 $\disttext{normals}$ \eqref{eq:normal_distance}
		 &
		 0.0261
		 &
		 0.0273
		 &
		 0.0234
		 &
		 0.0210
		 \\
		 \bottomrule
	\end{tabular}
	\caption{%
		Distance measures for the fandisk test case, see \cref{figure:fandisk:denoising}.
		\ltx@ifclassloaded{siamonline250106}{}{%
			TV refers to \cite{BaumgaertnerBergmannHerzogSchmidtVidalNunezWeiss:2025:1}.
			meshTGV refers to \cite{LiuLiWangLiuChen:2022:1}.
			rTGV refers to \cite{ZhangHeWang:2022:1}.
		}
	}
	\label{table:fandisk_distances}
\end{table}
\makeatother

As expected, the first-order total variation regularization method suffers from the staircasing effect, while all three $\TGV$ models manage to reconstruct the fandisk almost perfectly (\cref{figure:fandisk:denoising}).
In particular, no staircasing effect can be seen in the curved areas.
With regards to the distance measures reported in \cref{table:fandisk_distances}, recall that both meshTGV~\cite{LiuLiWangLiuChen:2022:1} and rTGV~\cite{ZhangHeWang:2022:1} are using a normal filtering approach, while our \cref{algorithm:ADMM} as well as the TV regularization from \cite{BaumgaertnerBergmannHerzogSchmidtVidalNunezWeiss:2025:1} are using vertex tracking \eqref{eq:mesh-denoising-with-tgv-of-the-normal-regularization}.
While vertex tracking prevents unnecessary changes in the overall size of the geometry, there is no such mechanism in the normal filtering approaches in meshTGV~\cite{LiuLiWangLiuChen:2022:1} and rTGV~\cite{ZhangHeWang:2022:1}.
Therefore, the geometries generally slightly grow or shrink through these approaches, which is the reason for the less favorable values of $\disttext{vertices}$ in \cref{table:fandisk_distances}.
Our approach is also superior with respect to the $\disttext{normals}$ metric, although the scores are much closer here.

\subsection{Joint Mesh}
\label{subsection:joint-mesh}

The third example concerns the geometry of a joint.
Again, the mesh is provided by the authors of \cite{ZhangHeWang:2022:1}, who added Gaussian noise with standard deviation of $0.3$ times the average edge length.
The numerical results comparing the three $\TGV$ models as well as the first-order $\TV$ model are presented in \cref{figure:joint:denoising}.
The distance measures to the original mesh are summarized in \cref{table:joint_distances}.

\begin{figure}[ht]
	\centering
	\begin{tikzpicture}[spy using outlines = {circle, size = 2.5cm, magnification = 3, connect spies}]
		\node (img01) at (0,6.5) {%
				\includegraphics[width = 0.3\linewidth]{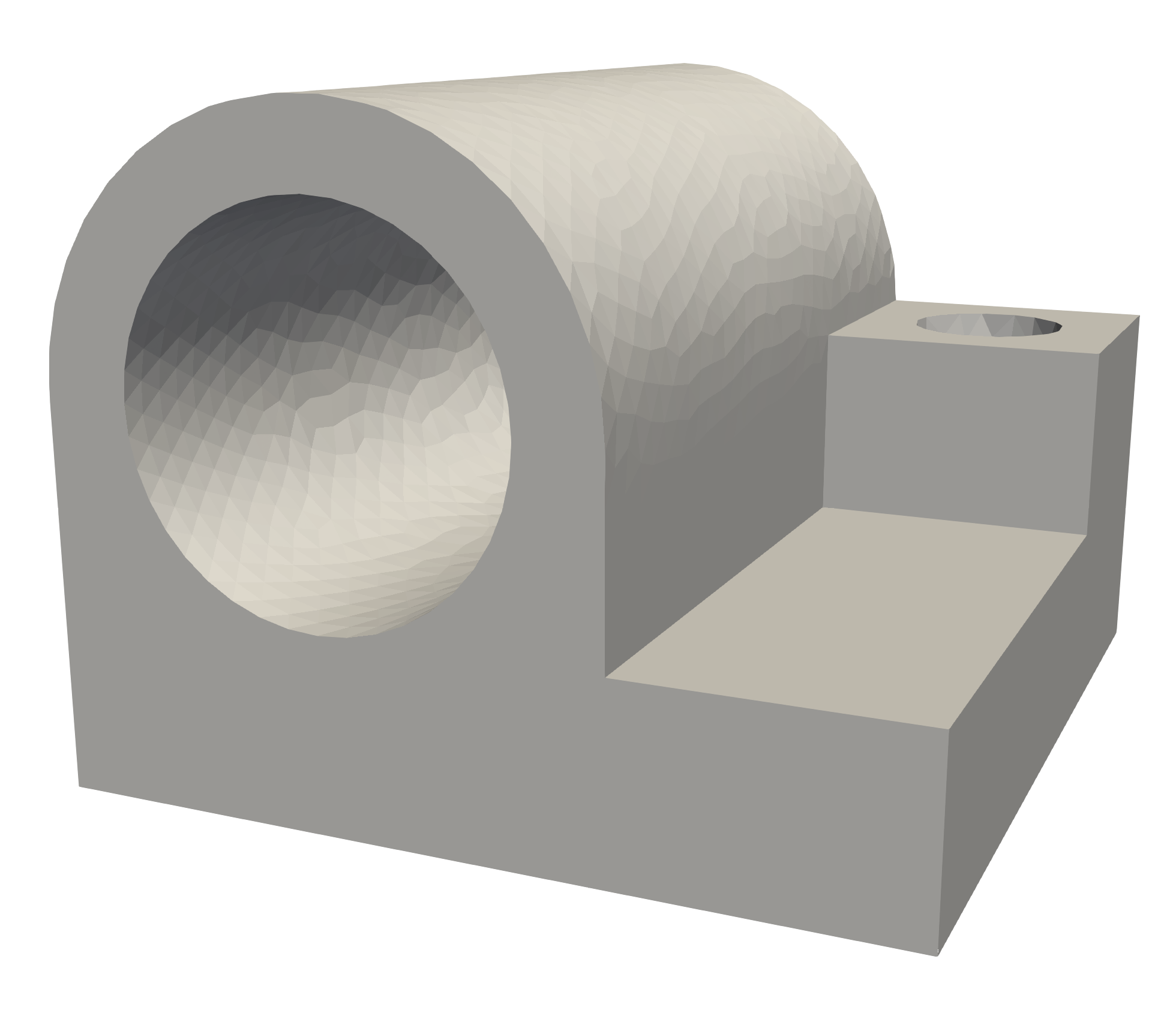}
			};
		\node (img02) at (0.33\linewidth,6.5) {%
				\includegraphics[width = 0.3\linewidth]{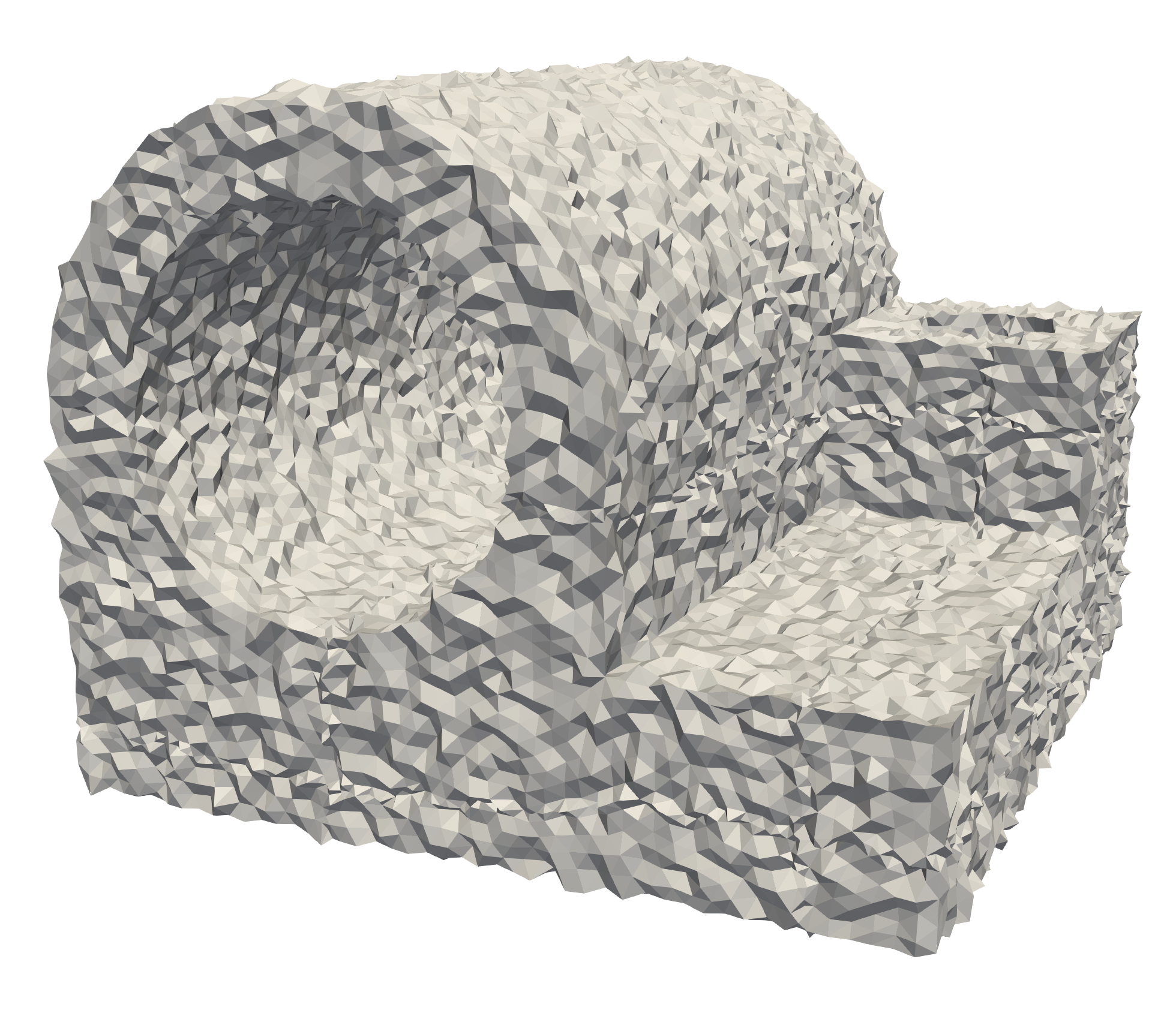}
			};
		\node (img03) at (0.66\linewidth,6.5) {%
				\includegraphics[width = 0.3\linewidth]{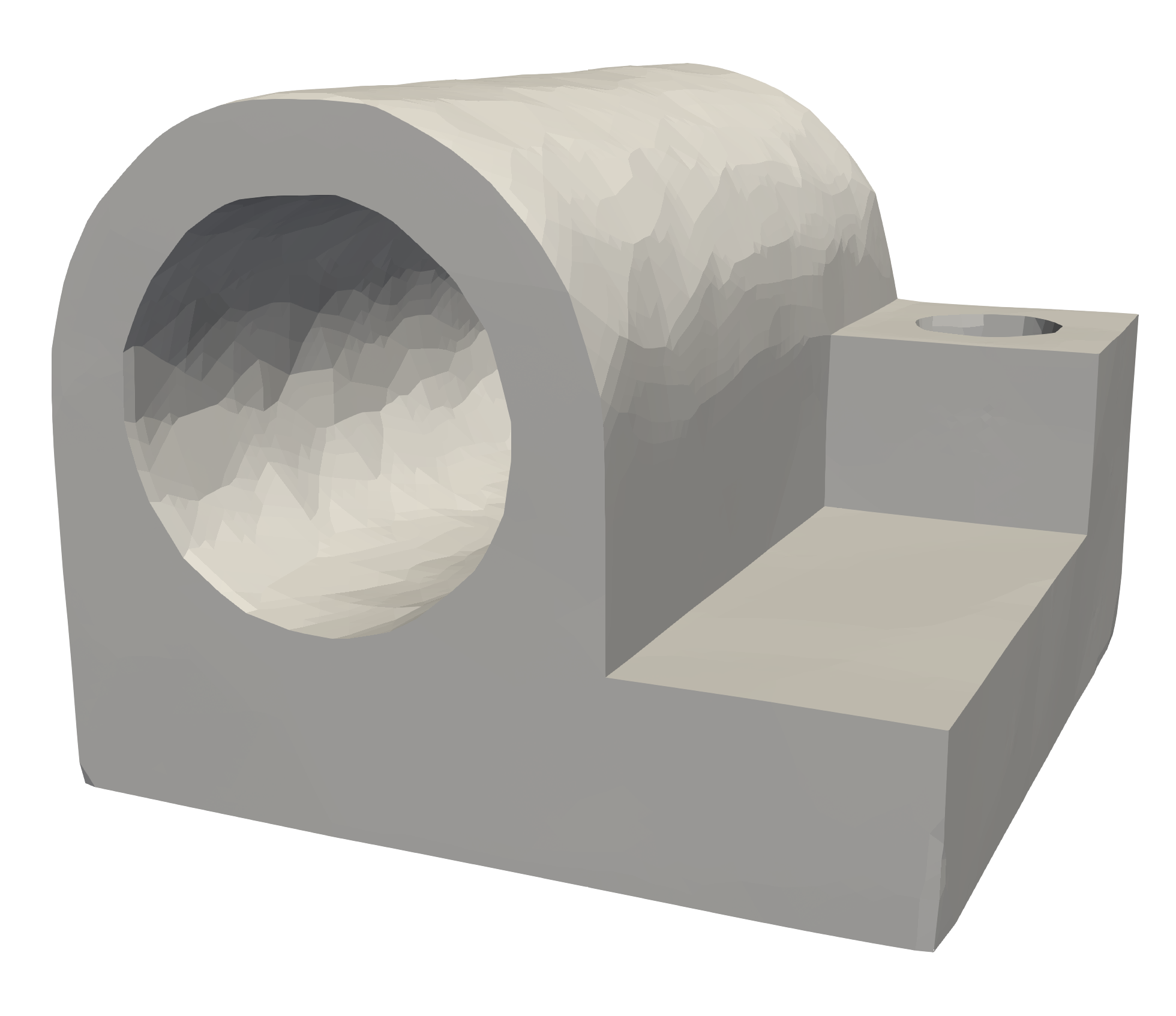}
			};
		\node (img1) at (0,0) {%
				\includegraphics[width = 0.3\linewidth]{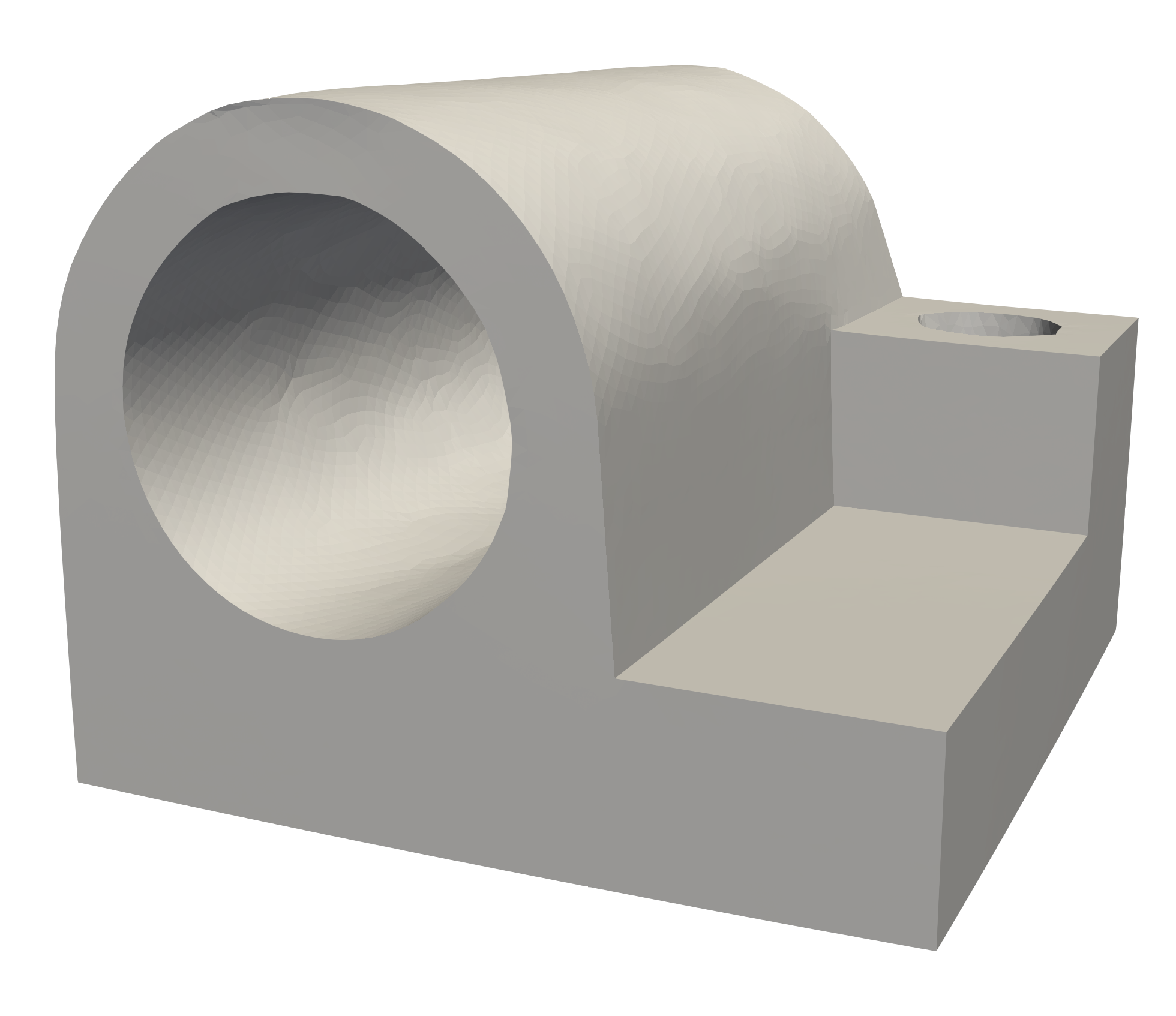}
			};
		\node (img2) at (0.33\linewidth,0) {%
				\includegraphics[width = 0.3\linewidth]{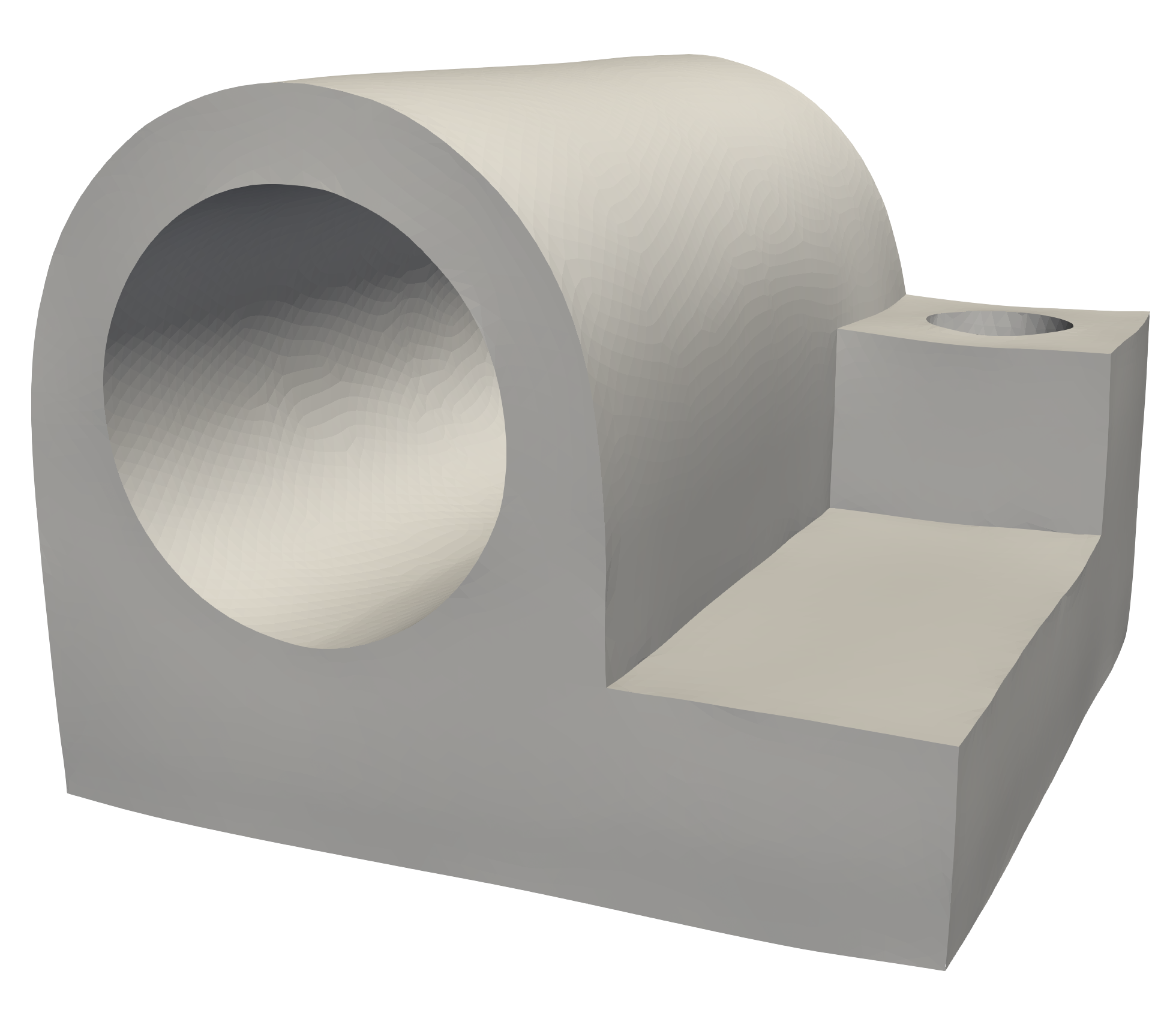}
			};
		\node (img3) at (0.66\linewidth,0) {%
				\includegraphics[width = 0.3\linewidth]{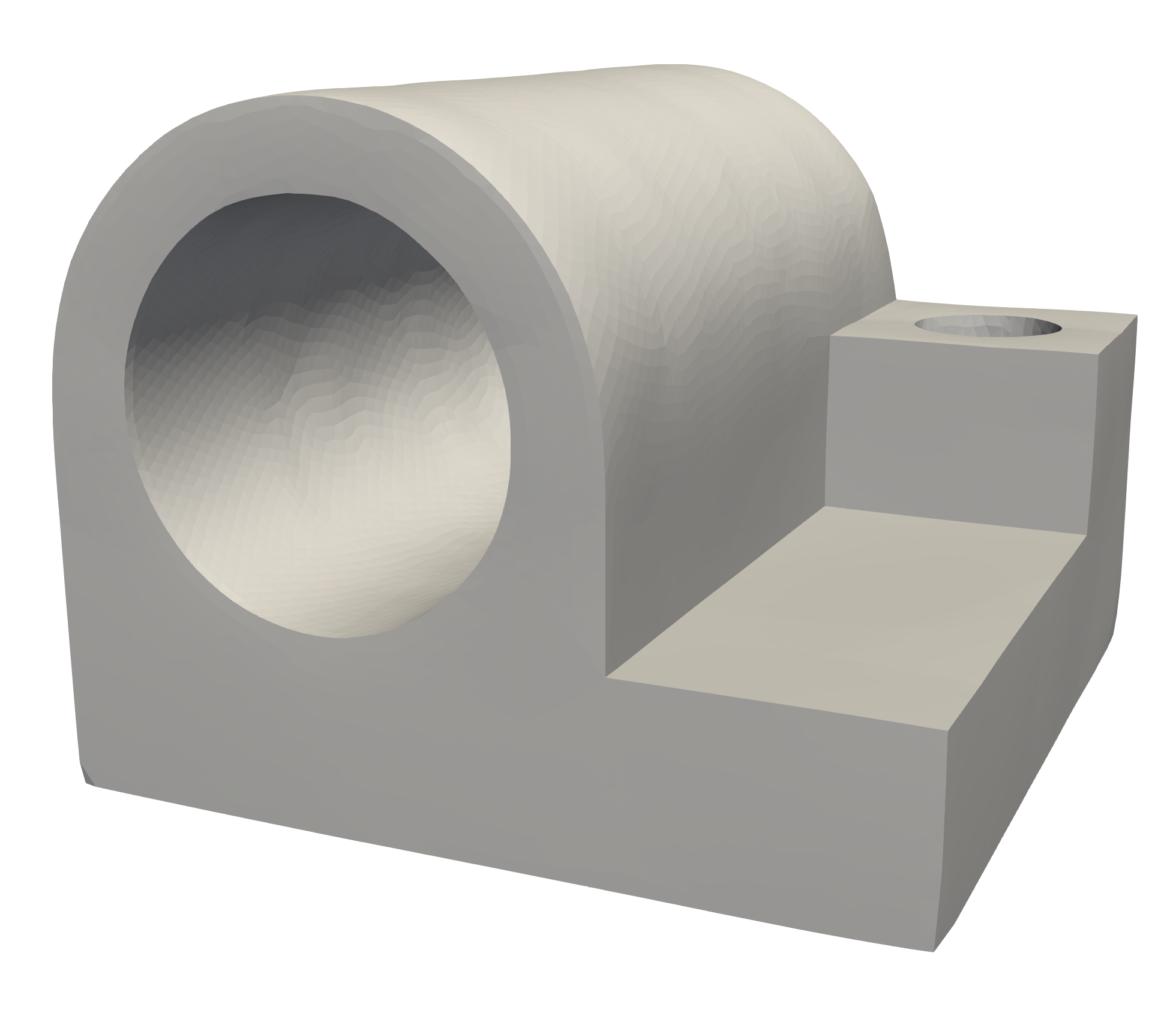}
			};
		\begin{scope}[x = {($ (img1.south east) - (img1.south west) $ )}, y = {( $ (img1.north west) - (img1.south west)$ )}, shift = {(img1.south west)}]
			%
			\node (spy1n) at (0.065\linewidth,0.5) {};
			\coordinate (spy1nto) at (.120\linewidth,1.25);
			\spy [black!100,thick] on (spy1n) in node at (spy1nto);%

			\node (spy2n) at (0.33\linewidth + 0.059\linewidth,0.5) {};
			\coordinate (spy2nto) at (.3736\linewidth,1.25);
			\spy [black!100,thick] on (spy2n) in node at (spy2nto);%

			\node (spy3n) at (0.66\linewidth + 0.065\linewidth,0.5) {};
			\coordinate (spy3nto) at (.6263\linewidth,1.25);
			\spy [black!100,thick] on (spy3n) in node at (spy3nto);%

			\node (spy4n) at (0.66\linewidth + 0.065\linewidth,2.01) {};
			\coordinate (spy4nto) at (.880\linewidth,1.25);
			\spy [black!100,thick] on (spy4n) in node at (spy4nto);%
		\end{scope}
	\end{tikzpicture}
	\caption{%
		Top row: original joint geometry (left), noisy geometry (middle) and $\TV$ reconstruction using \eqref{eq:tv:normal} from \cite{BaumgaertnerBergmannHerzogSchmidtVidalNunezWeiss:2025:1} with $\beta = 2 \cdot 10^{-2}$.
		Bottom row: reconstruction using meshTGV~\cite{LiuLiWangLiuChen:2022:1} (left) with $\alpha_0 = 0.5$ and $\alpha_1 = 2.0$, rTGV~\cite{ZhangHeWang:2022:1} (middle), and the proposed $\fetgv^2$ (right) using \eqref{eq:tgv:normal} with $\alpha_0 = 5 \cdot 10^{-5}$, $\alpha_1 = 4 \cdot 10^{-3}$.
	}
	\label{figure:joint:denoising}
\end{figure}

\makeatletter
\begin{table}[h]
	\centering
	\begin{tabular}{lSSSS}
		\toprule
		&
		\ltx@ifclassloaded{siamonline250106}{%
			{TV~\cite{BaumgaertnerBergmannHerzogSchmidtVidalNunezWeiss:2025:1}}
		}{%
			{TV}
		}
		&
		\ltx@ifclassloaded{siamonline250106}{%
			{meshTGV~\cite{LiuLiWangLiuChen:2022:1}}
		}{%
			{meshTGV}
		}
		&
		\ltx@ifclassloaded{siamonline250106}{%
			{rTGV~\cite{ZhangHeWang:2022:1}}
		}{%
			{rTGV}
		}
		&
		{our \cref{algorithm:ADMM}}
		\\
		\midrule
		 $\disttext{vertices}$ \eqref{eq:mesh_distance}
		&
		0.000817
		&
		0.00174
		&
		0.0129
		&
		0.000616
		\\
		 $\disttext{normals}$ \eqref{eq:normal_distance}
		&
		0.0382
		&
		0.0332
		&
		0.0372
		&
		0.0274
		\\
		\bottomrule
	\end{tabular}
	\caption{%
		Distance measures for the joint test case, see \cref{figure:joint:denoising}.
		\ltx@ifclassloaded{siamonline250106}{}{%
			TV refers to \cite{BaumgaertnerBergmannHerzogSchmidtVidalNunezWeiss:2025:1}.
			meshTGV refers to \cite{LiuLiWangLiuChen:2022:1}.
			rTGV refers to \cite{ZhangHeWang:2022:1}.
		}
	}
	\label{table:joint_distances}
\end{table}
\makeatother

As for the other examples, the results in \cref{figure:joint:denoising} look similar for all three $\TGV$ models, whereas the first-order $\TV$ model produces staircasing.
Similar as before, a significant change of size in the meshes is responsible for the less favorable scores of meshTGV~\cite{LiuLiWangLiuChen:2022:1} and rTGV~\cite{ZhangHeWang:2022:1} in the $\disttext{vertices}$ metric, see \cref{table:joint_distances}.
Our approach does not suffer from this problem and is thus superior in both metrics presented in \cref{table:joint_distances}.

We recognize that the extrinsic $\TGV$ models \cite{LiuLiWangLiuChen:2022:1} and \cite{ZhangHeWang:2022:1} are very capable to reconstruct the areas of (discrete) constant principal curvatures in all examples so that our intrinsic approach can not be considered superior in that regard.
However, it is worth noticing that we specifically derived our regularizer \eqref{eq:tgv:normal} to leave areas of (discrete) constant principal curvatures unpenalized, which makes our method favor piecewise planar, spherical or cylindrical areas.
Such properties have not been investigated for the approaches of \cite{LiuLiWangLiuChen:2022:1} and \cite{ZhangHeWang:2022:1}.

\section{Conclusion}
\label{section:conclusion}

We propose a discrete, intrinsic formulation of the second-order total generalized variation (TGV) of the normal vector field of oriented, triangulated meshes embedded in~$\R^3$.
Particular attention is given to the differentially geometric consequences arising from the fact that the normal vector is an element of the unit sphere.
To capture the derivative information via an auxiliary variable $W$, we introduce a new tangential Raviart--Thomas space.
At every point, a function of this space represents a mapping from the tangent space of the mesh to the tangent space of the sphere and thus matches the push-forward operator of the normal vector.

To solve minimization problems involving the new TGV regularizer, we derive an alternating direction method of multipliers (ADMM) capable of treating the non-smoothness of the problem.
We compare our approach to the extrinsic variants of discrete total generalized variation of the normal vector field from \cite{LiuLiWangLiuChen:2022:1} and \cite{ZhangHeWang:2022:1}, which treat the normal vector as an element of $\R^3$ rather than as an element of the unit sphere.
We also compare to first-order total variation of the normal \cite{BaumgaertnerBergmannHerzogSchmidtVidalNunezWeiss:2025:1} for mesh denoising problems.
While all second-order $\TGV$ models successfully remove most of the starcasing effect in the half-cylinder test case (\cref{figure:cylinders}) that would be produced by the first-order TV regularizer, our method $\fetgv^2$ performs favorably compared to meshTGV~\cite{LiuLiWangLiuChen:2022:1} and rTGV~\cite{ZhangHeWang:2022:1} in the general case where both principal curvature are non-vanishing, as demonstrated by the hemisphere example (\cref{figure:spheres}).
For real-world geometries, the differences between the $\fetgv^2$, meshTGV and rTGV are less pronounced in \cref{figure:fandisk:denoising} and \cref{figure:joint:denoising}.
However, it is worth noticing that our approach achieves a better score in the similarity measures \eqref{eq:mesh_distance} and \eqref{eq:normal_distance}; see \cref{table:fandisk_distances} and \cref{table:joint_distances}.
This is mostly because the geometries slightly grow or shrink through the normal filtering approaches of meshTGV~\cite{LiuLiWangLiuChen:2022:1} and rTGV~\cite{ZhangHeWang:2022:1}.

\section*{Acknowledgments}

We sincerely thank the authors of \cite{LiuLiWangLiuChen:2022:1} for making their method publicly available.
We also thank the authors of \cite{ZhangHeWang:2022:1} for providing us with access to their numerical results.
Parts of this paper were written while the third author was visiting the University of British Columbia, Vancouver.
He would like to thank the Department of Computer Science for their hospitality.

\appendix

\printbibliography

\end{document}